\documentclass{article}
\usepackage{iclr2026_conference_arxiv}
\usepackage{url}
\usepackage[T1]{fontenc}
\usepackage[utf8]{inputenc}
\usepackage{lmodern}
\usepackage{silence}
\usepackage{microtype}
\usepackage{amsmath,amssymb,amsfonts,amsthm,mathtools}
\usepackage{aliascnt}
\usepackage{bbm}
\usepackage{bm}
\usepackage{enumitem}
\usepackage{booktabs}
\usepackage{graphicx}
\usepackage{placeins}
\usepackage{float}
\usepackage[table]{xcolor}
\usepackage{hyperref}
\usepackage[nameinlink,capitalize]{cleveref}
\mathtoolsset{showonlyrefs}
\usepackage[scr=boondoxo]{mathalfa}
\usepackage{longtable}
\usepackage{array}
\newcolumntype{L}[1]{>{\raggedright\arraybackslash}p{#1}}
\usepackage{subcaption}

\hypersetup{
  colorlinks=true,
  linkcolor=blue!60!black,
  citecolor=blue!60!black,
  urlcolor=blue!60!black
}

\definecolor{lfHeaderGray}{HTML}{F3F4F6}
\definecolor{lfBlue}{HTML}{EAF2F8}
\definecolor{lfTeal}{HTML}{E7F4F2}
\definecolor{lfPurple}{HTML}{F3EEF8}
\definecolor{lfGreen}{HTML}{EAF6F0}
\definecolor{lfAmber}{HTML}{FFF4E3}
\definecolor{terracotta}{HTML}{9D5242}

\newtheorem{theorem}{Theorem}[section]
\newaliascnt{proposition}{theorem}
\newtheorem{proposition}[proposition]{Proposition}
\aliascntresetthe{proposition}
\newaliascnt{lemma}{theorem}

\aliascntresetthe{lemma}
\newaliascnt{corollary}{theorem}
\newtheorem{corollary}[corollary]{Corollary}
\aliascntresetthe{corollary}
\theoremstyle{definition}
\newaliascnt{definition}{theorem}

\aliascntresetthe{definition}
\newaliascnt{assumption}{theorem}
\newtheorem{assumption}[assumption]{Assumption}
\aliascntresetthe{assumption}
\newaliascnt{example}{theorem}

\aliascntresetthe{example}
\theoremstyle{remark}
\newaliascnt{remark}{theorem}
\newtheorem{remark}[remark]{Remark}
\aliascntresetthe{remark}

\crefname{section}{Section}{Sections}
\Crefname{section}{Section}{Sections}
\crefname{appendix}{appendix}{appendices}
\Crefname{appendix}{Appendix}{Appendices}

\newcommand{\R}{\mathbb R}
\newcommand{\N}{\mathcal N}
\newcommand{\E}{\mathbb E}
\newcommand{\KL}{\mathrm{KL}}
\newcommand{\Law}{\mathcal L}
\newcommand{\dd}{\,\mathrm d}

\newcommand{\calC}{\mathcal C}

\newcommand{\calE}{\mathcal E}

\newcommand{\vart}{\vartheta}

\newcommand{\1}{\mathbbm 1}

\providecommand{\Law}{\operatorname{Law}}

\title{\textsc{LatentFlow}: A General Framework for \\ Conditioning Stochastic Processes}

\author{Louis Sharrock \\
University College London\\
\And
Lachlan Astfalck \\
University of New South Wales \\
\And
Henry Moss \\
Lancaster University \\
}

\arxivpreprintcopy 

\begin{document}

\maketitle

\begin{abstract}
Stochastic-process models are, as a rule, far easier to simulate than to condition.
Non-linear observations, non-Gaussian likelihoods, black-box information, and global constraints all induce intractable conditional laws, requiring bespoke, model-specific constructions.
We introduce \textsc{LatentFlow}, a single framework for conditioning stochastic processes, with no learned neural approximations and no training.
Our starting point is to write the stochastic process as the deterministic image of a tractable latent innovation, \(f_0 = T_{\vart}(\xi_0)\), with \(\xi_0\) sampled from a simple reference distribution. 
This reduces process-level conditioning to latent-space inference: pull the likelihood back through $T_{\vart}$, sample the resulting latent law with a tractable guided probability flow, and push the samples forward.
This construction is provably exact at the level of the target law; in practice, approximation enters only through finite terminal noising, Monte Carlo guidance, and time discretisation of the continuous-time dynamics, each of which is explicit and systematically reducible.
As \textsc{LatentFlow} is training-free, conditioning reduces to solving a single reverse-time SDE.
This enables conditional sampling in seconds on a single desktop CPU across model classes that have never shared a scalable method: classical spatial priors, nonlinear stochastic dynamics, mechanistic models from the physical and life sciences, stochastic PDEs, heavy-tails and extremes, point and discrete-state processes, and neural or simulator-defined processes. 
\end{abstract}

\section{Introduction}
\label{sec:intro}
Sampling from a stochastic process only requires running it forward; conditioning it, until now, has remained the preserve of bespoke, model-specific constructions.
Examples include diffusion paths conditioned on fixed endpoints or partial observations \citep{delyon2006simulation,schauer2017guided};
functions with shape \citep{riihimaki2010gaussian} or physics-informed \citep{chen2021solving,hamelijnck2024physics} constraints; and simulator outputs conditioned on indirect measurements, rare events, or user-defined criteria \citep{botev2020sampling,finzi2023user}.
In each case the conditioning information cannot be absorbed into the model in closed form, demanding its own specialised, expensive, or approximate sampler; for instance, those based on MCMC \citep{andrieu2003introduction}, SMC \citep{doucet2001sequential}, variational inference \citep{blei2017variational} or Laplace approximations \citep{rue2009approximate}. 

Recently, \cite{moss2026conditioning} introduced \textsc{FlowGP}, which addresses conditioning for Gaussian process (GP) priors under intractable information by recasting conditional sampling as a differential equation with closed-form Gaussian dynamics and a likelihood-dependent guidance term.
The resulting sampler applies to any likelihood that can be evaluated pointwise, and requires no trained network to approximate the guidance field.
Underlying the model is a simple mechanism: the GP sample is the linear, deterministic image of a Gaussian innovation, and conditioning is a re-weighting of that innovation's density by the pulled-back likelihood.

Our key observation is that Gaussianity of the stochastic process prior is not the essential ingredient in this construction; instead, it is the availability of a tractable \emph{latent innovation representation} on which guided dynamics can be defined. 
Rather than constructing conditioning dynamics directly in the original process space, this allows us to express the process as the deterministic image of a tractable innovation variable, pull the conditioning information back through this generator, and sample the resulting reweighted law in latent space.
The conditioned process is then obtained by pushing the guided samples in the latent space forward through the same generator.
We refer to this method as \textsc{LatentFlow}.
Our construction does not require learning a process prior, score model, or conditional generator \citep[e.g.,][]{song2021score,chung2023diffusion,zammit2025neural}; the sampler is built directly from the known map \(T_{\vart}\) and pointwise evaluations of the pulled-back likelihood.
Our method also provides a natural way to learn the parameters of the process generator. 

We obtain formal theoretical guarantees for \textsc{LatentFlow}, and demonstrate its application to a wide variety of stochastic process priors. 
These include GPs; heavy-tailed Student-$t$ and Cauchy-convolution processes; discrete-valued Potts processes; linear and nonlinear diffusion processes including a model for cell-differentiation, a FitzHugh--Nagumo model for excitable neurons, and a Heston stochastic volatility model; mechanistic models of SIR epidemics and Lotka--Volterra predator--prey dynamics; and Allen--Cahn and advection--diffusion SPDEs.
See \Cref{fig:figure1} for several examples.
Until now, beyond simple or contrived cases, conditioning each of these processes required either low-fidelity approximations or expensive computation.
We sample in single-seconds time on a single desktop CPU, providing the first general method for conditioning stochastic processes in real time which is exact up to an explicit and systematically reducible numerical error.

\begin{figure}[!ht]
    \centering
    \includegraphics[width=.9\textwidth]{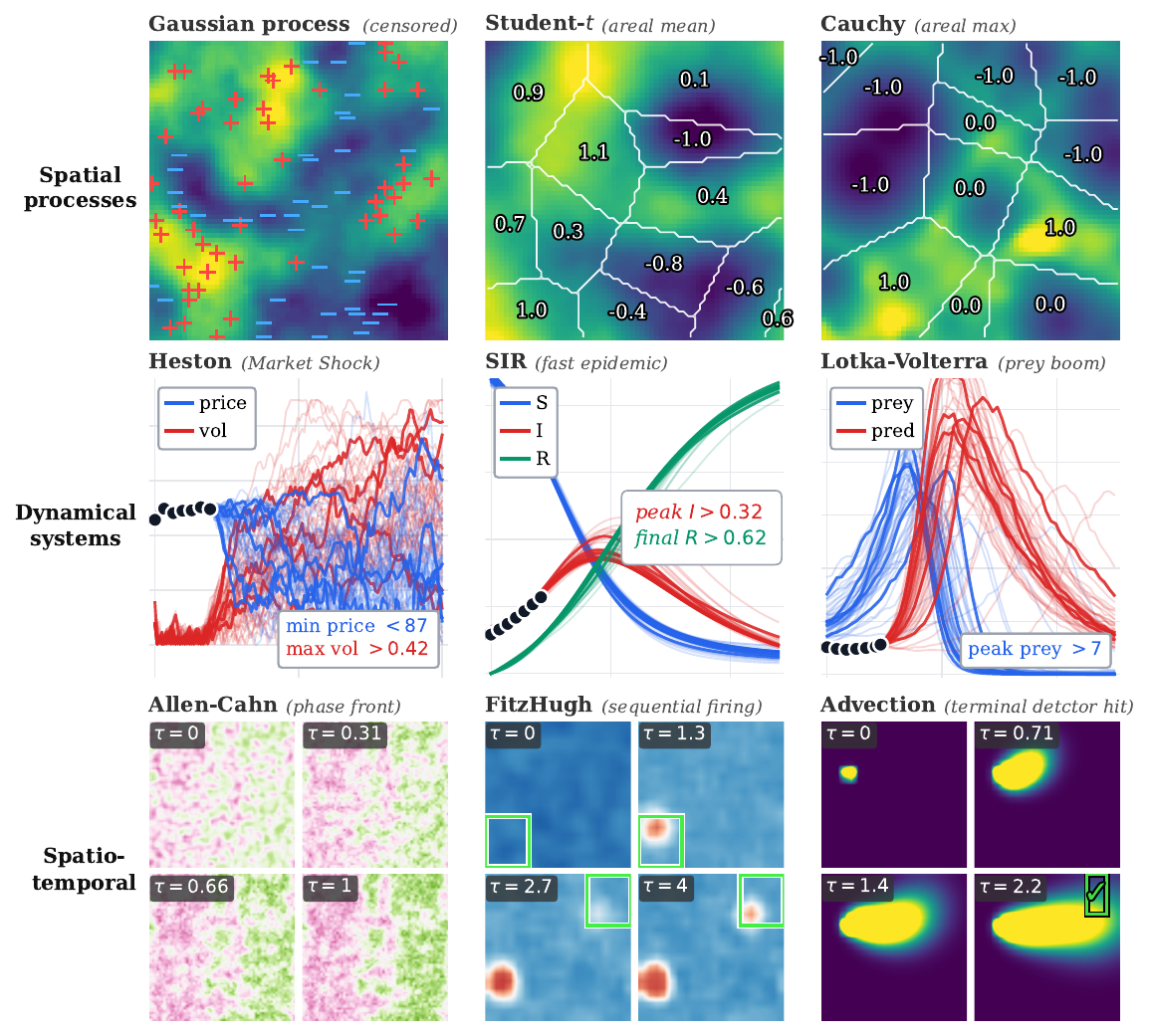}
    \caption{
    \textsc{LatentFlow} across diverse stochastic-process priors and conditions. 
    \textbf{Top}: spatial processes conditioned on censored point measurements for a Gaussian process; prescribed areal means for a Student-\(t\) process; and prescribed areal maxima for a Cauchy convolution field. 
    \textbf{Centre}: SDEs conditioned on observations and events. A Heston model under a joint market-crash and volatility spike; an SIR model under a fast epidemic peak; and a Lotka--Volterra system given a large prey boom. 
    \textbf{Bottom}: Spatio-temporal PDEs. Allen--Cahn conditioned on a prescribed phase front (negative left, positive right) at final time; FitzHugh--Nagumo conditioned on sequential source firing (bottom-left then top-right); and an advection--diffusion plume conditioned on reaching a downstream detector (box).
    }
    \label{fig:figure1}
    \vspace{-3mm}
\end{figure}

\section{Conditioning via a Latent Representation}
\label{sec:conditioning-latent-representation}
Let \(\mathscr{f} \in\mathcal E\) denote a stochastic process, viewed as a random element in a measurable space \(\mathcal E\), and let \(f_0\in E_m\) denote a finite-dimensional discretisation of $\mathscr{f}$.
In typical examples, 
\(
  E_m\subseteq\R^m
\)
or
\(
  E_m\subseteq\R^{md},
\)
where \(m\) is the number of input locations and \(d\) is the state dimension at each location.
Let $p_{\vart}(f_0)$ denote the prior density of $f_0$ on \(E_m\), with hyperparameters $\vart$.
We are interested in sampling the conditional distribution of this process, given some additional information \(\calC\) encoded by a likelihood
\(
  L_{\calC}:E_m\to[0,\infty).
\)
The desired conditional density is thus given by
\begin{equation}
  \pi_{\vart}(f_0 \mid \calC)
  =
  \frac{L_{\calC}(f_0)\,p_{\vart}(f_0)}{Z_{\vart}},
  \qquad
  Z_{\vart}
  =
  \int_{E_m} L_{\calC}(f_0)\,p_{\vart}(f_0)\dd f_0.
  \label{eq:process-target}
\end{equation}
%
%
Our key assumption is that \(p_{\vart}\) can be obtained by transforming a simple reference density.
\begin{assumption}
\label{ass:latent-rep}
There exists a latent dimension \(m_\xi\), a reference density \(r\) on \(\R^{m_\xi}\), and a map
\(
  T_{\vart}:\R^{m_\xi}\to E_m
\)
such that, if $\xi_0\sim r$, then $f_0=T_{\vart}(\xi_0)$ has density $p_{\vart}$. 
\end{assumption}
\Cref{ass:latent-rep} is very mild; after discretisation, it holds for essentially any process of interest.
Consider a scalar target with continuous distribution function $F$.
By the probability integral transform, the variable $F^{-1}(\mathrm{R}(\xi_0))$ with $\xi_0\sim r$ has distribution $F$, where $\mathrm{R}$ is the distribution function of $r$. 
The Rosenblatt transform \citep{rosenblatt1952remarks} generalises this to multivariate $f_0$ and writes \textit{any} absolutely continuous law on $E_m$ as the deterministic image of some $\xi_0 \sim r$.
The choice of $r$ is arbitrary; we fix $r(\xi_0)=\N(\xi_0;0,I_{m_\xi})$ throughout.
It is worth noting that we use densities here only to simplify the presentation. 
In fact, we require only a measure-theoretic version of \Cref{ass:latent-rep}, which also applies when the prior is singular or discrete; see \Cref{app:measure-theoretic-view} for details.

Once such a representation is available, conditioning the process on \(\mathcal C\) can be transferred to the latent space via the \emph{latent likelihood} \(
\Lambda_{\vart, \calC}(\xi_0)
  =
  (L_{\calC}\circ T_{\vart})(\xi_0).
\),  i.e. the standard likelihood pulled back to $\R^{m_\xi}$. 
The corresponding latent target density is
\begin{equation}
  \rho_{\vart}(\xi_0 \mid \calC)
  =
  \frac{\Lambda_{\vart, \calC}(\xi_0)\,r(\xi_0)}{Z_{\vart}},
  \qquad
  Z_{\vart}
  =
  \int_{\mathbb R^{m_\xi}} \Lambda_{\vart, \calC}(\xi_0)\,r(\xi_0)\dd\xi_0,
  \label{eq:latent-target}
\end{equation}
where the posterior is recovered by pushing $\rho_{\vart}$ through $T_{\vart}$. 
That is, if $\xi_0\sim\rho_{\vart}( \cdot \mid \calC)$, then $f_0=T_{\vart}(\xi_0)$ has density $\pi_{\vart}(\cdot\mid\calC)$.
We formalise this equivalence in \Cref{prop:latent-process-agree}, which states that the latent and process targets agree in distribution. 

\begin{proposition}
\label{prop:latent-process-agree}
Suppose that Assumption~\ref{ass:latent-rep} holds and that \(0<Z_{\vart}<\infty\). 
Then, if \(\xi_0\sim\rho_{\vart}(\cdot \mid \calC) \), the random variable \(f_0=T_{\vart}(\xi_0)\) has the conditioned process density $\pi_{\vart}(\cdot\mid\calC)$.
\end{proposition}

\subsection{Examples of latent innovation representations}
\label{sec:background-examples}
Below we provide three examples of stochastic processes that satisfy Assumption~\ref{ass:latent-rep}.
Further examples, including heavy-tailed and max-stable processes, state-space models and SPDEs, and neural processes and black-box simulators, are given in \Cref{app:innovation-representations}.

\textbf{Gaussian process}.
Let \(f\sim\operatorname{GP}(\mu_{\vart},k_{\vart})\), and let
\(
  f_0=(f(s_1),\ldots,f(s_m))\in\mathbb R^m
\)
denote its values at inputs \(s_{1:m}\).
Thus
\(
  f_0\sim \N(m_{\vart},K_{\vart}),
\)
where \(m_{\vart}=(\mu_{\vart}(s_1),\ldots,\mu_{\vart}(s_m))\) and \((K_{\vart})_{ij}=k_{\vart}(s_i,s_j)\).
Let \(L_{\vart}\) be a matrix satisfying
\(
  L_{\vart}L_{\vart}^{\top}=K_{\vart}.
\)
We can then write 
\[
  f_0
  =
  T_{\vart}(\xi_0)
  =
  m_{\vart}+L_{\vart}\xi_0, \qquad \xi_0\sim\N(0,I_m).
\] 
Thus, the latent innovation is the whitened Gaussian vector \(\xi_0\), and the generator \(T_{\vart}\) is the usual affine GP whitening map.

\textbf{Student-\(t\) process}. 
Let \(f\sim\operatorname{TP}_{\nu}(\mu_{\vart},k_{\vart})\), where \(k_{\vart}\) denotes the Student-\(t\) scale kernel, and let
\(
  f_0=(f(s_1),\ldots,f(s_m))\in\mathbb R^m
\)
denote its values at inputs \(s_{1:m}\).
We then have
\[
  f_0
  =
  m_{\vart}
  +
  \omega^{-1/2}L_{\vart}v,
\]
where $v\sim\N(0,I_m)$ and $\omega\sim\operatorname{Gamma}(\nu/2,\nu/2)$ independently.
We can also express $f_0$ in terms of a standard Gaussian reference. Let \(a\sim\N(0,1)\), and define
\(
  G_{\nu}(a)
  =
  F_{\Gamma,\nu}^{-1}(\Phi_{\mathrm N}(a)),
\)
where \(F_{\Gamma,\nu}\) is the \(\operatorname{Gamma}(\nu/2,\nu/2)\) CDF, and \(\Phi_{\mathrm N}\) is the standard normal CDF.
Then
\[
  f_0
  =
  T_{\vart}(\xi_0)
  =
  m_{\vart}
  +
  G_{\nu}(a)^{-1/2}L_{\vart}v, \qquad \xi_0=(v,a)\sim\N(0,I_{m+1}).
\]
Thus, \(m\) Gaussian variables generate the spatial Gaussian component, and one additional Gaussian variable generates the random scale.

\textbf{SDE-driven process}.
Consider a process defined via the stochastic differential equation 
\[
  \dd x_\tau
  =
  b_{\vart}(\tau,x_\tau)\dd \tau
  +
  \Sigma_{\vart}(\tau,x_\tau)\dd w_\tau,
\]
where $w=(w_\tau)_{\tau\geq 0}$ is a standard Brownian motion.
Define a time grid
\(
  0=\tau_0<\tau_1<\cdots<\tau_N=T,
\)
with
\(
  \Delta_n=\tau_{n+1}-\tau_n.
\)
The Euler--Maruyama discretisation of this SDE is
\[
  x_{n+1}
  =
  x_n
  +
  b_{\vart}(\tau_n,x_n)\Delta_n
  +
  \Sigma_{\vart}(\tau_n,x_n)\sqrt{\Delta_n}\eta_n,
\]
where $\eta_n\stackrel{\mathrm{i.i.d.}}{\sim}\N(0,I)$.  
If $x_0$ is random, we also write
\(
  x_0=\chi_{\vart}(\eta_{-1}),
\)
with
\(
  \eta_{-1}\sim\N(0,I),
\)
and omit \(\eta_{-1}\) when $x_0$ is fixed.
The latent innovation is the collection of Gaussian variables
\(
  \xi_0=(\eta_{-1},\eta_0,\ldots,\eta_{N-1}),
\)
with the obvious modification for deterministic \(x_0\).
Meanwhile, the generator \(T_{\vart}\) is the numerical solver itself: it takes \(\xi_0\), applies the Euler--Maruyama recursion, and returns the discretised path
\(
  (x_0,x_1,\ldots,x_N) = T_{\vart}(\xi_0).
\)
\section{LatentFlow: Inference with Guided Latent Flows}
\label{sec:bridge}

We now develop the method on which the paper rests. 
Whenever the latent likelihood is differentiable, or its score admits a tractable estimator, a guided diffusion draws samples \(\xi_0\sim\rho_{\vart}(\cdot \mid \calC) \).
We call the resulting framework \textsc{LatentFlow}. 
The statements provided here are formal; see \Cref{app:exactness-proofs} for a more rigorous treatment.

\subsection{Noising bridge, potential, and guidance field}
We first define a Gaussian noising process on the latent space.
Let \(\xi_0\sim r\) and \(\varepsilon\sim\N(0,I_{m_\xi})\) be independent of \(\xi_0\).
The variance-preserving (VP) latent noising bridge is given by
\begin{equation}
  \xi_t
  =
  \alpha(t)\xi_0+\sigma(t)\varepsilon,
  \label{eq:latent-bridge}
\end{equation}
where $\alpha(t)=\exp\{-\frac{1}{2}\int_0^t\beta(s)\dd s\}$ and $\sigma(t)^2=1-\alpha(t)^2$, for a locally integrable function \(\beta:[0,1)\to(0,\infty)\) such that $\lim_{t\uparrow 1}\alpha(t)=0$.
Thus \(\xi_0\) is the clean latent variable and \(\xi_t\) for \(t\uparrow1\) is increasingly noised.

Since $r(\xi_0)=\N(\xi_0;0,I_{m_\xi})$, we have \(\xi_t\sim r\) for every \(t\in[0,1)\). 
Moreover, the conditional density of $\xi_0 \mid \xi_t$ is given by
\begin{equation}
  r_{0\mid t}(\xi_0) \coloneqq p(\xi_0\mid \xi_t = u) = \N\bigl(\xi_0;\alpha(t)u,\sigma(t)^2I_{m_\xi}\bigr).
  \label{eq:bridge-conditional}
\end{equation}
The additional condition enters through the latent likelihood \(\Lambda\). 
For \(t\in[0,1)\), we define the \emph{latent conditioning potential} as
\begin{equation}
    \psi(t,u)
    =
    \E_{\xi_0\sim r}\left[\Lambda_{\vart,\calC}(\xi_0)\mid \xi_t=u\right]
    = 
    \E_{\varepsilon\sim r} \left[ \Lambda_{\vart,\calC} \left(\alpha(t)u+\sigma(t)\varepsilon\right) \right].
    \label{eq:psi-conditional-definition}
\end{equation}
This can be viewed as a smoothed version of the latent likelihood.
This function also has a density-ratio interpretation.
Let \(\xi_t\sim\rho_t\) be the random variable obtained by drawing \(\xi_0\sim\rho_{0}\) and then applying the VP noising bridge in \eqref{eq:latent-bridge}.
Then 
\(
\smash{\psi(t,u) =Z_{\vart}\frac{\rho_t(u)}{r(u)}}
\).

Finally, suppose \(\psi(t,u)>0\) and that \(\psi(t,\cdot)\) is differentiable at \(u\).
We define for $t\in[0,1)$ the \emph{latent guidance field}
\begin{equation}
  g(t,u)
  =
  \nabla_u\log\psi(t,u)
  =
  \nabla_u\log\tfrac{\rho_t(u)}{r(u)}.
  \label{eq:g-definition}
\end{equation}
This will later be used to steer the dynamics in the direction of the increased expected future likelihood.

\subsection{Exact latent dynamics}
\label{sec:exact-latent}
Our goal is a sampling procedure for $\xi_0 \sim \rho_0$ with $\rho_0 \equiv \rho_\vart$. 
First recall that the family of guided densities \((\rho_t)_{0\le t<1}\) are precisely the one-time marginals of the VP-SDE
\begin{equation} \label{eq:forward_SDE}
  \dd \xi_t = -\tfrac12\beta(t)\,\xi_t\,\dd t + \sqrt{\beta(t)}\,\dd w_t,
  \qquad \xi_0\sim\rho_0,\quad t\in[0,1),
\end{equation}
where $w$ is a standard $\R^{m_{\xi}}$-valued Brownian motion \citep[e.g.,][]{song2021score}.
By \citet{anderson1982reverse}, the time reversal of these dynamics is also an SDE, given by
\begin{equation} \label{eq:reverse_SDE}
\dd \xi_t = \beta(t)\big[\tfrac12\xi_t - g(t,\xi_t)\big]\dd t
+ \sqrt{\beta(t)}\,\dd\overline w_t,
\end{equation}
where $\overline w$ is a reverse-time Brownian motion, and this equation is integrated backwards in $t$.
In particular, if \eqref{eq:reverse_SDE} is initialised at
$\xi_{t_1}\sim\rho_{t_1}$ and integrated backwards from $t_1$ to $0$, then $\xi_t\sim\rho_t$ for every $t\in[0,t_1]$, and hence $\xi_0\sim\rho_0$.
Since $\rho_{t_1}\Rightarrow\mathcal N(0,I_{m_\xi})$ as $t_1\uparrow1$, initialising at $\mathcal N(0,I_{m_\xi})$ in place of $\rho_{t_1}$ is asymptotically exact, yielding a practical sampler. 

The same family of marginal densities can also be generated by a deterministic process known as the probability-flow ODE \citep[see][]{song2021score}. 
In practice, however, the SDE is often a preferable generating mechanism; see~\Cref{app:exactness-proofs} for further details.

\subsection{Monte Carlo guidance}
\label{sec:mc-guidance}

To generate $\xi_0$ using the reverse-time SDE in \eqref{eq:reverse_SDE} we first need \(g(t,u)=\nabla_u\log\psi(t,u)\). 
Typically, neither $g(t,u)$ nor $\psi(t,u)$ are available in closed form.
We can, however, obtain estimates of both via a Monte Carlo estimator of \eqref{eq:bridge-conditional}.
For a given pair \((t,u)\), draw
\begin{equation}
  \varepsilon^{(s)}\sim \N(0,I_{m_\xi}),
  \qquad
  \xi^{(s)}_{0 \mid t}
  =
  \alpha(t)u+\sigma(t)\varepsilon^{(s)},
  \qquad
  s=1,\ldots,S.
\end{equation}
Then
\(
  \widehat\psi_S(t,u)
  =
  \frac1S\sum_{s=1}^S\Lambda_{\vart, \calC}(\xi^{(s)}_{0 \mid t})
\)
is an unbiased estimator of \(\psi(t,u)\).
If \(\Lambda_{\vart, \calC}\) is differentiable, then
\begin{equation}
  \nabla_u\log\psi(t,u)
  =
  \alpha(t)
  \frac{
  \E_{\varepsilon\sim r}\left[
    \nabla_{\xi}\Lambda_{\vart, \calC}\bigl(\alpha(t)u+\sigma(t)\varepsilon\bigr)
  \right]
  }{
  \E_{\varepsilon\sim r}\left[
    \Lambda_{\vart,\calC}\bigl(\alpha(t)u+\sigma(t)\varepsilon\bigr)
  \right]
  }.
  \label{eq:exact-gradient-ratio}
\end{equation}
When \(\Lambda_{\vart, \calC}>0\), this motivates the self-normalised estimator
\begin{equation}
  \widehat g_S(t,u)
  =
  \alpha(t)\sum_{s=1}^S
  \bar w_s\nabla_{\xi}\log\Lambda_{\vart,\calC}(\xi^{(s)}_{0 \mid t}),
  \qquad
  \bar w_s
  =
  \frac{\Lambda_{\vart,\calC}(\xi^{(s)}_{0 \mid t})}
  {\sum_{\ell=1}^S\Lambda_{\vart,\calC}(\xi^{(\ell)}_{0 \mid t})}.
  \label{eq:mc-guidance}
\end{equation}
Since \(\Lambda_{\vart, \calC}=L_{\calC}\circ T_{\vart}\), the latent score can be computed by differentiating through the process generator; in practice, this is typically available via automatic differentiation.
In particular, when \(T_\vart\) and \(L_{\calC}\) are differentiable and \(L_{\calC}>0\), the chain rule gives
\begin{equation}
  \nabla_{\xi}\log\Lambda_{\vart,\calC}(\xi)
  =
  [D_{\xi}T_\vart(\xi)]^{\top}
  \nabla_{f_0}\log L_{\calC}(f_0)
  \big|_{f_0=T_\vart(\xi)},
  \label{eq:latent-score-chain-rule}
\end{equation}
where $D_{\xi}T_\vart(\xi)$ is the Jacobian of $T_\vart(\xi)$. For hard constraints, discontinuous likelihoods, or non-differentiable simulators, this gradient-based form must be replaced by a smooth relaxation, a differentiable surrogate, or a derivative-free guidance estimator; see~\Cref{sec:derivative-free-guidance}.

In practice, we replace \(g(t,u)\) by \(\widehat g_S(t,u)\) and solve \eqref{eq:reverse_SDE} numerically.
The resulting error decomposes into three terms: terminal initialisation, guidance approximation, and numerical discretisation; see~\Cref{sec:approx-guidance-stability}.
Under a common exact terminal coupling, if the Monte Carlo guidance field is uniformly root-\(S\) consistent and the SDE solver has strong order \(q_{\mathrm{sde}}\), then the coupled latent grid error is
\(\mathcal O_{\mathbb P}(S^{-1/2}+h^{q_{\mathrm{sde}}})\).
The error contribution of each can be made arbitrarily small, so the sampling error is limited only by available compute.
These bounds pass to process space whenever \(T_\vart\) is Lipschitz.

\section{Guidance-aware parameter estimation}
\label{sec:hyperparams}

We have, thus far, treated the parameter $\vart$ as fixed. 
Adapting it based on the information $\calC$ offers further advantages of our training-free approach.
A learned conditional generator or amortised sampler must be retrained as $\vart$ is updated, or trained in advance over a prescribed range of parameter values. 
In \textsc{LatentFlow}, $T_{\vart}$ is a pre-defined generative prior, and $\vart$ enters through the pulled-back likelihood.
Parameter estimation can thus be performed using the same latent conditional samples, without retraining a conditional generator.

The normalising constant \(Z_{\vart}\) in \eqref{eq:process-target} and \eqref{eq:latent-target} can be interpreted as the evidence for \(\calC\) under the prior.
This suggests the empirical Bayes objective
\(
  \mathcal J(\vart)
  =
  \log Z_{\vart}.
\)
This balances satisfaction of $\calC$ with the relative-entropy cost of deforming \(p_{\vart}\) into \(\pi_{\vart}(\cdot\mid\calC)\); see~\Cref{prop:conditional-evidence-gibbs}.
This relative entropy cost also admits a dynamic interpretation as the guidance energy; see~\Cref{thm:girsanov}. 
Thus, parameters are compatible with the condition when they can be realised with high likelihood and low guidance energy.

Given \(\nabla_{\vart} \mathcal J(\vart) =  \nabla_{\vart}\log Z_{\vart}\), or an approximation thereof, we can update parameters using an outer-loop scheme.
Fix \(\vart_0\).
For each \(k\geq0\), run \textsc{LatentFlow} with \(\vart_k\) fixed, and update
\begin{equation}
  \vart_{k+1}
  =
  \vart_k
  +
  \eta_k
  \nabla_{\vart} \mathcal J(\vart)
  \big|_{\vart=\vart_k},
  \label{eq:outer-theta-update-main}
\end{equation}
where \(\eta_k>0\) is the step size.
When $\nabla_{\vart} \mathcal J(\vart)$ is unknown, we require an approximation.
If \(\Lambda_{\vart,\calC}\) is differentiable in \(\vart\) and differentiation and integration are interchangeable, then
\begin{align}
  \nabla_{\vart} \mathcal J(\vart)
  =
  \E_{\xi_0\sim\rho_{\vart}(\cdot\mid\calC)}
  \left[
    \nabla_{\vart}
    \log\Lambda_{\vart,\calC}(\xi_0)
  \right].
  \label{eq:endpoint-theta-gradient}
\end{align}
Thus samples from the latent target provide a direct estimator of the conditional-evidence gradient.
If the likelihood \(L_{\calC}\) is not a function of \(\vart\), then
\begin{equation}
  \nabla_{\vart}
  \log\Lambda_{\vart,\calC}(\xi_0)
  =
  [D_{\vart}T_{\vart}(\xi_0)]^{\top}
  \nabla_{f_0}\log L_{\calC}(f_0)
  \big|_{f_0=T_{\vart}(\xi_0)},
  \label{eq:theta-chain-rule}
\end{equation}
where \(D_{\vart}T_{\vart}(\xi_0)\) is the Jacobian of \(T_{\vart}(\xi_0)\).
The same gradient can also be represented at intermediate noising times. 
In particular, we have that
\begin{equation}
  \nabla_{\vart} \mathcal J(\vart)
  =
  \nabla_{\vart}\log Z_{\vart}
  =
  \E_{u\sim\rho_{\vart,t}(\cdot\mid\calC)}
  \left[
    \nabla_{\vart}
    \log\psi_{\vart}(t,u)
  \right].
  \label{eq:intermediate-gradient-init}
\end{equation}
This intermediate-time gradient can be estimated using the same bridge samples used for guidance; see \Cref{sec:mc-guidance}.
In particular, differentiating the estimator of the conditioning potential with respect to \(\vart\) yields
\begin{equation}
  \widehat a_{\vart,S}(t,u)
  \coloneqq
  \nabla_{\vart}
  \log\widehat\psi_{\vart,S}(t,u)
  =
  \sum_{r=1}^S
  \bar w_r
  \nabla_{\vart}
  \log\Lambda_{\vart,\calC}(\xi^{(r)}_{0 \mid t}),
  \quad
  \bar w_r
  =
  \frac{
    \Lambda_{\vart,\calC}(\xi^{(r)}_{0 \mid t})
  }{
    \sum_{\ell=1}^S
    \Lambda_{\vart,\calC}(\xi^{(\ell)}_{0 \mid t})
  }.
  \label{eq:mc-theta-guidance}
\end{equation}
Thus, the same weighted bridge candidates provide two derivatives: \eqref{eq:mc-guidance} guides the latent state, while \eqref{eq:mc-theta-guidance} adapts the process parameter.
Given \(M\) particles and selected noising times \(t_j\), a practical estimator is
\begin{equation}
  \nabla_{\vart}\mathcal{J}(\vart) \approx
  \frac{1}{MJ}
  \sum_{i=1}^M
  \sum_{j=1}^J
  \widehat a_{\vart,S}
  \bigl(t_j,u_j^{(i)}\bigr).
  \label{eq:time-particle-hypergradient}
\end{equation}
\section{Related work}
\label{sec:related-work}

The closest antecedent to our work is \textsc{FlowGP} \citep{moss2026conditioning}.
Our framework contains \textsc{FlowGP} as a special case, but also applies to a substantially broader class of stochastic processes; see~\Cref{app:innovation-representations}.
More broadly, our approach relates to diffusion models, flow matching, and stochastic interpolants, which generate samples via continuous-time transport \citep{sohl2015deep,ho2020denoising,song2021score,lipman2023flow,albergo2023stochastic}.
Unlike pretrained diffusion-model guidance \citep[e.g.,][]{ho2022classifierfree}, however, \textsc{LatentFlow} does not need to be pre-trained on a neural estimate of an unconditional data score or reverse process.
The noising bridge is imposed in innovation space, and the guidance is the gradient of a smoothed pulled-back likelihood under the known latent bridge.

A closely related line of work conditions generative models by performing inference in the latent noise space \citep{whang2021composing,holden2022bayesian,patel2022solution,achituve2025inverse,askari2025latent,purohit2024posterior,venkatraman2025outsourced,xia2026noise}.
These methods share the idea of conditioning a generator by changing the distribution of its input noise, but focus only on learned neural generators.
Our framework is significantly more general: it includes pre-trained neural generators as a special case, but also extends to a much broader class of genuine stochastic-process priors.
The papers above also rely on different, often more expensive, samplers, for instance, Langevin dynamics \citep{purohit2024posterior}, HMC \citep{patel2022solution,xia2026noise}, pCN \citep{holden2022bayesian}, or SMC \citep{achituve2025inverse}.
\textsc{LatentFlow} instead targets the latent posterior using a guided, training-free diffusion process, whose guidance field is constructed from the pulled-back likelihood.

Our work is also related to many other areas, including diffusion bridges, rare-event simulation, and transport-assisted MCMC; see \Cref{sec:add-related-work} for a more detailed discussion.

\section{Numerical Experiments}
\label{sec:numerics}

We now present a range of numerical experiments designed to assess the flexibility, accuracy, and scalability of \textsc{LatentFlow}. 
We first evaluate \textsc{LatentFlow} on three standard sampling benchmarks; see \Cref{app:sampling-benchmarks} for full details.
Across all three benchmarks, we achieve MMD$^2$ scores competitive with or better than NUTS. 
As \textsc{LatentFlow} generates independent samples, it scales to problems where standard MCMC either collapses modally or explores very slowly. 
In this section, we consider higher-dimensional and more challenging processes. 
All timings are reported on a 13th Gen Intel i9-13900K CPU. Additional results are provided in \Cref{app:numerics-additional}.

\subsection{Spatial Processes}
\label{sec:spatial-processes}
Spatial processes are central to many applications, including environmental modelling, climate science, and ecology \citep[e.g.,][]{cressie1993statistics}.
While GPs provide a convenient prior in such settings, many applications exhibit extremal or heavy tails, discontinuities, sharp local anomalies, or discrete observations.
This motivates alternatives such as Student-\(t\) processes, L\'evy random fields, and Potts models \citep{potts1952some,besag1974spatial,rajput1989spectral,wolpert2011levy,shah2014student}.
The difficulty here has two sources. 
The prior alone can break conjugacy as non-Gaussian priors induce conditional laws involving scale mixtures or jump-like innovations \citep{walchessen2025neural}. 
Additionally, non-linear conditioning statements such as areal maxima, threshold exceedances, or censored measurements do not admit closed-form conditionals, even under a Gaussian prior.

In \Cref{fig:spatial_processes}, we use \textsc{LatentFlow} to condition four spatial priors on four conditions.
The conditioning statements include point observations, regional summaries, and censored observations.
The resulting samples qualitatively respect the imposed conditioning information, whilst also preserving the distinctive behaviour of each prior.
The speed of \textsc{LatentFlow} is notable here: conditional samples are generated in just \(3-4\) seconds per sample.

\begin{figure}[!b]
\vspace{-3mm}
    \centering
    \includegraphics[width=.95\textwidth]{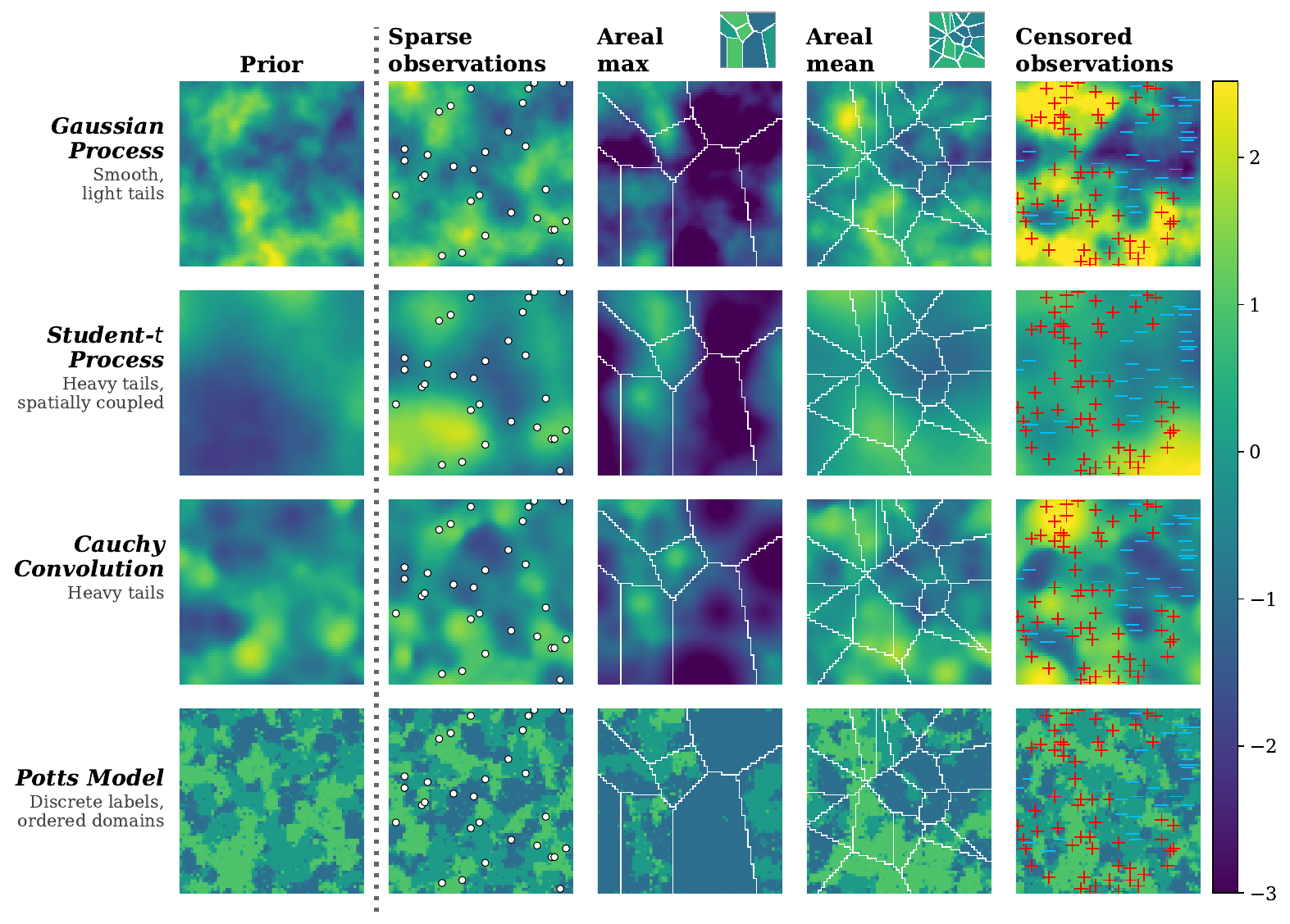}
    \caption{Posterior samples from four spatial process priors under four observation types via \textsc{LatentFlow}. \textbf{Rows:} Gaussian process (Mat\'{e}rn-$\tfrac{3}{2}$); Student-$t$ process ($\nu=4$); Cauchy convolution field; Potts model ($q=3$). \textbf{Columns:} unconditional; 30 point observations; areal maxima (8 regions); areal means (20 regions); 60 censored observations ($+$/$-$). Thumbnails show areal structures. Sampling time: $3-4$s per sample.}
\label{fig:spatial_processes}
\end{figure}

\subsection{Temporal processes}
\label{sec:temporal-processes}

We next consider {temporal processes}.
Examples include diffusion processes \citep[e.g.,][]{oksendal2003stochastic}, and dynamical systems with random initial conditions or exogenous forcing \citep[e.g.,][]{law2015data}.
This presents a different use case for \textsc{LatentFlow}: the object being conditioned is no longer a static field, but a trajectory generated by propagating latent randomness through a dynamical model. 

\textbf{Diffusion bridges.} Diffusion bridges condition diffusion processes on a fixed terminal condition.
They arise in many applications, including computational chemistry \citep{bolhuis2002transition}, econometrics \citep{elerian2001likelihood}, developmental dynamics \citep{wang2011quantifying}, and shape analysis \citep{arnaudon2022diffusion}.
Except in a few tractable cases, the bridge drift depends on an unavailable transition density, motivating bespoke guided proposals and path-space MCMC methods \citep{delyon2006simulation,schauer2017guided}, as well as recent neural approaches that approximate the bridge score or guidance field \citep{baker2024conditioning,heng2021simulating,yang2025neural}.
In \Cref{fig:neural-guided-diffusion-main}, we assess whether \textsc{LatentFlow} can generate diffusion bridges by acting on the exogenous simulation noise.
In our case, exact endpoint conditioning is approximated by a narrow finite-width likelihood.
Our method recovers the analytic OU bridge behaviour, and produces nonlinear cell-differentiation and FitzHugh--Nagumo paths consistent with rejection-sampled references.
It also agrees qualitatively with existing methods \citep[e.g.,][]{yang2025neural} in cases where no valid rejection samples were found after simulating the unconditioned process 100,000 times; see \Cref{fig:cell-differentiation-bridges} in \Cref{app:temporal-processes}.

\begin{figure}[!b]
    \centering
    \begin{subfigure}[t]{0.32\textwidth}
        \centering
        \includegraphics[width=\linewidth]{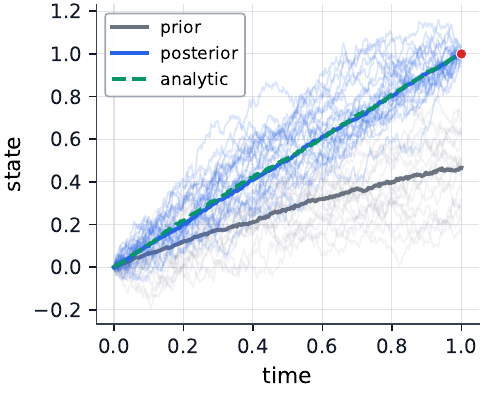}
        \caption{Ornstein--Uhlenbeck.}
        \label{fig:ngd-ou-slow}
    \end{subfigure}
    \hfill
    \begin{subfigure}[t]{0.32\textwidth}
        \centering
        \includegraphics[width=\linewidth]{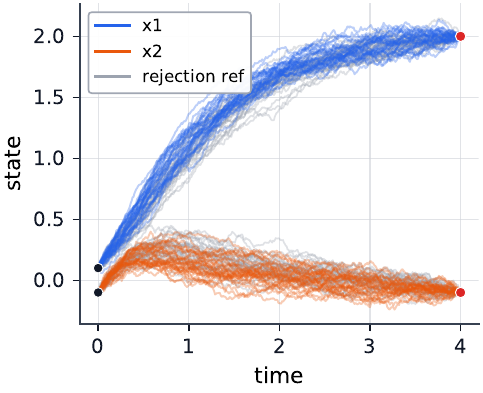}
        \caption{Cell differentiation.}
        \label{fig:ngd-cell-normal}
    \end{subfigure}
    \hfill
    \begin{subfigure}[t]{0.32\textwidth}
        \centering
        \includegraphics[width=\linewidth]{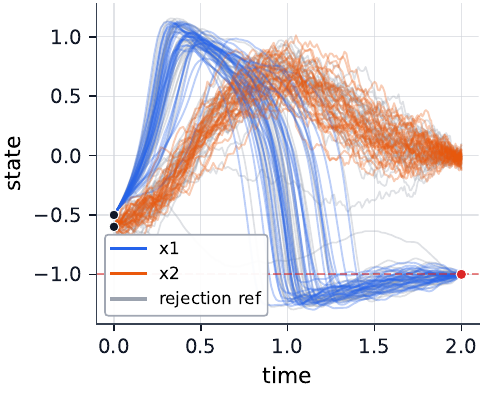}
        \caption{FitzHugh--Nagumo.}
        \label{fig:ngd-fhn-low}
    \end{subfigure}
    \vspace{1mm}
    \caption{
Diffusion bridges generated by \textsc{LatentFlow}.
\textbf{(a)} An Ornstein--Uhlenbeck bridge.
\textbf{(b)} A two-dimensional nonlinear cell-differentiation diffusion conditioned to reach a prescribed terminal state.
\textbf{(c)} A FitzHugh--Nagumo diffusion conditioned on a rare transition between metastable states.
Red markers indicate terminal targets, blue \& orange curves indicate conditioned state components; grey curves indicate prior or rejection-reference trajectories.
}
    \label{fig:neural-guided-diffusion-main}
\end{figure}

\textbf{Trajectory-level conditioning.}
We next consider more general trajectory-level conditioning, where the information may involve observations, terminal events, or pathwise threshold constraints \citep[e.g.,][]{bolhuis2021transition,finzi2023user,grafke2025sampling}.
Such constraints act on the realised trajectory after the latent randomness has been propagated through a nonlinear dynamical simulator, and need not reduce to a standard filtering, smoothing, or bridge-simulation problem.
In \Cref{fig:temporal_processes}, we consider Lotka--Volterra dynamics \citep{lotka1925elements,volterra1926fluctuations}, an SIR epidemic model \citep{kermack1927contribution}, and the Heston stochastic-volatility model \citep{heston1993closed}.
The resulting conditional paths satisfy the imposed trajectory-level information while retaining the characteristic dynamics of each model.
Sampling remains very fast, taking roughly \(1\) second per sample for the SIR and Lotka--Volterra models, and roughly \(3\) seconds per sample for the Heston model.

\begin{figure}[!t]
    \centering
    \includegraphics[width=.95\textwidth]{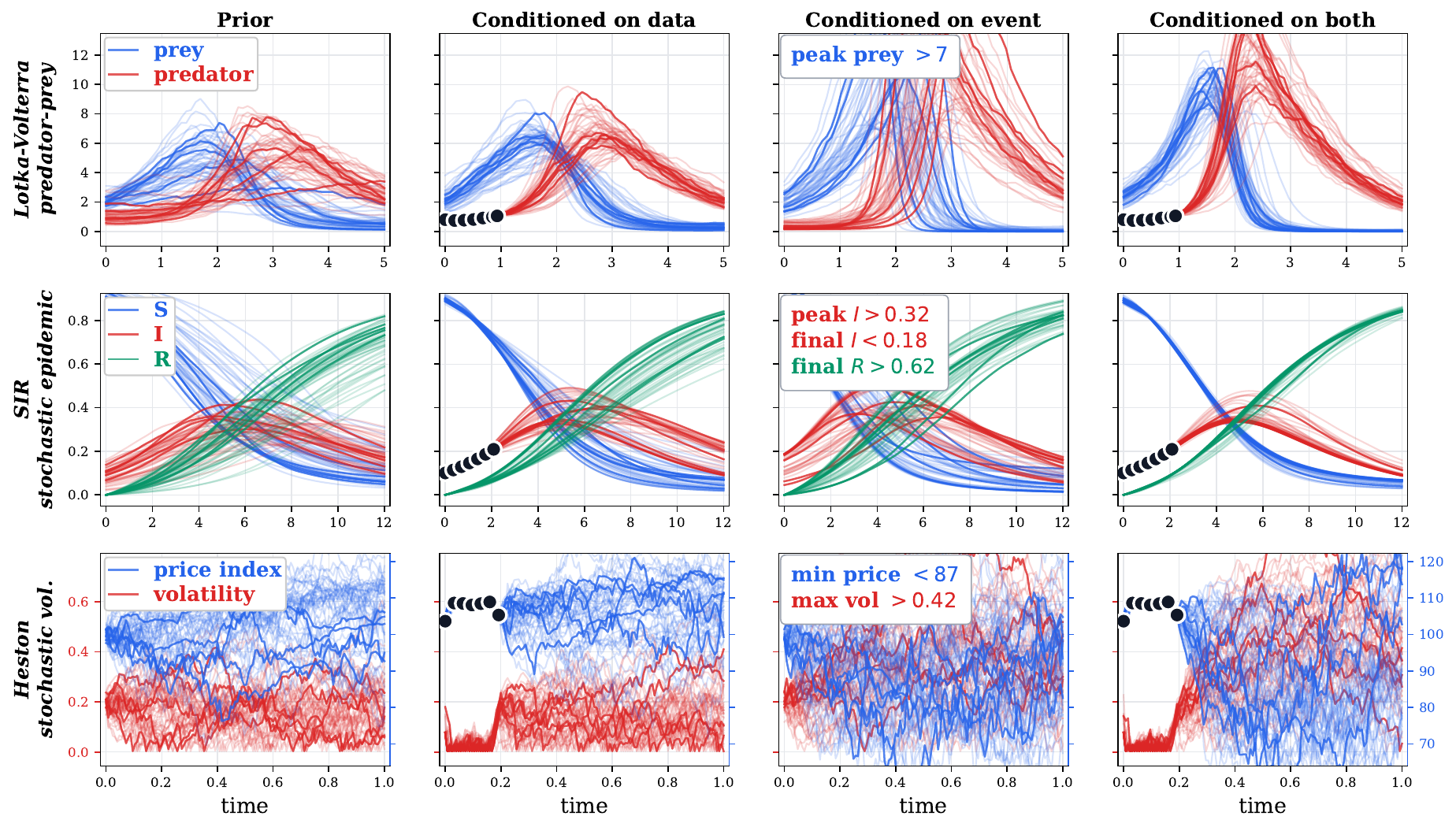}
    \caption{\textsc{LatentFlow} can condition SDE priors on data and rare events. \textbf{Rows}: Lotka--Volterra predator--prey dynamics, a stochastic SIR epidemic, and Heston stochastic volatility. \textbf{Columns}: the prior, the posterior given observations (black circles), the posterior given rare events: a prey boom, a fast growing but fast terminating epidemic, and a stress event (a market crash followed by a volatility spike). Sampling time: $<1$s for  SIR and Lotka--Volterra; $3$s Heston.}
    \label{fig:temporal_processes}
\end{figure}
\subsection{Spatio-Temporal Processes}
\label{sec:spatio-temporal-processes}

We finally consider spatio-temporal models, which describe the random evolution of spatial fields.
Such models arise in turbulent transport, environmental dispersion, materials science, excitable media, and phase separation \citep{majda1999simplified,allen1979microscopic,nagumo1962active}.
Conditioning them on sparse observations, future events, or aggregate space--time constraints is central to data assimilation, inverse modelling, and rare-event prediction \citep{dashti2017bayesian,cerou2007adaptive,cerou2012sequential,botev2020sampling}.

We consider three examples: a stochastic Allen--Cahn field for noisy phase separation, an advection--diffusion plume model, and a stochastic FitzHugh--Nagumo excitable-media model that produces travelling activation waves \citep{fitzhugh1961impulses,nagumo1962active,allen1979microscopic,majda1999simplified}.
Representative samples are shown in~\Cref{ffig:spatiotemp_demo}.
Across all three models, \textsc{LatentFlow} produces fields that remain consistent with the prior dynamics while satisfying the imposed conditions. 

\begin{figure}[!b]
    \centering
    \includegraphics[width=.95\textwidth]{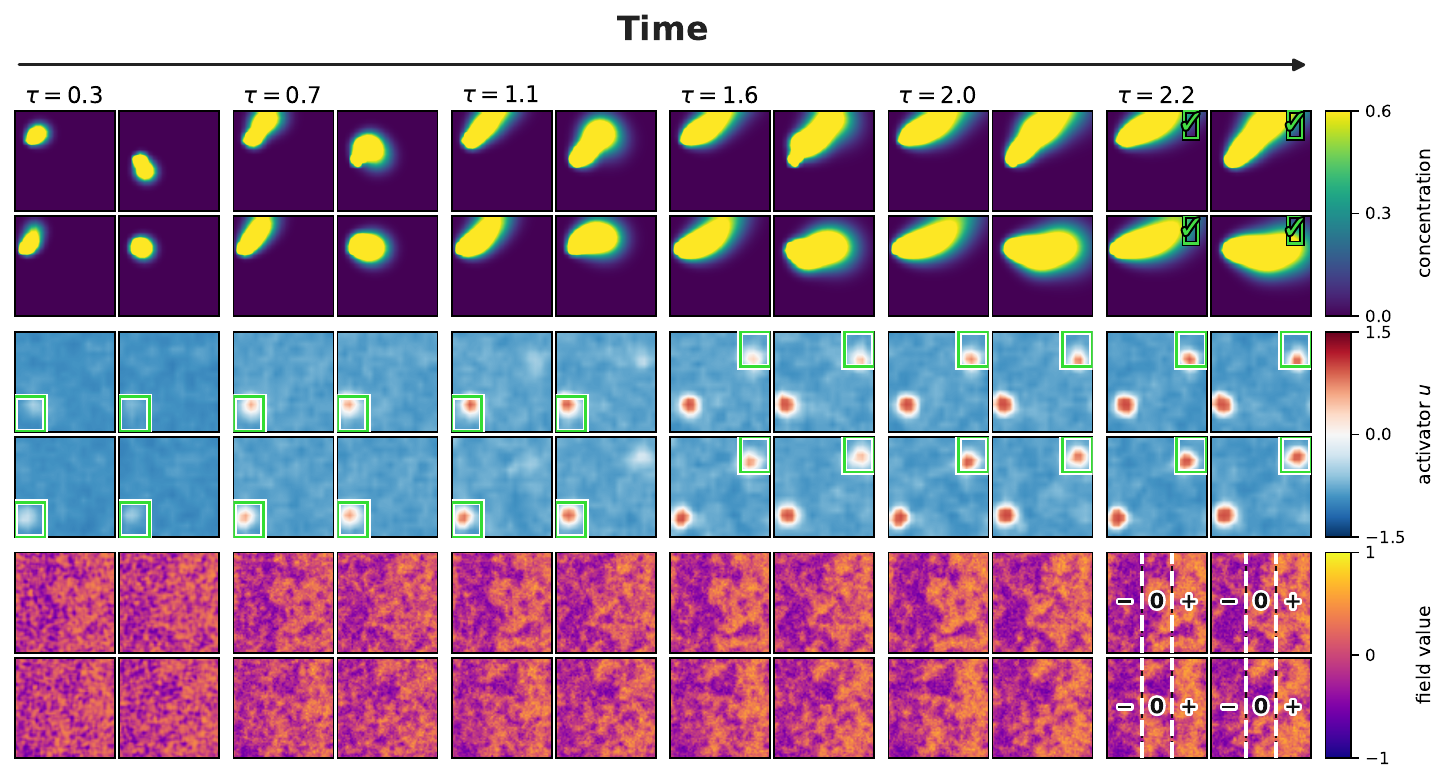}
    \caption{Posterior samples from three spatio-temporal systems using our \textsc{LatentFlow}: \textbf{(top)} a concentration plume conditioned on triggering a downstream detector evolving under advection--diffusion dynamics ($160$ latent dimensions); \textbf{(middle)} a FitzHugh--Nagumo excitable medium conditioned on sequential pacemaker firing ($40$K latent dimensions); and \textbf{(bottom)} an Allen--Cahn phase field conditioned on a prescribed phase front ($122$K latent dimensions). Each row shows four independent draws, over six time snapshots. Sampling time: $1$s for the  advection--diffusion; $5$s for FitzHugh--Nagumo, and $2$s for Allen--Cahn.}
    \label{ffig:spatiotemp_demo}
\end{figure}

\subsection{Parameter Estimation}
\label{sec:gaussian-processes}
Finally, we assess the guidance-aware parameter adaptation framework introduced in \Cref{sec:hyperparams}.
We revisit two of the processes considered in earlier sections, and compare the results obtained with fixed parameters, with the hyperparameter selected after incorporating the full conditioning information.
The results are shown in \Cref{fig:parameter-estimation}. 
Across both examples, our guidance-aware scheme recovers the ground truth, resulting in samples that more closely respect the target observations and constraints.

\begin{figure}[t!]
    \centering
    \begin{subfigure}[b]{0.2\textwidth}
        \includegraphics[height=3.0cm]{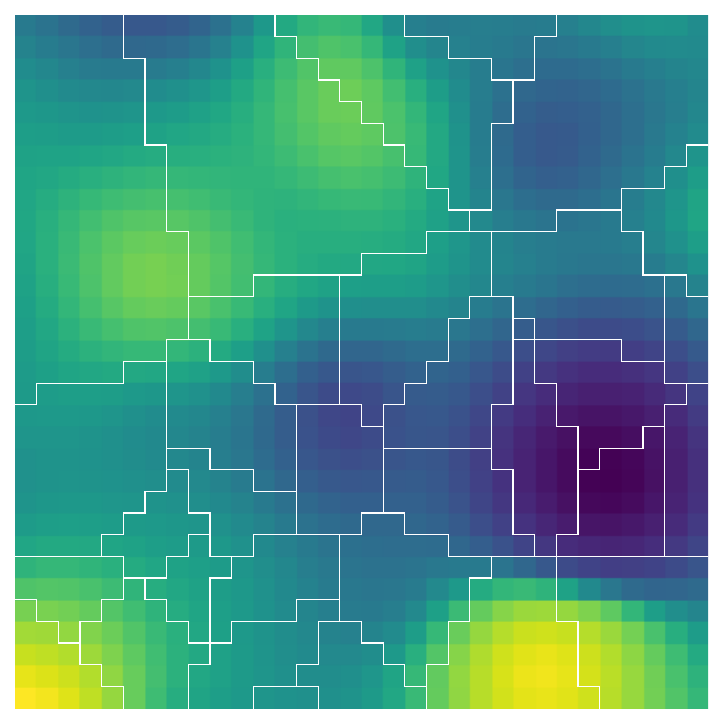}
        \caption{Ground truth}
        \label{fig:sub1}
    \end{subfigure}
    \hfill
    \begin{subfigure}[b]{0.2\textwidth}
        \includegraphics[height=3.0cm]{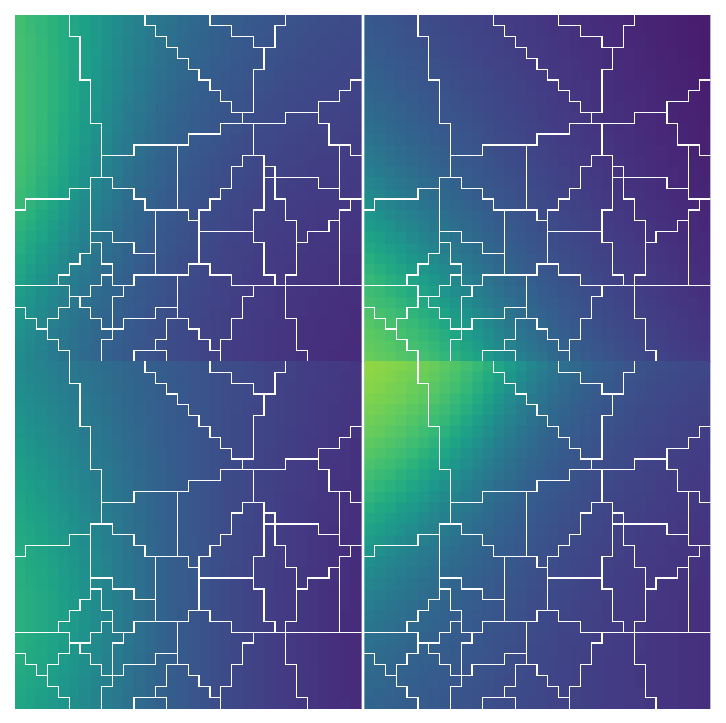}
        \caption{Pre-fit}
        \label{fig:sub2}
    \end{subfigure}
    \hfill
    \begin{subfigure}[b]{0.2\textwidth}
        \includegraphics[height=3.0cm]{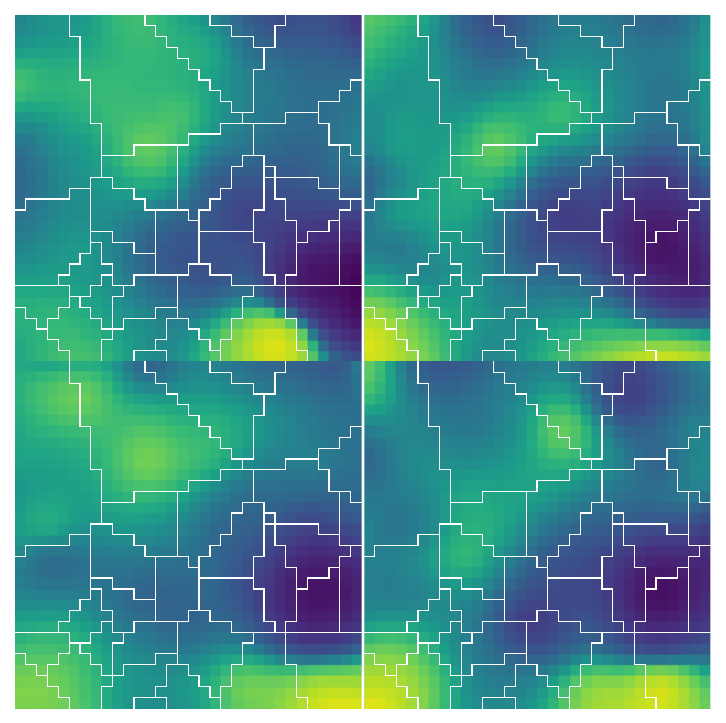}
        \caption{Post-fit}
        \label{fig:sub3}
    \end{subfigure}
    \hfill
    \begin{subfigure}[b]{0.36\textwidth}
        \includegraphics[height=3.0cm]{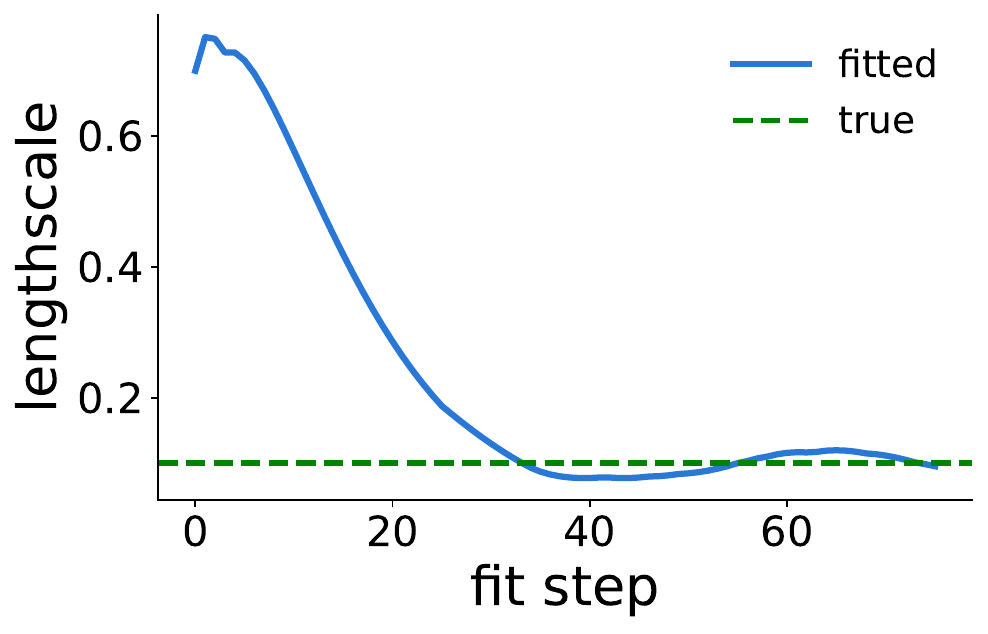}
        \caption{Optimisation trajectory}
        \label{fig:sub4}
    \end{subfigure}

    \vskip\baselineskip
    \begin{subfigure}[b]{0.24\textwidth}
        \includegraphics[height=3.0cm]{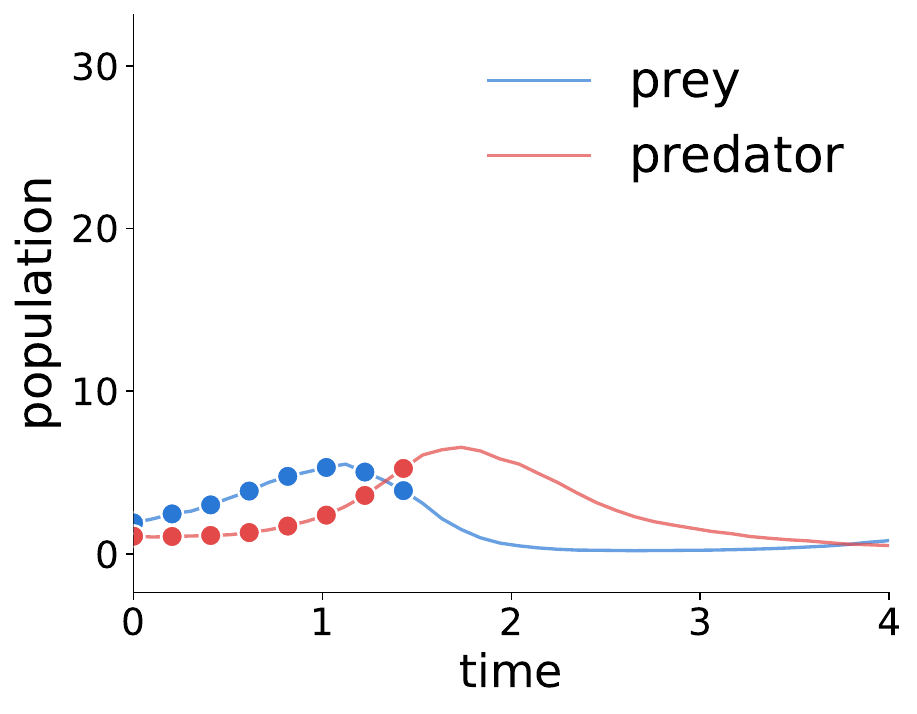}
        \caption{Ground truth}
        \label{fig:sub5}
    \end{subfigure}
    \hfill
    \begin{subfigure}[b]{0.18\textwidth}
        \includegraphics[height=3.0cm]{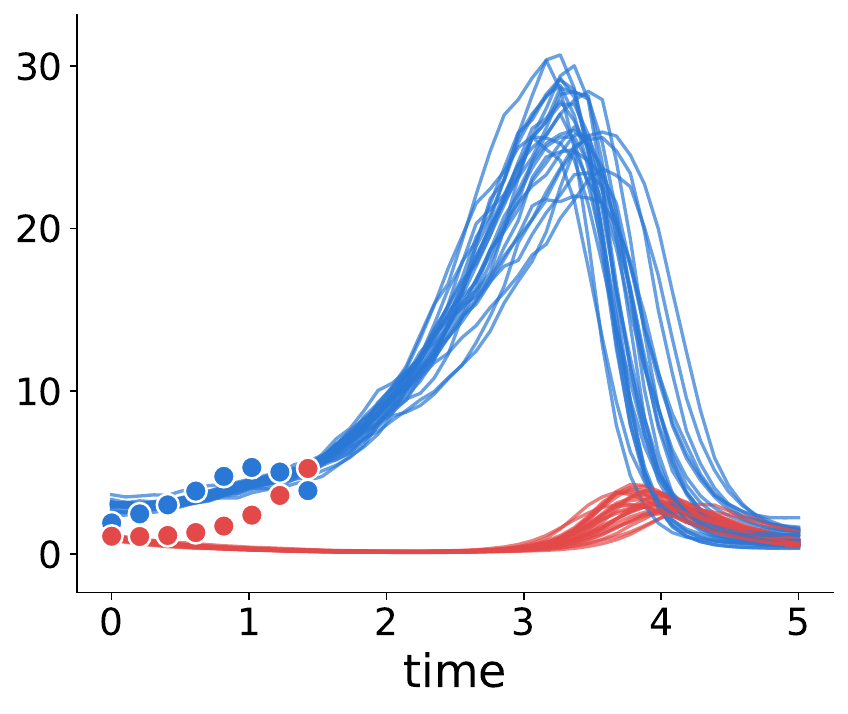}
        \caption{Pre-fit}
        \label{fig:sub6}
    \end{subfigure}
    \hfill
    \begin{subfigure}[b]{0.18\textwidth}
        \includegraphics[height=3.0cm]{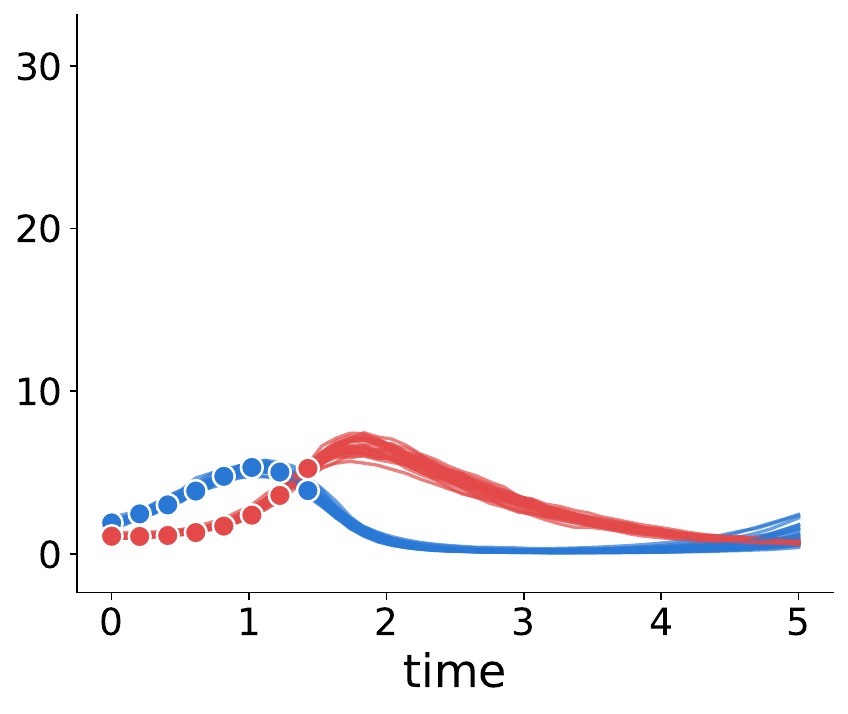}
        \caption{Post-fit}
        \label{fig:sub7}
    \end{subfigure}
    \hfill
    \begin{subfigure}[b]{0.36\textwidth}
        \includegraphics[height=3.0cm]{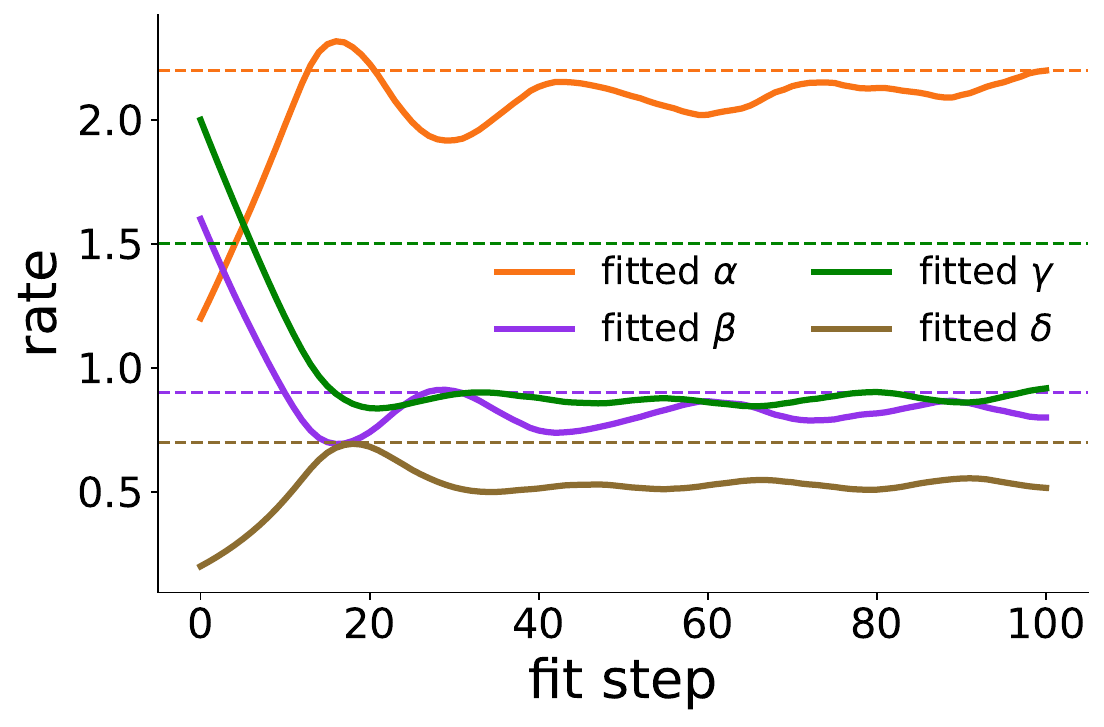}
        \caption{Optimisation trajectory}
        \label{fig:sub8}
    \end{subfigure}
\vspace{2mm}
\caption{
Guidance-aware parameter estimation recovers reasonable process
hyperparameters, modulo identifiability.
\textbf{Top}: Cauchy-convolution field. (a) Truth, observed only via the \emph{maximum} value within each of $K=32$ irregular regions. 
(b) Guided posterior samples under a mis-specified lengthscale ($\ell=0.7$). 
(c) Samples after adapting the lengthscale against the
regional-maximum condition. 
(d) Lengthscale trajectory converging to the true value ($\ell=0.1$, dashed).
\textbf{Bottom}: stochastic Lotka--Volterra process. (e) Truth, with dense early observations (dots). 
(f) Guided posterior samples under mis-specified rates.
(g) Samples after jointly fitting all four rates against the same early-window condition. 
(h) Identifiable rates converge towards true values (dashed).
}
    \label{fig:parameter-estimation}
\end{figure}
\section{Discussion}
\label{sec:conclusions}

\textsc{LatentFlow} replaces the bespoke, model-by-model samplers that have previously been required for conditioning stochastic processes with a single, unified approach. 
The central idea is to move conditioning from process space to a latent innovation space, with the likelihood pulled back through the process generator and the resulting conditional law sampled using guided reverse-time dynamics.
This gives an exact target-law construction, with practical approximation confined to three explicit and systematically reducible sources: finite terminal noising, Monte Carlo guidance, and time discretisation.
This is a fundamental shift in capability.
We expect the largest gains for rare, nonlinear, or global conditioning events under non-Gaussian and simulator-defined process priors, including extreme events, pathwise conditions, and constraints on SDE and SPDE driven models.

Several limitations remain.
\textsc{LatentFlow} requires a \emph{tractable} latent innovation representation of the discretised prior, and its performance can depend on the dimension, smoothness, and geometry of this representation.
It also requires pointwise evaluation of the pulled-back likelihood, and ideally that the likelihood is differentiable through the generator.
Discontinuous simulators, hard constraints, rare events, sharp likelihoods, multimodal conditionals, or poorly scaled latent parameterisations may therefore require relaxations, surrogates, derivative-free guidance, or larger Monte Carlo budgets.

\bibliographystyle{iclr2026_conference}
\bibliography{refs}

\begin{thebibliography}{174}
\providecommand{\natexlab}[1]{#1}
\providecommand{\url}[1]{\texttt{#1}}
\expandafter\ifx\csname urlstyle\endcsname\relax
  \providecommand{\doi}[1]{doi: #1}\else
  \providecommand{\doi}{doi: \begingroup \urlstyle{rm}\Url}\fi

\bibitem[Achituve et~al.(2025)Achituve, Habi, Rosenfeld, Netzer, Diamant, and
  Fetaya]{achituve2025inverse}
Idan Achituve, Hai~Victor Habi, Amir Rosenfeld, Arnon Netzer, Idit Diamant, and
  Ethan Fetaya.
\newblock Inverse problem sampling in latent space using sequential {Monte
  Carlo}.
\newblock In \emph{Proceedings of the 42nd International Conference on Machine
  Learning}, pp.\  420--443, 2025.

\bibitem[Agrell(2019)]{agrell2019gaussian}
Christian Agrell.
\newblock {Gaussian} processes with linear operator inequality constraints.
\newblock \emph{Journal of Machine Learning Research}, 20\penalty0
  (135):\penalty0 1--36, 2019.

\bibitem[Ajalloeian \& Stich(2020)Ajalloeian and
  Stich]{ajalloeian2020convergence}
Ahmad Ajalloeian and Sebastian~U Stich.
\newblock On the convergence of sgd with biased gradients.
\newblock \emph{arXiv preprint arXiv:2008.00051}, 2020.

\bibitem[Albergo \& Vanden-Eijnden(2023)Albergo and
  Vanden-Eijnden]{albergo2023building}
Michael~S. Albergo and Eric Vanden-Eijnden.
\newblock Building normalizing flows with stochastic interpolants.
\newblock In \emph{Proceedings of the 11th International Conference on Learning
  Representations}, 2023.

\bibitem[Albergo et~al.(2025)Albergo, Boffi, and
  Vanden-Eijnden]{albergo2023stochastic}
Michael~S. Albergo, Nicholas~M. Boffi, and Eric Vanden-Eijnden.
\newblock Stochastic interpolants: A unifying framework for flows and
  diffusions.
\newblock \emph{Journal of Machine Learning Research}, 26\penalty0
  (209):\penalty0 1--80, 2025.

\bibitem[Allen \& Cahn(1979)Allen and Cahn]{allen1979microscopic}
Samuel~M. Allen and John~W. Cahn.
\newblock A microscopic theory for antiphase boundary motion and its
  application to antiphase domain coarsening.
\newblock \emph{Acta Metallurgica}, 27\penalty0 (6):\penalty0 1085--1095, 1979.

\bibitem[{\'A}lvarez et~al.(2009){\'A}lvarez, Luengo, and
  Lawrence]{alvarez2009latent}
Mauricio {\'A}lvarez, David Luengo, and Neil~D. Lawrence.
\newblock Latent force models.
\newblock In \emph{Proceedings of The 12th International Conference on
  Artificial Intelligence and Statistics}, pp.\  9--16, 2009.

\bibitem[{\'A}lvarez et~al.(2013){\'A}lvarez, Luengo, and
  Lawrence]{alvarez2013linear}
Mauricio~A. {\'A}lvarez, David Luengo, and Neil~D. Lawrence.
\newblock Linear latent force models using {Gaussian} processes.
\newblock \emph{IEEE Transactions on Pattern Analysis and Machine
  Intelligence}, 35\penalty0 (11):\penalty0 2693--2705, 2013.

\bibitem[Anderson(1982)]{anderson1982reverse}
Brian D.~O. Anderson.
\newblock Reverse-time diffusion equation models.
\newblock \emph{Stochastic Processes and their Applications}, 12\penalty0
  (3):\penalty0 313--326, 1982.

\bibitem[Andrieu et~al.(2003)Andrieu, de~Freitas, Doucet, and
  Jordan]{andrieu2003introduction}
Christophe Andrieu, Nando de~Freitas, Arnaud Doucet, and Michael~I. Jordan.
\newblock An introduction to {MCMC} for machine learning.
\newblock \emph{Machine Learning}, 50\penalty0 (1--2):\penalty0 5--43, 2003.

\bibitem[Arbel et~al.(2021)Arbel, Matthews, and Doucet]{arbel2021annealed}
Michael Arbel, Alexander G. D.~G. Matthews, and Arnaud Doucet.
\newblock Annealed flow transport {Monte Carlo}.
\newblock In \emph{Proceedings of the 38th International Conference on Machine
  Learning}, pp.\  318--330, 2021.

\bibitem[Arnaudon et~al.(2022)Arnaudon, {van der Meulen}, Schauer, and
  Sommer]{arnaudon2022diffusion}
Alexis Arnaudon, Frank {van der Meulen}, Moritz Schauer, and Stefan Sommer.
\newblock Diffusion bridges for stochastic hamiltonian systems and shape
  evolutions.
\newblock \emph{SIAM Journal on Imaging Sciences}, 15\penalty0 (1):\penalty0
  293--323, 2022.

\bibitem[Askari et~al.(2025)Askari, Luo, Sun, and Roosta]{askari2025latent}
Hossein Askari, Yadan Luo, Hongfu Sun, and Fred Roosta.
\newblock Latent refinement via flow matching for training-free linear inverse
  problem solving.
\newblock In \emph{Advances in Neural Information Processing Systems},
  volume~38, 2025.

\bibitem[Astfalck et~al.(2018)Astfalck, Sen, Patra, Cripps, and
  Dunson]{astfalck2026posterior}
Lachlan Astfalck, Deborshee Sen, Sayan Patra, Edward Cripps, and David Dunson.
\newblock Posterior projection for inference in constrained spaces.
\newblock \emph{arXiv preprint arXiv:1812.05741}, 2018.

\bibitem[Astfalck et~al.(2024)Astfalck, Bird, and
  Williamson]{astfalck2024generalised}
Lachlan Astfalck, Cassandra Bird, and Daniel Williamson.
\newblock Generalised {B}ayes linear inference.
\newblock \emph{arXiv preprint arXiv:2405.14145}, 2024.

\bibitem[Baker et~al.(2026)Baker, Denker, and Frellsen]{baker2026supervised}
Elizabeth~L. Baker, Alexander Denker, and Jes Frellsen.
\newblock Supervised guidance training for infinite-dimensional diffusion
  models.
\newblock In \emph{Proceedings of the 43rd International Conference on Machine
  Learning}, 2026.

\bibitem[Baker et~al.(2024)Baker, Yang, Severinsen, Hipsley, and
  Sommer]{baker2024conditioning}
Elizabeth~Louise Baker, Gefan Yang, Michael~Lind Severinsen, Christy~Anna
  Hipsley, and Stefan Sommer.
\newblock Conditioning non-linear and infinite-dimensional diffusion processes.
\newblock In \emph{Advances in Neural Information Processing Systems},
  volume~37, pp.\  10801--10826, 2024.

\bibitem[Baker et~al.(2025)Baker, Schauer, and Sommer]{baker2025score}
Elizabeth~Louise Baker, Moritz Schauer, and Stefan Sommer.
\newblock Score matching for bridges without learning time-reversals.
\newblock In \emph{Proceedings of The 28th International Conference on
  Artificial Intelligence and Statistics}, pp.\  775--783, 2025.

\bibitem[Baldassari et~al.(2023)Baldassari, Siahkoohi, Garnier, S{\o}lna, and
  de~Hoop]{baldassari2023conditional}
Lorenzo Baldassari, Ali Siahkoohi, Josselin Garnier, Knut S{\o}lna, and
  Maarten~V. de~Hoop.
\newblock Conditional score-based diffusion models for {Bayesian} inference in
  infinite dimensions.
\newblock In \emph{Advances in Neural Information Processing Systems},
  volume~36, pp.\  24262--24290, 2023.

\bibitem[Baldassari et~al.(2024)Baldassari, Siahkoohi, Garnier, S{\o}lna, and
  de~Hoop]{baldassari2024taming}
Lorenzo Baldassari, Ali Siahkoohi, Josselin Garnier, Knut S{\o}lna, and
  Maarten~V. de~Hoop.
\newblock Taming score-based diffusion priors for infinite-dimensional
  nonlinear inverse problems.
\newblock \emph{arXiv preprint arXiv:2405.15676}, 2024.

\bibitem[Baldassari et~al.(2025)Baldassari, Garnier, S{\o}lna, and
  de~Hoop]{baldassari2025preconditioned}
Lorenzo Baldassari, Josselin Garnier, Knut S{\o}lna, and Maarten~V. de~Hoop.
\newblock Preconditioned {Langevin} dynamics with score-based generative models
  for infinite-dimensional linear {Bayesian} inverse problems.
\newblock In \emph{Advances in Neural Information Processing Systems},
  volume~38, 2025.

\bibitem[Bansal et~al.(2024)Bansal, Chu, Schwarzschild, Sengupta, Goldblum,
  Geiping, and Goldstein]{bansal2023universal}
Arpit Bansal, Hong-Min Chu, Avi Schwarzschild, Roni Sengupta, Micah Goldblum,
  Jonas Geiping, and Tom Goldstein.
\newblock Universal guidance for diffusion models.
\newblock In \emph{Proceedings of the 12th International Conference on Learning
  Representations}, 2024.

\bibitem[Bernton et~al.(2019)Bernton, Heng, Doucet, and
  Jacob]{bernton2019schrodinger}
Espen Bernton, Jeremy Heng, Arnaud Doucet, and Pierre~E. Jacob.
\newblock Schr{\"o}dinger bridge samplers.
\newblock \emph{arXiv preprint arXiv:1912.13170}, 2019.

\bibitem[Besag(1974)]{besag1974spatial}
Julian Besag.
\newblock Spatial interaction and the statistical analysis of lattice systems.
\newblock \emph{Journal of the Royal Statistical Society: Series B
  (Methodological)}, 36\penalty0 (2):\penalty0 192--225, 1974.

\bibitem[Beskos et~al.(2011)Beskos, Pinski, Sanz-Serna, and
  Stuart]{beskos2011hybrid}
Alexandros Beskos, Frank~J. Pinski, Jesus~Maria Sanz-Serna, and Andrew~M.
  Stuart.
\newblock Hybrid {Monte Carlo} on hilbert spaces.
\newblock \emph{Stochastic Processes and their Applications}, 121\penalty0
  (10):\penalty0 2201--2230, 2011.

\bibitem[Bierkens et~al.(2020)Bierkens, van~der Meulen, and
  Schauer]{bierkens2020simulation}
Joris Bierkens, Frank van~der Meulen, and Moritz Schauer.
\newblock Simulation of elliptic and hypo-elliptic conditional diffusions.
\newblock \emph{Advances in Applied Probability}, 52\penalty0 (1):\penalty0
  173--212, 2020.

\bibitem[Bladt \& S{\o}rensen(2014)Bladt and S{\o}rensen]{bladt2014simple}
Mogens Bladt and Michael S{\o}rensen.
\newblock Simple simulation of diffusion bridges with application to likelihood
  inference for diffusions.
\newblock \emph{Bernoulli}, 20\penalty0 (2):\penalty0 645--675, 2014.

\bibitem[Bladt et~al.(2016)Bladt, Finch, and S{\o}rensen]{bladt2016simulation}
Mogens Bladt, Samuel Thomas~John Finch, and Michael S{\o}rensen.
\newblock Simulation of multivariate diffusion bridges.
\newblock \emph{Journal of the Royal Statistical Society: Series B (Statistical
  Methodology)}, 78\penalty0 (2):\penalty0 343--369, 2016.

\bibitem[Blei et~al.(2017)Blei, Kucukelbir, and McAuliffe]{blei2017variational}
David~M. Blei, Alp Kucukelbir, and Jon~D. McAuliffe.
\newblock Variational inference: A review for statisticians.
\newblock \emph{Journal of the American Statistical Association}, 112\penalty0
  (518):\penalty0 859--877, 2017.

\bibitem[B{\"o}hm et~al.(2019)B{\"o}hm, Lanusse, and
  Seljak]{bohm2019uncertainty}
Vanessa B{\"o}hm, Fran{\c{c}}ois Lanusse, and Uro{\v{s}} Seljak.
\newblock Uncertainty quantification with generative models.
\newblock \emph{arXiv preprint arXiv:1910.10046}, 2019.

\bibitem[Bolhuis \& Swenson(2021)Bolhuis and Swenson]{bolhuis2021transition}
Peter~G. Bolhuis and David W.~H. Swenson.
\newblock Transition path sampling as {Markov} chain {Monte Carlo} of
  trajectories: Recent algorithms, software, applications, and future outlook.
\newblock \emph{Advanced Theory and Simulations}, 4\penalty0 (4):\penalty0
  2000237, 2021.

\bibitem[Bolhuis et~al.(2002)Bolhuis, Chandler, Dellago, and
  Geissler]{bolhuis2002transition}
Peter~G. Bolhuis, David Chandler, Christoph Dellago, and Phillip~L. Geissler.
\newblock Transition path sampling: Throwing ropes over rough mountain passes,
  in the dark.
\newblock \emph{Annual Review of Physical Chemistry}, 53:\penalty0 291--318,
  2002.

\bibitem[Botev \& L'Ecuyer(2020)Botev and L'Ecuyer]{botev2020sampling}
Zdravko~I. Botev and Pierre L'Ecuyer.
\newblock Sampling conditionally on a rare event via generalized splitting.
\newblock \emph{INFORMS Journal on Computing}, 32\penalty0 (4):\penalty0
  986--995, 2020.

\bibitem[Cabezas \& Nemeth(2023)Cabezas and Nemeth]{cabezas2023transport}
Alberto Cabezas and Christopher Nemeth.
\newblock Transport elliptical slice sampling.
\newblock In \emph{Proceedings of The 26th International Conference on
  Artificial Intelligence and Statistics}, pp.\  3664--3676, 2023.

\bibitem[Cabezas et~al.(2024)Cabezas, Sharrock, and
  Nemeth]{cabezas2024markovian}
Alberto Cabezas, Louis Sharrock, and Christopher Nemeth.
\newblock Markovian flow matching: Accelerating {MCMC} with continuous
  normalizing flows.
\newblock In \emph{Advances in Neural Information Processing Systems},
  volume~37, pp.\  104383--104411, 2024.

\bibitem[Cardoso et~al.(2024)Cardoso, Janati El~Idrissi, Le~Corff, and
  Moulines]{cardoso2024monte}
Gabriel Cardoso, Yazid Janati El~Idrissi, Sylvain Le~Corff, and Eric Moulines.
\newblock {Monte Carlo} guided denoising diffusion models for {Bayesian} linear
  inverse problems.
\newblock In \emph{Proceedings of the 12th International Conference on Learning
  Representations}, 2024.

\bibitem[C{\'e}rou \& Guyader(2007)C{\'e}rou and Guyader]{cerou2007adaptive}
Fr{\'e}d{\'e}ric C{\'e}rou and Arnaud Guyader.
\newblock Adaptive multilevel splitting for rare event analysis.
\newblock \emph{Stochastic Analysis and Applications}, 25\penalty0
  (2):\penalty0 417--443, 2007.

\bibitem[C{\'e}rou et~al.(2012)C{\'e}rou, Del~Moral, Furon, and
  Guyader]{cerou2012sequential}
Fr{\'e}d{\'e}ric C{\'e}rou, Pierre Del~Moral, Teddy Furon, and Arnaud Guyader.
\newblock Sequential {Monte Carlo} for rare event estimation.
\newblock \emph{Statistics and Computing}, 22\penalty0 (3):\penalty0 795--808,
  2012.

\bibitem[Chen et~al.(2021)Chen, Hosseini, Owhadi, and Stuart]{chen2021solving}
Yifan Chen, Bamdad Hosseini, Houman Owhadi, and Andrew~M. Stuart.
\newblock Solving and learning nonlinear {PDE}s with {Gaussian} processes.
\newblock \emph{Journal of Computational Physics}, 447:\penalty0 110668, 2021.

\bibitem[Chen et~al.(2016)Chen, Georgiou, and Pavon]{chen2016relation}
Yongxin Chen, Tryphon~T. Georgiou, and Michele Pavon.
\newblock On the relation between optimal transport and {Schr{\"o}dinger}
  bridges: A stochastic control viewpoint.
\newblock \emph{Journal of Optimization Theory and Applications}, 169\penalty0
  (2):\penalty0 671--691, 2016.

\bibitem[Chung et~al.(2023)Chung, Kim, McCann, Klasky, and
  Ye]{chung2023diffusion}
Hyungjin Chung, Jeongsol Kim, Michael~T. McCann, Marc~L. Klasky, and Jong~Chul
  Ye.
\newblock Diffusion posterior sampling for general noisy inverse problems.
\newblock In \emph{Proceedings of the 11th International Conference on Learning
  Representations}, 2023.

\bibitem[Corenflos et~al.(2025)Corenflos, Zhao, Sch{\"o}n, S{\"a}rkk{\"a}, and
  Sj{\"o}lund]{corenflos2025conditioning}
Adrien Corenflos, Zheng Zhao, Thomas~B. Sch{\"o}n, Simo S{\"a}rkk{\"a}, and
  Jens Sj{\"o}lund.
\newblock Conditioning diffusion models by explicit forward-backward bridging.
\newblock In \emph{Proceedings of The 28th International Conference on
  Artificial Intelligence and Statistics}, pp.\  3709--3717, 2025.

\bibitem[Cotter et~al.(2013)Cotter, Roberts, Stuart, and White]{cotter2013mcmc}
Simon~L. Cotter, Gareth~O. Roberts, Andrew~M. Stuart, and David White.
\newblock {MCMC} methods for functions: Modifying old algorithms to make them
  faster.
\newblock \emph{Statistical Science}, 28\penalty0 (3):\penalty0 424--446, 2013.

\bibitem[Cressie(1993)]{cressie1993statistics}
Noel A.~C. Cressie.
\newblock \emph{Statistics for Spatial Data}.
\newblock John Wiley \& Sons, New York, revised edition, 1993.
\newblock ISBN 9780471002550.

\bibitem[Da~Veiga \& Marrel(2012)Da~Veiga and Marrel]{daveiga2012gaussian}
S{\'e}bastien Da~Veiga and Amandine Marrel.
\newblock {Gaussian} process modeling with inequality constraints.
\newblock \emph{Annales de la Facult{\'e} des Sciences de Toulouse:
  Math{\'e}matiques}, 21\penalty0 (3):\penalty0 529--555, 2012.

\bibitem[Damianou \& Lawrence(2013)Damianou and Lawrence]{damianou2013deep}
Andreas Damianou and Neil~D. Lawrence.
\newblock Deep {Gaussian} processes.
\newblock In \emph{Proceedings of The 16th International Conference on
  Artificial Intelligence and Statistics}, pp.\  207--215, 2013.

\bibitem[Dashti \& Stuart(2017)Dashti and Stuart]{dashti2017bayesian}
Masoumeh Dashti and Andrew~M. Stuart.
\newblock The {Bayesian} approach to inverse problems.
\newblock In Roger Ghanem, David Higdon, and Houman Owhadi (eds.),
  \emph{Handbook of Uncertainty Quantification}, pp.\  311--428. Springer,
  Cham, 2017.

\bibitem[De~Bortoli et~al.(2021)De~Bortoli, Thornton, Heng, and
  Doucet]{debortoli2021diffusion}
Valentin De~Bortoli, James Thornton, Jeremy Heng, and Arnaud Doucet.
\newblock Diffusion {Schr{\"o}dinger} bridge with applications to score-based
  generative modeling.
\newblock In \emph{Advances in Neural Information Processing Systems},
  volume~34, pp.\  17695--17709, 2021.

\bibitem[Del~Moral \& Garnier(2005)Del~Moral and
  Garnier]{delmoral2005genealogical}
Pierre Del~Moral and Josselin Garnier.
\newblock Genealogical particle analysis of rare events.
\newblock \emph{The Annals of Applied Probability}, 15\penalty0 (4):\penalty0
  2496--2534, 2005.

\bibitem[Del~Moral et~al.(2006)Del~Moral, Doucet, and
  Jasra]{delmoral2006sequential}
Pierre Del~Moral, Arnaud Doucet, and Ajay Jasra.
\newblock Sequential {Monte Carlo} samplers.
\newblock \emph{Journal of the Royal Statistical Society: Series B (Statistical
  Methodology)}, 68\penalty0 (3):\penalty0 411--436, 2006.

\bibitem[Dellago et~al.(1998)Dellago, Bolhuis, Csajka, and
  Chandler]{dellago1998transition}
Christoph Dellago, Peter~G. Bolhuis, Felix~S. Csajka, and David Chandler.
\newblock Transition path sampling and the calculation of rate constants.
\newblock \emph{The Journal of Chemical Physics}, 108\penalty0 (5):\penalty0
  1964--1977, 1998.

\bibitem[Dellago et~al.(2002)Dellago, Bolhuis, and
  Geissler]{dellago2002transition}
Christoph Dellago, Peter~G. Bolhuis, and Phillip~L. Geissler.
\newblock Transition path sampling.
\newblock In Ilya Prigogine and Stuart~A. Rice (eds.), \emph{Advances in
  Chemical Physics}, volume 123, pp.\  1--78. John Wiley \& Sons, 2002.

\bibitem[Delyon \& Hu(2006)Delyon and Hu]{delyon2006simulation}
Bernard Delyon and Ying Hu.
\newblock Simulation of conditioned diffusion and application to parameter
  estimation.
\newblock \emph{Stochastic Processes and their Applications}, 116\penalty0
  (11):\penalty0 1660--1675, 2006.

\bibitem[Denker et~al.(2024)Denker, Vargas, Padhy, Didi, Mathis, Dutordoir,
  Barbano, Mathieu, Komorowska, and Li{\`o}]{denker2024deft}
Alexander Denker, Francisco Vargas, Shreyas Padhy, Kieran Didi, Simon Mathis,
  Vincent Dutordoir, Riccardo Barbano, Emile Mathieu, Urszula~Julia Komorowska,
  and Pietro Li{\`o}.
\newblock {DEFT}: Efficient fine-tuning of diffusion models by learning the
  generalised {$h$}-transform.
\newblock In \emph{Advances in Neural Information Processing Systems},
  volume~37, pp.\  19636--19682, 2024.

\bibitem[Dhariwal \& Nichol(2021)Dhariwal and Nichol]{dhariwal2021diffusion}
Prafulla Dhariwal and Alexander~Quinn Nichol.
\newblock Diffusion models beat {GAN}s on image synthesis.
\newblock In \emph{Advances in Neural Information Processing Systems},
  volume~34, pp.\  8780--8794, 2021.

\bibitem[Didi et~al.(2023)Didi, Vargas, Mathis, Dutordoir, Mathieu, Komorowska,
  and Li{\`o}]{didi2024framework}
Kieran Didi, Francisco Vargas, Simon~V. Mathis, Vincent Dutordoir, Emile
  Mathieu, Urszula~J. Komorowska, and Pietro Li{\`o}.
\newblock A framework for conditional diffusion modelling with applications in
  motif scaffolding for protein design.
\newblock In \emph{NeurIPS 2023 Workshop AI4D3}, 2023.

\bibitem[Donsker \& Varadhan(1975)Donsker and Varadhan]{donsker1975asymptotic}
Monroe~D. Donsker and Sathamangalam R.~S. Varadhan.
\newblock Asymptotic evaluation of certain {Markov} process expectations for
  large time, {I}.
\newblock \emph{Communications on Pure and Applied Mathematics}, 28\penalty0
  (1):\penalty0 1--47, 1975.

\bibitem[Doob(1957)]{doob1957conditional}
Joseph~L. Doob.
\newblock Conditional {Brownian} motion and the boundary limits of harmonic
  functions.
\newblock \emph{Bulletin de la Soci{\'e}t{\'e} Math{\'e}matique de France},
  85:\penalty0 431--458, 1957.

\bibitem[Doucet et~al.(2001)Doucet, de~Freitas, and
  Gordon]{doucet2001sequential}
Arnaud Doucet, Nando de~Freitas, and Neil Gordon (eds.).
\newblock \emph{Sequential {Monte Carlo} Methods in Practice}.
\newblock Information Science and Statistics. Springer, New York, NY, 2001.
\newblock ISBN 978-0-387-95146-1.

\bibitem[Dutordoir et~al.(2023)Dutordoir, Saul, Ghahramani, and
  Simpson]{dutordoir2023neural}
Vincent Dutordoir, Alan Saul, Zoubin Ghahramani, and Fergus Simpson.
\newblock Neural diffusion processes.
\newblock In \emph{Proceedings of the 40th International Conference on Machine
  Learning}, pp.\  8990--9012, 2023.

\bibitem[El~Moselhy \& Marzouk(2012)El~Moselhy and
  Marzouk]{elmoselhy2012bayesian}
Tarek~A. El~Moselhy and Youssef~M. Marzouk.
\newblock {Bayesian} inference with optimal maps.
\newblock \emph{Journal of Computational Physics}, 231\penalty0 (23):\penalty0
  7815--7850, 2012.

\bibitem[Elerian et~al.(2001)Elerian, Chib, and
  Shephard]{elerian2001likelihood}
Ola Elerian, Siddhartha Chib, and Neil Shephard.
\newblock Likelihood inference for discretely observed nonlinear diffusions.
\newblock \emph{Econometrica}, 69\penalty0 (4):\penalty0 959--993, 2001.

\bibitem[Feng et~al.(2025)Feng, Yu, Deng, Hu, and Wu]{feng2025guidance}
Ruiqi Feng, Chenglei Yu, Wenhao Deng, Peiyan Hu, and Tailin Wu.
\newblock On the guidance of flow matching.
\newblock In \emph{Proceedings of the 42nd International Conference on Machine
  Learning}, pp.\  16993--17029, 2025.

\bibitem[Finzi et~al.(2023)Finzi, Boral, Wilson, Sha, and
  Zepeda-N{\'u}{\~n}ez]{finzi2023user}
Marc~Anton Finzi, Anudhyan Boral, Andrew~Gordon Wilson, Fei Sha, and Leonardo
  Zepeda-N{\'u}{\~n}ez.
\newblock User-defined event sampling and uncertainty quantification in
  diffusion models for physical dynamical systems.
\newblock In \emph{Proceedings of the 40th International Conference on Machine
  Learning}, pp.\  10136--10152, 2023.

\bibitem[FitzHugh(1961)]{fitzhugh1961impulses}
Richard FitzHugh.
\newblock Impulses and physiological states in theoretical models of nerve
  membrane.
\newblock \emph{Biophysical Journal}, 1\penalty0 (6):\penalty0 445--466, 1961.

\bibitem[F{\"o}llmer(1985)]{follmer1985entropy}
Hans F{\"o}llmer.
\newblock An entropy approach to the time reversal of diffusion processes.
\newblock In Michel M{\'e}tivier and {\'E}tienne Pardoux (eds.),
  \emph{Stochastic Differential Systems: Filtering and Control}, volume~69 of
  \emph{Lecture Notes in Control and Information Sciences}, pp.\  156--163.
  Springer, Berlin, Heidelberg, 1985.

\bibitem[Franzese et~al.(2023)Franzese, Corallo, Rossi, Heinonen, Filippone,
  and Michiardi]{franzese2023continuous}
Giulio Franzese, Giulio Corallo, Simone Rossi, Markus Heinonen, Maurizio
  Filippone, and Pietro Michiardi.
\newblock Continuous-time functional diffusion processes.
\newblock In \emph{Advances in Neural Information Processing Systems},
  volume~36, pp.\  37370--37400, 2023.

\bibitem[Frigola et~al.(2013)Frigola, Lindsten, Sch{\"o}n, and
  Rasmussen]{frigola2013bayesian}
Roger Frigola, Fredrik Lindsten, Thomas~B. Sch{\"o}n, and Carl~Edward
  Rasmussen.
\newblock {Bayesian} inference and learning in {Gaussian} process state-space
  models with particle {MCMC}.
\newblock In \emph{Advances in Neural Information Processing Systems},
  volume~26, pp.\  3156--3164, 2013.

\bibitem[Frigola et~al.(2014)Frigola, Chen, and
  Rasmussen]{frigola2014variational}
Roger Frigola, Yutian Chen, and Carl~Edward Rasmussen.
\newblock Variational {Gaussian} process state-space models.
\newblock In \emph{Advances in Neural Information Processing Systems},
  volume~27, pp.\  3680--3688, 2014.

\bibitem[Garnelo et~al.(2018)Garnelo, Schwarz, Rosenbaum, Viola, Rezende,
  Eslami, and Teh]{garnelo2018neural}
Marta Garnelo, Jonathan Schwarz, Dan Rosenbaum, Fabio Viola, Danilo~J. Rezende,
  S.~M.~Ali Eslami, and Yee~Whye Teh.
\newblock Neural processes.
\newblock \emph{arXiv preprint arXiv:1807.01622}, 2018.

\bibitem[Girsanov(1960)]{girsanov1960transforming}
Igor~V. Girsanov.
\newblock On transforming a certain class of stochastic processes by absolutely
  continuous substitution of measures.
\newblock \emph{Theory of Probability \& Its Applications}, 5\penalty0
  (3):\penalty0 285--301, 1960.

\bibitem[Golightly \& Wilkinson(2008)Golightly and
  Wilkinson]{golightly2008bayesian}
Andrew Golightly and Darren~J. Wilkinson.
\newblock Bayesian inference for nonlinear multivariate diffusion models
  observed with error.
\newblock \emph{Computational Statistics \& Data Analysis}, 52\penalty0
  (3):\penalty0 1674--1693, 2008.

\bibitem[Gordon et~al.(2020)Gordon, Bruinsma, Foong, Requeima, Dubois, and
  Turner]{gordon2020convolutional}
Jonathan Gordon, Wessel~P. Bruinsma, Andrew Y.~K. Foong, James Requeima, Yann
  Dubois, and Richard~E. Turner.
\newblock Convolutional conditional neural processes.
\newblock In \emph{Proceedings of the 8th International Conference on Learning
  Representations}, 2020.

\bibitem[Grafke(2025)]{grafke2025sampling}
Tobias Grafke.
\newblock Sampling conditioned diffusions via pathspace projected monte carlo.
\newblock \emph{arXiv preprint arXiv:2506.15743}, 2025.

\bibitem[Graham \& Storkey(2017)Graham and Storkey]{graham2017asymptotically}
Matthew Graham and Amos Storkey.
\newblock Asymptotically exact inference in differentiable generative models.
\newblock In \emph{Proceedings of the 20th International Conference on
  Artificial Intelligence and Statistics}, pp.\  499--508, 2017.

\bibitem[Graham et~al.(2022)Graham, Thiery, and Beskos]{graham2022manifold}
Matthew~M. Graham, Alexandre~H. Thiery, and Alexandros Beskos.
\newblock Manifold {Markov} chain {Monte Carlo} methods for {Bayesian}
  inference in diffusion models.
\newblock \emph{Journal of the Royal Statistical Society: Series B (Statistical
  Methodology)}, 84\penalty0 (4):\penalty0 1229--1256, 2022.

\bibitem[Gulian et~al.(2022)Gulian, Frankel, and Swiler]{gulian2020gaussian}
Mamikon Gulian, Ari~L. Frankel, and Laura~P. Swiler.
\newblock {Gaussian} process regression constrained by boundary value problems.
\newblock \emph{Computer Methods in Applied Mechanics and Engineering},
  388:\penalty0 114117, 2022.

\bibitem[Guo et~al.(2026)Guo, Tang, and Xu]{guo2026conditional}
Zhengyi Guo, Wenpin Tang, and Renyuan Xu.
\newblock Conditional diffusion guidance under hard constraint: A stochastic
  analysis approach.
\newblock \emph{arXiv preprint arXiv:2602.05533}, 2026.

\bibitem[Hairer et~al.(2014)Hairer, Stuart, and Vollmer]{hairer2014spectral}
Martin Hairer, Andrew~M. Stuart, and Sebastian~J. Vollmer.
\newblock Spectral gaps for a {Metropolis--Hastings} algorithm in infinite
  dimensions.
\newblock \emph{The Annals of Applied Probability}, 24\penalty0 (6):\penalty0
  2455--2490, 2014.

\bibitem[Hamelijnck et~al.(2024)Hamelijnck, Solin, and
  Damoulas]{hamelijnck2024physics}
Oliver Hamelijnck, Arno Solin, and Theodoros Damoulas.
\newblock Physics-informed variational state-space {Gaussian} processes.
\newblock In \emph{Advances in Neural Information Processing Systems},
  volume~37, pp.\  98505--98536, 2024.

\bibitem[Heng et~al.(2025)Heng, De~Bortoli, Doucet, and
  Thornton]{heng2021simulating}
Jeremy Heng, Valentin De~Bortoli, Arnaud Doucet, and James Thornton.
\newblock Simulating diffusion bridges with score matching.
\newblock \emph{Biometrika}, 112\penalty0 (4):\penalty0 asaf048, 2025.

\bibitem[Hensman et~al.(2015)Hensman, Matthews, and
  Ghahramani]{hensman2015scalable}
James Hensman, Alexander G. de~G. Matthews, and Zoubin Ghahramani.
\newblock Scalable variational {Gaussian} process classification.
\newblock In \emph{Proceedings of The 18th International Conference on
  Artificial Intelligence and Statistics}, pp.\  351--360, 2015.

\bibitem[Heston(1993)]{heston1993closed}
Steven~L. Heston.
\newblock A closed-form solution for options with stochastic volatility with
  applications to bond and currency options.
\newblock \emph{The Review of Financial Studies}, 6\penalty0 (2):\penalty0
  327--343, 1993.

\bibitem[Ho \& Salimans(2022)Ho and Salimans]{ho2022classifierfree}
Jonathan Ho and Tim Salimans.
\newblock Classifier-free diffusion guidance.
\newblock \emph{arXiv preprint arXiv:2207.12598}, 2022.

\bibitem[Ho et~al.(2020)Ho, Jain, and Abbeel]{ho2020denoising}
Jonathan Ho, Ajay Jain, and Pieter Abbeel.
\newblock Denoising diffusion probabilistic models.
\newblock In \emph{Advances in Neural Information Processing Systems},
  volume~33, pp.\  6840--6851, 2020.

\bibitem[Hoffman \& Gelman(2014)Hoffman and Gelman]{hoffman2014no}
Matthew~D. Hoffman and Andrew Gelman.
\newblock The {No-U-Turn} sampler: Adaptively setting path lengths in
  {H}amiltonian {M}onte {C}arlo.
\newblock \emph{Journal of Machine Learning Research}, 15\penalty0
  (47):\penalty0 1593--1623, 2014.

\bibitem[Hoffman et~al.(2019)Hoffman, Sountsov, Dillon, Langmore, Tran, and
  Vasudevan]{hoffman2019neutra}
Matthew~D. Hoffman, Pavel Sountsov, Joshua~V. Dillon, Ian Langmore, Dustin
  Tran, and Srinivas Vasudevan.
\newblock {NeuTra}-lizing bad geometry in hamiltonian {Monte Carlo} using
  neural transport.
\newblock \emph{arXiv preprint arXiv:1903.03704}, 2019.

\bibitem[Holden et~al.(2022)Holden, Pereyra, and Zygalakis]{holden2022bayesian}
Matthew Holden, Marcelo Pereyra, and Konstantinos~C. Zygalakis.
\newblock {Bayesian} imaging with data-driven priors encoded by neural
  networks: Theory, methods, and algorithms.
\newblock \emph{SIAM Journal on Imaging Sciences}, 15\penalty0 (2):\penalty0
  892--924, 2022.

\bibitem[Hosseini et~al.(2025)Hosseini, Hsu, and
  Taghvaei]{hosseini2025conditional}
Bamdad Hosseini, Alexander~W. Hsu, and Amirhossein Taghvaei.
\newblock Conditional optimal transport on function spaces.
\newblock \emph{SIAM/ASA Journal on Uncertainty Quantification}, 13\penalty0
  (1):\penalty0 304--338, 2025.

\bibitem[Huang et~al.(2021)Huang, Jiao, Kang, Liao, Liu, and
  Liu]{huang2021schrodinger}
Jian Huang, Yuling Jiao, Lican Kang, Xu~Liao, Jin Liu, and Yanyan Liu.
\newblock {Schr{\"o}dinger--F{\"o}llmer} sampler: Sampling without ergodicity.
\newblock \emph{arXiv preprint arXiv:2106.10880}, 2021.

\bibitem[Jidling et~al.(2017)Jidling, Wahlstr{\"o}m, Wills, and
  Sch{\"o}n]{jidling2017linearly}
Carl Jidling, Niklas Wahlstr{\"o}m, Adrian Wills, and Thomas~B. Sch{\"o}n.
\newblock Linearly constrained {Gaussian} processes.
\newblock In \emph{Advances in Neural Information Processing Systems},
  volume~30, pp.\  1215--1224, 2017.

\bibitem[Kermack \& McKendrick(1927)Kermack and
  McKendrick]{kermack1927contribution}
William~O. Kermack and Anderson~G. McKendrick.
\newblock A contribution to the mathematical theory of epidemics.
\newblock \emph{Proceedings of the Royal Society of London. Series A,
  Containing Papers of a Mathematical and Physical Character}, 115\penalty0
  (772):\penalty0 700--721, 1927.

\bibitem[Kerrigan et~al.(2024)Kerrigan, Migliorini, and
  Smyth]{kerrigan2024functional}
Gavin Kerrigan, Giosue Migliorini, and Padhraic Smyth.
\newblock Functional flow matching.
\newblock In \emph{Proceedings of The 27th International Conference on
  Artificial Intelligence and Statistics}, pp.\  3934--3942, 2024.

\bibitem[Kim et~al.(2019)Kim, Mnih, Schwarz, Garnelo, Eslami, Rosenbaum,
  Vinyals, and Teh]{kim2019attentive}
Hyunjik Kim, Andriy Mnih, Jonathan Schwarz, Marta Garnelo, S.~M.~Ali Eslami,
  Dan Rosenbaum, Oriol Vinyals, and Yee~Whye Teh.
\newblock Attentive neural processes.
\newblock In \emph{Proceedings of the 7th International Conference on Learning
  Representations}, 2019.

\bibitem[Kuss \& Rasmussen(2005)Kuss and Rasmussen]{kuss2005assessing}
Malte Kuss and Carl~Edward Rasmussen.
\newblock Assessing approximate inference for binary {Gaussian} process
  classification.
\newblock \emph{Journal of Machine Learning Research}, 6:\penalty0 1679--1704,
  2005.

\bibitem[Lange-Hegermann(2018)]{langehegermann2018algorithmic}
Markus Lange-Hegermann.
\newblock Algorithmic linearly constrained {Gaussian} processes.
\newblock In \emph{Advances in Neural Information Processing Systems},
  volume~31, pp.\  2137--2148, 2018.

\bibitem[Law et~al.(2015)Law, Stuart, and Zygalakis]{law2015data}
Kody J.~H. Law, Andrew~M. Stuart, and Konstantinos~C. Zygalakis.
\newblock \emph{Data Assimilation: A Mathematical Introduction}, volume~62 of
  \emph{Texts in Applied Mathematics}.
\newblock Springer, 2015.

\bibitem[L{\'e}onard(2014)]{leonard2014survey}
Christian L{\'e}onard.
\newblock A survey of the {Schr{\"o}dinger} problem and some of its connections
  with optimal transport.
\newblock \emph{Discrete \& Continuous Dynamical Systems - A}, 34\penalty0
  (4):\penalty0 1533--1574, 2014.

\bibitem[Lim et~al.(2025)Lim, Kovachki, Baptista, Beckham, Azizzadenesheli,
  Kossaifi, Voleti, Song, Kreis, Kautz, Pal, Vahdat, and
  Anandkumar]{lim2025score}
Jae~Hyun Lim, Nikola~B. Kovachki, Ricardo Baptista, Christopher Beckham, Kamyar
  Azizzadenesheli, Jean Kossaifi, Vikram Voleti, Jiaming Song, Karsten Kreis,
  Jan Kautz, Christopher Pal, Arash Vahdat, and Anima Anandkumar.
\newblock Score-based diffusion models in function space.
\newblock \emph{Journal of Machine Learning Research}, 26\penalty0
  (158):\penalty0 1--62, 2025.

\bibitem[Lin \& Dunson(2014)Lin and Dunson]{lin2014bayesian}
Lizhen Lin and David~B. Dunson.
\newblock Bayesian monotone regression using {Gaussian} process projection.
\newblock \emph{Biometrika}, 101\penalty0 (2):\penalty0 303--317, 2014.

\bibitem[Lindgren et~al.(2011)Lindgren, Rue, and
  Lindstr{\"o}m]{lindgren2011explicit}
Finn Lindgren, H{\aa}vard Rue, and Johan Lindstr{\"o}m.
\newblock An explicit link between {Gaussian} fields and {Gaussian Markov}
  random fields: The stochastic partial differential equation approach.
\newblock \emph{Journal of the Royal Statistical Society: Series B (Statistical
  Methodology)}, 73\penalty0 (4):\penalty0 423--498, 2011.

\bibitem[Lipman et~al.(2023)Lipman, Chen, Ben-Hamu, Nickel, and
  Le]{lipman2023flow}
Yaron Lipman, Ricky T.~Q. Chen, Heli Ben-Hamu, Maximilian Nickel, and Matt Le.
\newblock Flow matching for generative modeling.
\newblock In \emph{Proceedings of the 11th International Conference on Learning
  Representations}, 2023.

\bibitem[Liu et~al.(2023{\natexlab{a}})Liu, Vahdat, Huang, Theodorou, Nie, and
  Anandkumar]{liu2023i2sb}
Guan-Horng Liu, Arash Vahdat, De-An Huang, Evangelos~A. Theodorou, Weili Nie,
  and Anima Anandkumar.
\newblock {I$^2$SB}: Image-to-image {Schr{\"o}dinger} bridge.
\newblock In \emph{Proceedings of the 40th International Conference on Machine
  Learning}, pp.\  22042--22062, 2023{\natexlab{a}}.

\bibitem[Liu et~al.(2024)Liu, Lipman, Nickel, Karrer, Theodorou, and
  Chen]{liu2024generalized}
Guan-Horng Liu, Yaron Lipman, Maximilian Nickel, Brian Karrer, Evangelos~A.
  Theodorou, and Ricky T.~Q. Chen.
\newblock Generalized {Schr{\"o}dinger} bridge matching.
\newblock In \emph{Proceedings of the 12th International Conference on Learning
  Representations}, 2024.

\bibitem[Liu et~al.(2023{\natexlab{b}})Liu, Gong, and Liu]{liu2023flow}
Xingchao Liu, Chengyue Gong, and Qiang Liu.
\newblock Flow straight and fast: Learning to generate and transfer data with
  rectified flow.
\newblock In \emph{Proceedings of the 11th International Conference on Learning
  Representations}, 2023{\natexlab{b}}.

\bibitem[L{\'o}pez-Lopera et~al.(2018)L{\'o}pez-Lopera, Bachoc, Durrande, and
  Roustant]{lopezlopera2018finite}
Andr{\'e}s~F. L{\'o}pez-Lopera, Fran{\c{c}}ois Bachoc, Nicolas Durrande, and
  Olivier Roustant.
\newblock Finite-dimensional {Gaussian} approximation with linear inequality
  constraints.
\newblock \emph{SIAM/ASA Journal on Uncertainty Quantification}, 6\penalty0
  (3):\penalty0 1224--1255, 2018.

\bibitem[Lotka(1925)]{lotka1925elements}
Alfred~J. Lotka.
\newblock \emph{Elements of Physical Biology}.
\newblock Williams \& Wilkins, Baltimore, 1925.

\bibitem[Ma et~al.(2019)Ma, Li, and Hernandez-Lobato]{ma2019variational}
Chao Ma, Yingzhen Li, and Jose~Miguel Hernandez-Lobato.
\newblock Variational implicit processes.
\newblock In \emph{Proceedings of the 36th International Conference on Machine
  Learning}, pp.\  4222--4233, 2019.

\bibitem[Maatouk \& Bay(2017)Maatouk and Bay]{maatouk2017gaussian}
Hassan Maatouk and Xavier Bay.
\newblock {Gaussian} process emulators for computer experiments with inequality
  constraints.
\newblock \emph{Mathematical Geosciences}, 49\penalty0 (5):\penalty0 557--582,
  2017.

\bibitem[Maatouk et~al.(2025)Maatouk, Rulli{\`e}re, and
  Bay]{maatouk2025bayesian}
Hassan Maatouk, Didier Rulli{\`e}re, and Xavier Bay.
\newblock {Bayesian} analysis of constrained {Gaussian} processes.
\newblock \emph{Bayesian Analysis}, 20\penalty0 (3):\penalty0 973--1002, 2025.

\bibitem[Majda \& Kramer(1999)Majda and Kramer]{majda1999simplified}
Andrew~J. Majda and Peter~R. Kramer.
\newblock Simplified models for turbulent diffusion: Theory, numerical
  modelling, and physical phenomena.
\newblock \emph{Physics Reports}, 314\penalty0 (4--5):\penalty0 237--574, 1999.

\bibitem[Mark et~al.(2025)Mark, Galustian, Kovar, and Heid]{mark2025feynman}
Konstantin Mark, Leonard Galustian, Maximilian P.-P. Kovar, and Esther Heid.
\newblock Feynman--kac-flow: Inference steering of conditional flow matching to
  an energy-tilted posterior.
\newblock \emph{arXiv preprint arXiv:2509.01543}, 2025.

\bibitem[Moss et~al.(2026)Moss, Astfalck, Cowperthwaite, Doumont, Willis,
  Hennig, Nemeth, and Zammit-Mangion]{moss2026conditioning}
Henry~B. Moss, Lachlan Astfalck, Thomas Cowperthwaite, Colin Doumont, Sam
  Willis, Philipp Hennig, Christopher Nemeth, and Andrew Zammit-Mangion.
\newblock Conditioning {Gaussian} processes on almost anything.
\newblock \emph{arXiv preprint arXiv:2605.21041}, 2026.

\bibitem[Nagumo et~al.(1962)Nagumo, Arimoto, and Yoshizawa]{nagumo1962active}
Jin-Ichi Nagumo, Suguru Arimoto, and Shuji Yoshizawa.
\newblock An active pulse transmission line simulating nerve axon.
\newblock \emph{Proceedings of the IRE}, 50\penalty0 (10):\penalty0 2061--2070,
  1962.

\bibitem[Neal(2001)]{neal2001annealed}
Radford~M. Neal.
\newblock Annealed importance sampling.
\newblock \emph{Statistics and Computing}, 11\penalty0 (2):\penalty0 125--139,
  2001.

\bibitem[Nickisch \& Rasmussen(2008)Nickisch and
  Rasmussen]{nickisch2008approximations}
Hannes Nickisch and Carl~Edward Rasmussen.
\newblock Approximations for binary {Gaussian} process classification.
\newblock \emph{Journal of Machine Learning Research}, 9:\penalty0 2035--2078,
  2008.

\bibitem[Nickisch et~al.(2018)Nickisch, Solin, and
  Grigorevskiy]{nickisch2018state}
Hannes Nickisch, Arno Solin, and Alexander Grigorevskiy.
\newblock State space {Gaussian} processes with non-{Gaussian} likelihood.
\newblock In \emph{Proceedings of the 35th International Conference on Machine
  Learning}, pp.\  3789--3798, 2018.

\bibitem[Nijkamp et~al.(2022)Nijkamp, Gao, Sountsov, Vasudevan, Pang, Zhu, and
  Wu]{nijkamp2022mcmc}
Erik Nijkamp, Ruiqi Gao, Pavel Sountsov, Srinivas Vasudevan, Bo~Pang, Song-Chun
  Zhu, and Ying~Nian Wu.
\newblock {MCMC} should mix: Learning energy-based model with neural transport
  latent space {MCMC}.
\newblock In \emph{Proceedings of the 10th International Conference on Learning
  Representations}, 2022.

\bibitem[{\O}ksendal(2003)]{oksendal2003stochastic}
Bernt {\O}ksendal.
\newblock \emph{Stochastic Differential Equations: An Introduction with
  Applications}.
\newblock Universitext. Springer, Berlin, 6 edition, 2003.

\bibitem[Papamakarios et~al.(2021)Papamakarios, Nalisnick, Rezende, Mohamed,
  and Lakshminarayanan]{papamakarios2021normalizing}
George Papamakarios, Eric Nalisnick, Danilo~Jimenez Rezende, Shakir Mohamed,
  and Balaji Lakshminarayanan.
\newblock Normalizing flows for probabilistic modeling and inference.
\newblock \emph{Journal of Machine Learning Research}, 22\penalty0
  (57):\penalty0 1--64, 2021.

\bibitem[Papaspiliopoulos et~al.(2013)Papaspiliopoulos, Roberts, and
  Stramer]{papaspiliopoulos2013data}
Omiros Papaspiliopoulos, Gareth~O. Roberts, and Osnat Stramer.
\newblock Data augmentation for diffusions.
\newblock \emph{Journal of Computational and Graphical Statistics}, 22\penalty0
  (3):\penalty0 665--688, 2013.

\bibitem[Parikh et~al.(2026)Parikh, Chen, and Wang]{parikh2026dflow}
Meet~Hemant Parikh, Yaqin Chen, and Jian-Xun Wang.
\newblock {D-Flow SGLD}: Source-space posterior sampling for scientific inverse
  problems with flow matching.
\newblock \emph{arXiv preprint arXiv:2602.21469}, 2026.

\bibitem[Park et~al.(2024)Park, Choi, Lim, and Lee]{park2024stochastic}
Byoungwoo Park, Jungwon Choi, Sungbin Lim, and Juho Lee.
\newblock Stochastic optimal control for diffusion bridges in function spaces.
\newblock In \emph{Advances in Neural Information Processing Systems},
  volume~37, pp.\  28745--28771, 2024.

\bibitem[Parno \& Marzouk(2018)Parno and Marzouk]{parno2018transport}
Matthew~D. Parno and Youssef~M. Marzouk.
\newblock Transport map accelerated {Markov} chain {Monte Carlo}.
\newblock \emph{SIAM/ASA Journal on Uncertainty Quantification}, 6\penalty0
  (2):\penalty0 645--682, 2018.

\bibitem[Patel et~al.(2022)Patel, Ray, and Oberai]{patel2022solution}
Dhruv~V. Patel, Deep Ray, and Assad~A. Oberai.
\newblock Solution of physics-based {Bayesian} inverse problems with deep
  generative priors.
\newblock \emph{Computer Methods in Applied Mechanics and Engineering},
  400:\penalty0 115428, 2022.

\bibitem[Peluchetti(2023)]{peluchetti2023diffusion}
Stefano Peluchetti.
\newblock Diffusion bridge mixture transports, {Schr{\"o}dinger} bridge
  problems and generative modeling.
\newblock \emph{Journal of Machine Learning Research}, 24\penalty0
  (374):\penalty0 1--51, 2023.

\bibitem[Pidstrigach et~al.(2025)Pidstrigach, Baker, Domingo-Enrich,
  Deligiannidis, and N{\"u}sken]{pidstrigach2025conditioning}
Jakiw Pidstrigach, Elizabeth~Louise Baker, Carles Domingo-Enrich, George
  Deligiannidis, and Nikolas N{\"u}sken.
\newblock Conditioning diffusions using {Malliavin} calculus.
\newblock In \emph{Proceedings of the 42nd International Conference on Machine
  Learning}, pp.\  49292--49315, 2025.

\bibitem[Pieper-Sethmacher et~al.(2025)Pieper-Sethmacher, van~der Meulen, and
  van~der Vaart]{pieper2025simulation}
Thorben Pieper-Sethmacher, Frank van~der Meulen, and Aad van~der Vaart.
\newblock Simulation of infinite-dimensional diffusion bridges.
\newblock \emph{arXiv preprint arXiv:2503.13177}, 2025.

\bibitem[Potts(1952)]{potts1952some}
Renfrey~B. Potts.
\newblock Some generalized order-disorder transformations.
\newblock \emph{Mathematical Proceedings of the Cambridge Philosophical
  Society}, 48\penalty0 (1):\penalty0 106--109, 1952.

\bibitem[Purohit et~al.(2025)Purohit, Repasky, Lu, Qiu, Xie, and
  Cheng]{purohit2024posterior}
Vishal Purohit, Matthew Repasky, Jianfeng Lu, Qiang Qiu, Yao Xie, and Xiuyuan
  Cheng.
\newblock Consistency posterior sampling for diverse image synthesis.
\newblock In \emph{Proceedings of the IEEE/CVF Conference on Computer Vision
  and Pattern Recognition (CVPR)}, pp.\  28327--28336, 2025.

\bibitem[Rahimi \& Recht(2007)Rahimi and Recht]{rahimi2007random}
Ali Rahimi and Benjamin Recht.
\newblock Random features for large-scale kernel machines.
\newblock In \emph{Advances in Neural Information Processing Systems},
  volume~20, pp.\  1177--1184, 2007.

\bibitem[Rajput \& Rosi{\'n}ski(1989)Rajput and
  Rosi{\'n}ski]{rajput1989spectral}
Balram~S. Rajput and Jan Rosi{\'n}ski.
\newblock Spectral representations of infinitely divisible processes.
\newblock \emph{Probability Theory and Related Fields}, 82\penalty0
  (3):\penalty0 451--487, 1989.

\bibitem[Rasmussen \& Williams(2006)Rasmussen and
  Williams]{rasmussen2006gaussian}
Carl~Edward Rasmussen and Christopher~K.I. Williams.
\newblock \emph{{Gaussian} Processes for Machine Learning}.
\newblock Adaptive Computation and Machine Learning. MIT Press, Cambridge, MA,
  2006.

\bibitem[Rezende \& Mohamed(2015)Rezende and Mohamed]{rezende2015variational}
Danilo~Jimenez Rezende and Shakir Mohamed.
\newblock Variational inference with normalizing flows.
\newblock In \emph{Proceedings of the 32nd International Conference on Machine
  Learning}, pp.\  1530--1538, 2015.

\bibitem[Riihim{\"a}ki \& Vehtari(2010)Riihim{\"a}ki and
  Vehtari]{riihimaki2010gaussian}
Jaakko Riihim{\"a}ki and Aki Vehtari.
\newblock {Gaussian} processes with monotonicity information.
\newblock In \emph{Proceedings of The 13th International Conference on
  Artificial Intelligence and Statistics}, pp.\  645--652, 2010.

\bibitem[Roberts \& Stramer(2001)Roberts and Stramer]{roberts2001inference}
Gareth~O. Roberts and Osnat Stramer.
\newblock On inference for partially observed nonlinear diffusion models using
  the {Metropolis--Hastings} algorithm.
\newblock \emph{Biometrika}, 88\penalty0 (3):\penalty0 603--621, 2001.

\bibitem[Rosenblatt(1952)]{rosenblatt1952remarks}
Murray Rosenblatt.
\newblock Remarks on a multivariate transformation.
\newblock \emph{The Annals of Mathematical Statistics}, 23\penalty0
  (3):\penalty0 470--472, 1952.

\bibitem[Rue \& Held(2005)Rue and Held]{rue2005gaussian}
H{\aa}vard Rue and Leonhard Held.
\newblock \emph{Gaussian {Markov} Random Fields: Theory and Applications}.
\newblock Monographs on Statistics and Applied Probability. Chapman \&
  Hall/CRC, Boca Raton, FL, 2005.

\bibitem[Rue et~al.(2009)Rue, Martino, and Chopin]{rue2009approximate}
H{\aa}vard Rue, Sara Martino, and Nicolas Chopin.
\newblock Approximate {Bayesian} inference for latent {Gaussian} models by
  using integrated nested {Laplace} approximations.
\newblock \emph{Journal of the Royal Statistical Society: Series B (Statistical
  Methodology)}, 71\penalty0 (2):\penalty0 319--392, 2009.

\bibitem[S{\"a}rkk{\"a}(2013)]{sarkka2013bayesian}
Simo S{\"a}rkk{\"a}.
\newblock \emph{{Bayesian} Filtering and Smoothing}.
\newblock Cambridge University Press, Cambridge, 2013.

\bibitem[Schauer et~al.(2017)Schauer, van~der Meulen, and van
  Zanten]{schauer2017guided}
Moritz Schauer, Frank van~der Meulen, and Harry van Zanten.
\newblock Guided proposals for simulating multi-dimensional diffusion bridges.
\newblock \emph{Bernoulli}, 23\penalty0 (4A):\penalty0 2917--2950, 2017.

\bibitem[Seong et~al.(2025)Seong, Park, Kim, Kim, and Ahn]{seong2025transition}
Kiyoung Seong, Seonghyun Park, Seonghwan Kim, Woo~Youn Kim, and Sungsoo Ahn.
\newblock Transition path sampling with improved off-policy training of
  diffusion path samplers.
\newblock In \emph{Proceedings of the 13th International Conference on Learning
  Representations}, 2025.

\bibitem[Shah et~al.(2014)Shah, Wilson, and Ghahramani]{shah2014student}
Amar Shah, Andrew~Gordon Wilson, and Zoubin Ghahramani.
\newblock {Student-t} processes as alternatives to {Gaussian} processes.
\newblock In \emph{Proceedings of The 17th International Conference on
  Artificial Intelligence and Statistics}, pp.\  877--885, 2014.

\bibitem[Singhal et~al.(2025)Singhal, Horvitz, Teehan, Ren, Yu, McKeown, and
  Ranganath]{singhal2025general}
Raghav Singhal, Zachary Horvitz, Ryan Teehan, Mengye Ren, Zhou Yu, Kathleen
  McKeown, and Rajesh Ranganath.
\newblock A general framework for inference-time scaling and steering of
  diffusion models.
\newblock In \emph{Proceedings of the 42nd International Conference on Machine
  Learning}, pp.\  55810--55827, 2025.

\bibitem[Skreta et~al.(2025)Skreta, Akhound-Sadegh, Ohanesian, Bondesan,
  Aspuru-Guzik, Doucet, Brekelmans, Tong, and Neklyudov]{skreta2025feynman}
Marta Skreta, Tara Akhound-Sadegh, Viktor Ohanesian, Roberto Bondesan, Alan
  Aspuru-Guzik, Arnaud Doucet, Rob Brekelmans, Alexander Tong, and Kirill
  Neklyudov.
\newblock {Feynman--Kac} correctors in diffusion: Annealing, guidance, and
  product of experts.
\newblock In \emph{Proceedings of the 42nd International Conference on Machine
  Learning}, pp.\  55906--55949, 2025.

\bibitem[Sohl-Dickstein et~al.(2015)Sohl-Dickstein, Weiss, Maheswaranathan, and
  Ganguli]{sohl2015deep}
Jascha Sohl-Dickstein, Eric Weiss, Niru Maheswaranathan, and Surya Ganguli.
\newblock Deep unsupervised learning using nonequilibrium thermodynamics.
\newblock In \emph{Proceedings of the 32nd International Conference on Machine
  Learning}, pp.\  2256--2265, 2015.

\bibitem[Song et~al.(2021)Song, Sohl-Dickstein, Kingma, Kumar, Ermon, and
  Poole]{song2021score}
Yang Song, Jascha Sohl-Dickstein, Diederik~P. Kingma, Abhishek Kumar, Stefano
  Ermon, and Ben Poole.
\newblock Score-based generative modeling through stochastic differential
  equations.
\newblock In \emph{Proceedings of the 9th International Conference on Learning
  Representations}, 2021.

\bibitem[Stuart(2010)]{stuart2010inverse}
Andrew~M. Stuart.
\newblock Inverse problems: A {Bayesian} perspective.
\newblock \emph{Acta Numerica}, 19:\penalty0 451--559, 2010.

\bibitem[Swiler et~al.(2020)Swiler, Gulian, Frankel, Safta, and
  Jakeman]{swiler2020survey}
Laura Swiler, Mamikon Gulian, Ari Frankel, Cosmin Safta, and John Jakeman.
\newblock A survey of constrained {Gaussian} process {Regression}: Approaches
  and implementation challenges.
\newblock \emph{Journal of Machine Learning for Modeling and Computing},
  1\penalty0 (2):\penalty0 119--156, 2020.

\bibitem[Tamogashev \& Malkin(2026)Tamogashev and Malkin]{tamogashev2026data}
Kirill Tamogashev and Nikolay Malkin.
\newblock Data-to-energy stochastic dynamics.
\newblock In \emph{Proceedings of the 14th International Conference on Learning
  Representations}, 2026.

\bibitem[Vargas et~al.(2023{\natexlab{a}})Vargas, Grathwohl, and
  Doucet]{vargas2023denoising}
Francisco Vargas, Will~Sussman Grathwohl, and Arnaud Doucet.
\newblock Denoising diffusion samplers.
\newblock In \emph{Proceedings of the 11th International Conference on Learning
  Representations}, 2023{\natexlab{a}}.

\bibitem[Vargas et~al.(2023{\natexlab{b}})Vargas, Ovsianas, Fernandes,
  Girolami, Lawrence, and N{\"u}sken]{vargas2022bayesian}
Francisco Vargas, Andrius Ovsianas, David Fernandes, Mark Girolami, Neil~D.
  Lawrence, and Nikolas N{\"u}sken.
\newblock Bayesian learning via neural {Schr{\"o}dinger--F{\"o}llmer} flows.
\newblock \emph{Statistics and Computing}, 33\penalty0 (3), 2023{\natexlab{b}}.

\bibitem[Vargas et~al.(2024)Vargas, Padhy, Blessing, and
  N{\"u}sken]{vargas2024controlled}
Francisco Vargas, Shreyas Padhy, Denis Blessing, and Nikolas N{\"u}sken.
\newblock Transport meets variational inference: Controlled {Monte Carlo}
  diffusions.
\newblock In \emph{Proceedings of the 12th International Conference on Learning
  Representations}, 2024.

\bibitem[Venkatraman et~al.(2025)Venkatraman, Hasan, Kim, Scimeca, Sendera,
  Bengio, Berseth, and Malkin]{venkatraman2025outsourced}
Siddarth Venkatraman, Mohsin Hasan, Minsu Kim, Luca Scimeca, Marcin Sendera,
  Yoshua Bengio, Glen Berseth, and Nikolay Malkin.
\newblock Outsourced diffusion sampling: Efficient posterior inference in
  latent spaces of generative models.
\newblock In \emph{Proceedings of the 42nd International Conference on Machine
  Learning}, pp.\  61212--61239, 2025.

\bibitem[Volterra(1926)]{volterra1926fluctuations}
Vito Volterra.
\newblock Fluctuations in the abundance of a species considered mathematically.
\newblock \emph{Nature}, 118:\penalty0 558--560, 1926.

\bibitem[Walchessen et~al.(2025)Walchessen, Zammit-Mangion, Huser, and
  Kuusela]{walchessen2025neural}
Julia Walchessen, Andrew Zammit-Mangion, Rapha{\"e}l Huser, and Mikael Kuusela.
\newblock Neural conditional simulation for complex spatial processes.
\newblock \emph{arXiv preprint arXiv:2508.20067}, 2025.

\bibitem[Wang et~al.(2011)Wang, Zhang, Xu, and Wang]{wang2011quantifying}
Jin Wang, Kun Zhang, Li~Xu, and Erkang Wang.
\newblock Quantifying the waddington landscape and biological paths for
  development and differentiation.
\newblock \emph{Proceedings of the National Academy of Sciences}, 108\penalty0
  (20):\penalty0 8257--8262, 2011.

\bibitem[Wang \& Tartakovsky(2026)Wang and Tartakovsky]{wang2026latent}
Yuanzhe Wang and Alexandre~M. Tartakovsky.
\newblock Latent diffusion posterior sampling with surrogate likelihood
  guidance for {PDE} inverse problems.
\newblock \emph{arXiv preprint arXiv:2606.26592}, 2026.

\bibitem[Wang et~al.(2026)Wang, Harting, Barreau, Zavlanos, and
  Johansson]{wang2025sourceguided}
Zifan Wang, Alice Harting, Matthieu Barreau, Michael~M. Zavlanos, and Karl~H.
  Johansson.
\newblock Source-guided flow matching.
\newblock In \emph{Proceedings of the 14th International Conference on Learning
  Representations}, 2026.

\bibitem[Whang et~al.(2021)Whang, Lindgren, and Dimakis]{whang2021composing}
Jay Whang, Erik~M. Lindgren, and Alexandros~G. Dimakis.
\newblock Composing normalizing flows for inverse problems.
\newblock In \emph{Proceedings of the 38th International Conference on Machine
  Learning}, pp.\  11158--11169, 2021.

\bibitem[Whitaker et~al.(2017)Whitaker, Golightly, Boys, and
  Sherlock]{whitaker2017bayesian}
Gavin~A. Whitaker, Andrew Golightly, Richard~J. Boys, and Christopher~G.
  Sherlock.
\newblock Bayesian inference for diffusion-driven mixed-effects models.
\newblock \emph{Bayesian Analysis}, 12\penalty0 (2):\penalty0 435--463, 2017.

\bibitem[Wilson et~al.(2020)Wilson, Borovitskiy, Terenin, Mostowsky, and
  Deisenroth]{wilson2020efficient}
James~T. Wilson, Viacheslav Borovitskiy, Alexander Terenin, Peter Mostowsky,
  and Marc~Peter Deisenroth.
\newblock Efficiently sampling functions from {Gaussian} process posteriors.
\newblock In \emph{Proceedings of the 37th International Conference on Machine
  Learning}, pp.\  10292--10302, 2020.

\bibitem[Wilson et~al.(2021)Wilson, Borovitskiy, Terenin, Mostowsky, and
  Deisenroth]{wilson2021pathwise}
James~T. Wilson, Viacheslav Borovitskiy, Alexander Terenin, Peter Mostowsky,
  and Marc~Peter Deisenroth.
\newblock Pathwise conditioning of {Gaussian} processes.
\newblock \emph{Journal of Machine Learning Research}, 22\penalty0
  (105):\penalty0 1--47, 2021.

\bibitem[Wolpert et~al.(2011)Wolpert, Clyde, and Tu]{wolpert2011levy}
Robert~L. Wolpert, Merlise~A. Clyde, and Chong Tu.
\newblock Stochastic expansions using continuous dictionaries: L{\'e}vy
  adaptive regression kernels.
\newblock \emph{The Annals of Statistics}, 39\penalty0 (4):\penalty0
  1916--1962, 2011.

\bibitem[Wu et~al.(2022)Wu, Motamed, Srivastava, and {De la
  Torre}]{wu2022promptgen}
Chen~Henry Wu, Saman Motamed, Shaunak Srivastava, and Fernando~D. {De la
  Torre}.
\newblock Generative visual prompt: Unifying distributional control of
  pre-trained generative models.
\newblock In \emph{Advances in Neural Information Processing Systems},
  volume~35, pp.\  22422--22437, 2022.

\bibitem[Wu et~al.(2023)Wu, Trippe, Naesseth, Blei, and
  Cunningham]{wu2023practical}
Luhuan Wu, Brian~L. Trippe, Christian~A. Naesseth, David~M. Blei, and John~P.
  Cunningham.
\newblock Practical and asymptotically exact conditional sampling in diffusion
  models.
\newblock In \emph{Advances in Neural Information Processing Systems},
  volume~36, pp.\  31372--31403, 2023.

\bibitem[Wu et~al.(2025)Wu, Han, Naesseth, and Cunningham]{wu2025reverse}
Luhuan Wu, Yi~Han, Christian~Andersson Naesseth, and John~P. Cunningham.
\newblock Reverse diffusion sequential {Monte Carlo} samplers.
\newblock In \emph{Advances in Neural Information Processing Systems},
  volume~38, 2025.

\bibitem[Xia et~al.(2026)Xia, Tanomkiattikun, Zhen, and Gu]{xia2026noise}
Yingzhi Xia, Setthakorn Tanomkiattikun, Liangli Zhen, and Zaiwang Gu.
\newblock Noise-adaptive diffusion sampling for inverse problems without
  task-specific tuning.
\newblock In \emph{Proceedings of the 14th International Conference on Learning
  Representations}, 2026.

\bibitem[Yang et~al.(2025{\natexlab{a}})Yang, Baker, Severinsen, Hipsley, and
  Sommer]{yang2025infinite}
Gefan Yang, Elizabeth~Louise Baker, Michael~Lind Severinsen, Christy~Anna
  Hipsley, and Stefan Sommer.
\newblock Infinite-dimensional diffusion bridge simulation via operator
  learning.
\newblock In \emph{Proceedings of the 28th International Conference on
  Artificial Intelligence and Statistics}, volume 258, pp.\  3556--3564,
  2025{\natexlab{a}}.

\bibitem[Yang et~al.(2025{\natexlab{b}})Yang, van~der Meulen, and
  Sommer]{yang2025neural}
Gefan Yang, Frank van~der Meulen, and Stefan Sommer.
\newblock Neural guided diffusion bridges.
\newblock In \emph{Proceedings of the 42nd International Conference on Machine
  Learning}, pp.\  71210--71230, 2025{\natexlab{b}}.

\bibitem[Yao et~al.(2025)Yao, Mammadov, Berner, Kerrigan, Ye, Azizzadenesheli,
  and Anandkumar]{yao2025guided}
Jiachen Yao, Abbas Mammadov, Julius Berner, Gavin Kerrigan, Jong~Chul Ye,
  Kamyar Azizzadenesheli, and Anima Anandkumar.
\newblock Guided diffusion sampling on function spaces with applications to
  {PDE}s.
\newblock In \emph{Advances in Neural Information Processing Systems},
  volume~38, 2025.

\bibitem[Zammit-Mangion et~al.(2025)Zammit-Mangion, Sainsbury-Dale, and
  Huser]{zammit2025neural}
Andrew Zammit-Mangion, Matthew Sainsbury-Dale, and Rapha{\"e}l Huser.
\newblock Neural methods for amortized inference.
\newblock \emph{Annual Review of Statistics and Its Application}, 12\penalty0
  (1):\penalty0 311--335, 2025.

\bibitem[Zhang \& Chen(2022)Zhang and Chen]{zhang2022path}
Qinsheng Zhang and Yongxin Chen.
\newblock Path integral sampler: a stochastic control approach for sampling.
\newblock In \emph{Proceedings of the 10th International Conference on Learning
  Representations}, 2022.

\bibitem[Zhao et~al.(2025)Zhao, Luo, Sj{\"o}lund, and
  Sch{\"o}n]{zhao2025conditional}
Zheng Zhao, Ziwei Luo, Jens Sj{\"o}lund, and Thomas~B. Sch{\"o}n.
\newblock Conditional sampling within generative diffusion models.
\newblock \emph{Philosophical Transactions of the Royal Society A:
  Mathematical, Physical and Engineering Sciences}, 383\penalty0
  (2299):\penalty0 20240329, 2025.

\end{thebibliography}

\appendix

\appendix
\section{Examples of latent innovation representations}
\label[appendix]{app:innovation-representations}

This appendix contains examples of stochastic-process priors that admit the latent innovation representation in \Cref{ass:latent-rep}, or its measure-theoretic extension; see \Cref{app:measure-theoretic-view} for details. 
Throughout, \(s_{1:m}\) denotes the discretisation inputs, \(f_0\in E_m\) denotes the finite-dimensional process realisation, and \(T_\vart\) denotes a generator satisfying \(f_0=T_\vart(\xi_0)\).

In several examples, the model is most naturally written in terms of a native innovation variable \(\zeta_0\), such as a uniform simulator seed, a Gamma scale variable, or a collection of non-Gaussian increments.
To ensure all examples are consistent with the common latent form used by \textsc{LatentFlow}, we express this native innovation as a deterministic function of a standard Gaussian reference:
\(
  \zeta_0 = H_\vart(\xi_0),
\)
with
\(
  \xi_0\sim\N(0,I).
\)
The generator in Assumption~\ref{ass:latent-rep} is then the composite map from the Gaussian innovation to the process realisation.
When the native innovation is already Gaussian, \(H_\vart\) is simply the identity map.

\begingroup
\small
\setlength{\tabcolsep}{5pt}
\setlength{\LTcapwidth}{\textwidth}
\setlength{\abovedisplayskip}{3pt}
\setlength{\belowdisplayskip}{3pt}
\setlength{\abovedisplayshortskip}{3pt}
\setlength{\belowdisplayshortskip}{3pt}
\setlength{\jot}{2pt}
\renewcommand{\arraystretch}{1.35}

\begin{longtable}{L{0.23\textwidth}p{0.69\textwidth}}
\caption{Examples of stochastic-process priors and their latent innovation representations.}
\label{tab:innovation-representations}
\\
\toprule
\rowcolor{lfHeaderGray}
\textbf{Model class}
&
\textbf{Innovation representation}
\\
\midrule
\endfirsthead

\toprule
\rowcolor{lfHeaderGray}
\textbf{Model class}
&
\textbf{Innovation representation}
\\
\midrule
\endhead

\midrule
\multicolumn{2}{r}{\emph{continued on next page}}
\\
\endfoot

\bottomrule
\endlastfoot

\rowcolor{lfBlue}
\multicolumn{2}{l}{\textbf{\emph{Gaussian, Student-$t$, and basis-expansion priors}}}
\\
\midrule

Gaussian process
&
A finite-dimensional realisation \(f_0=(f(s_1),\ldots,f(s_m))\) of a Gaussian process can be written as
\[
  f_0=\mu_\vart+L_\vart \xi_0,
  \qquad
  L_\vart L_\vart^\top=K_\vart,
\]
where \((K_\vart)_{ij}=k_\vart(s_i,s_j)\) is the kernel matrix, and \(\xi_0\sim\N(0,I_m)\).
The innovation is given by
\[
  \xi_0\in\mathbb R^m,
  \qquad
  \xi_0\sim\N(0,I_m).
\]
The generator is given by
\[
  T_\vart(\xi_0)=\mu_\vart+L_\vart \xi_0.
\]
This is the affine whitening map used in finite-dimensional Gaussian process inference.
\\

\midrule

Student-\(t\) process
&
A finite-dimensional realisation \(f_0=(f(s_1),\ldots,f(s_m))\) of a Student-\(t\) process can be written as
\[
  f_0
  =
  \mu_\vart+\omega^{-1/2}L_\vart v,
  \qquad
  L_\vart L_\vart^\top=K_\vart,
\]
where \((K_\vart)_{ij}=k_\vart(s_i,s_j)\) is the scale matrix induced by the Student-\(t\) process kernel \(k_\vart\), and \(v\) and \(\omega\) are independent, with
\(
  v\sim\N(0,I_m)
\)
and
\(
  \omega\sim\operatorname{Gamma}(\nu/2,\nu/2)
\)
in the shape--rate parameterisation.
The natural innovation is given by
\[
  \zeta_0=(v,\omega)\in\mathbb R^{m+1}.
\]
Let
\(
  w=\Phi_{\mathrm N}^{-1}(F_{\Gamma,\nu}(\omega)),
\)
where \(F_{\Gamma,\nu}\) is the distribution function of \(\operatorname{Gamma}(\nu/2,\nu/2)\), and \(\Phi_{\mathrm N}\) is the standard normal distribution function.
The Gaussian innovation is then given by
\[
   \xi_0=(v,w)\in\mathbb R^{m+1},
   \qquad
   \xi_0\sim\N(0,I_{m+1}).
\]
The corresponding generator is given by
\[
  T_\vart(\xi_0)
  =
  \mu_\vart+
  F_{\Gamma,\nu}^{-1}(\Phi_{\mathrm N}(w))^{-1/2}
  L_\vart v.
\]
This gives a heavy-tailed analogue of the Gaussian process while retaining an explicit scale-mixture innovation map.
\\

\midrule

Finite-basis, random-feature, or pathwise GP approximation
&
A truncated Karhunen--Loève expansion, finite basis expansion, random-feature approximation, or pathwise Gaussian process approximation can be written as
\[
  f_0=\mu_\vart+\Phi_\vart \xi_0,
\]
where \(\Phi_\vart\in\mathbb R^{m\times q}\) contains the basis or feature evaluations at \(s_{1:m}\), and \(\xi_0\sim\N(0,I_q)\).
For example,
\[
  (\Phi_\vart)_{ik}
  =
  \sqrt{\lambda_{\vart,k}}\phi_{\vart,k}(s_i)
\]
for a truncated Karhunen--Loève expansion.
The innovation is given by
\[
  \xi_0\in\mathbb R^q,
  \qquad
  \xi_0\sim\N(0,I_q).
\]
The generator is given by
\[
  T_\vart(\xi_0)=\mu_\vart+\Phi_\vart \xi_0.
\]
The innovation dimension \(q\) need not equal the discretisation dimension \(m\).
When \(q<m\), the induced prior on \(E_m\) is typically supported on a low-dimensional subspace unless an additional residual or nugget term is included.
More elaborate pathwise Gaussian process constructions fit the same form after augmenting \(\xi_0\) with auxiliary Gaussian variables.
\\

\midrule

Spectral random field
&
A finite spectral random field can be written as
\[
  f(s)
  =
  \mu_\vart(s)
  +
  \sum_{k=1}^q A_{\vart,k}
  \bigl(
    a_k\cos(\omega_k^\top s)
    +
    b_k\sin(\omega_k^\top s)
  \bigr),
\]
where \(A_{\vart,k}\) are spectral amplitudes, \(\omega_k\) are frequencies, and \(a_k,b_k\stackrel{\mathrm{i.i.d.}}{\sim}\N(0,1)\).
Here the frequencies \(\omega_k\) are treated as fixed, or as part of \(\vart\).
If the frequencies are themselves random, they can be appended to the innovation variable.
The innovation is given by
\[
  \xi_0=(a_1,b_1,\ldots,a_q,b_q)\in\mathbb R^{2q},
  \qquad
  \xi_0\sim\N(0,I_{2q}).
\]
The generator is given by
\[
  (T_\vart(\xi_0))_i
  =
  \mu_\vart(s_i)
  +
  \sum_{k=1}^q A_{\vart,k}
  \bigl(
    a_k\cos(\omega_k^\top s_i)
    +
    b_k\sin(\omega_k^\top s_i)
  \bigr),
\]
for components \(i=1,\dots,m\).
\\

\midrule

Gaussian Markov random field, or SPDE-defined Matérn field
&
A Gaussian Markov random field, including an SPDE-defined Matérn field after finite-element discretisation, can be written as
\[
  R_\vart^\top R_\vart=Q_\vart,
  \qquad
  f_0=\mu_\vart+A_\vart R_\vart^{-1}\xi_0,
\]
where \(x_\vart = R_\vart^{-1}\xi_0\in\mathbb R^q\) is a vector of GMRF or finite-element coefficients, \(Q_\vart\succ0\) is its sparse precision matrix, \(A_\vart:\mathbb R^q\to E_m\) maps the coefficients to the requested process values, and \(\xi_0\sim\N(0,I_q)\).
For an SPDE-defined Matérn field, \(Q_\vart\) is the sparse precision matrix of the finite-element coefficient vector, and \(A_\vart\) is the corresponding finite-element evaluation or projection matrix.
The innovation is given by
\[
  \xi_0\in\mathbb R^q,
  \qquad
  \xi_0\sim\N(0,I_q).
\]
The generator is given by
\[
  T_\vart(\xi_0)
  =
  \mu_\vart+A_\vart R_\vart^{-1}\xi_0.
\]
In practice, one solves a sparse triangular system rather than forming \(R_\vart^{-1}\).
This includes proper areal Gaussian Markov random fields such as proper CAR and Leroux-type models, as well as SPDE-defined Matérn fields.
Intrinsic CAR or intrinsic GMRF precision matrices are singular and instead require an identifying constraint, a reduced-coordinate factorisation on the constrained subspace, or an explicitly specified generalised-inverse construction.
\\

\newpage
\midrule

\rowcolor{lfTeal}
\multicolumn{2}{l}{\textbf{\emph{Intensity and convolution-field priors}}}
\\
\midrule

Log-Gaussian Cox intensity
&
A discretised log-Gaussian Cox process intensity can be written as
\[
  f_0
  =
  \exp(\mu_\vart+L_\vart\xi_0),
  \qquad
  L_\vart L_\vart^\top=K_\vart,
\]
where \(f_0=(\lambda(s_1),\ldots,\lambda(s_m))\), \(g_0=\mu_\vart+L_\vart\xi_0\) is the discretised latent Gaussian process, \(\lambda(s_i)=\exp(g_0(s_i))\), and \(\xi_0\sim\N(0,I_m)\).
The innovation is given by
\[
  \xi_0\in\mathbb R^m,
  \qquad
  \xi_0\sim\N(0,I_m).
\]
The generator is given by
\[
  T_\vart(\xi_0)=\exp(\mu_\vart+L_\vart \xi_0),
\]
where the exponential is applied componentwise.
\\

\midrule

Lévy-driven or convolution random field
&
A finite Lévy-driven or stable-convolution random field can be written as
\[
  f(s_i)
  =
  \sum_{k=1}^q
  G_\vart(s_i-c_k)J_k,
  \qquad i=1,\ldots,m,
\]
where \(G_\vart\) is a convolution kernel, \(c_k\) are support or quadrature locations, and \(J_k\sim\nu_{\vart,k}\) are increments of an independently scattered random measure.
The natural innovation is given by
\[
  \zeta_0=(J_1,\ldots,J_q).
\]
Let \(H_{\vart,k}\) be a map such that
\[
  J_k=H_{\vart,k}(w_k),
  \qquad
  w_k\sim\N(0,I).
\]
The Gaussian innovation is then given by
\[
  \xi_0=(w_1,\ldots,w_q).
\]
The corresponding generator is given by
\[
  (T_\vart(\xi_0))_i
  =
  \sum_{k=1}^q
  G_\vart(s_i-c_k)H_{\vart,k}(w_k),
  \qquad i=1,\ldots,m.
\]
This gives a non-Gaussian random-field prior, including Cauchy, normal-inverse Gaussian (NIG), stable, or other heavy-tailed convolution fields, through independently scattered random-measure increments.
\\

\newpage
\midrule

\rowcolor{lfPurple}
\multicolumn{2}{l}{\textbf{\emph{Discrete and extreme-value priors}}}
\\
\midrule

Potts or discrete Markov random field
&
A Potts or discrete Markov random field on sites \(s_{1:m}\) can be written as
\[
  \mathbb P_\vart(f_0=x)
  \propto
  \exp\Big(
    \sum_{i=1}^m h_{\vart,i}(x_i)
    +
    \sum_{(i,j)\in\calE}
    \beta_{\vart,ij}\mathbf 1(x_i=x_j)
  \Big),
\]
where \(x=(x_1,\ldots,x_m)\in\{1,\ldots,K\}^m\), \(x_i=f_0(s_i)\), \(\calE\) is the neighbourhood graph on \(s_{1:m}\), \(h_{\vart,i}\) are site potentials, and \(\beta_{\vart,ij}\) are interaction parameters.
The natural innovation can be represented by a simulator seed
\[
  \zeta_0=u\in(0,1)^q,
  \qquad
  u\sim\operatorname{Unif}((0,1)^q),
\]
together with an exact sampler \(\mathcal S_\vart\) such that
\(
  f_0=\mathcal S_\vart(u)\sim\mathbb P_\vart.
\)
If \(\mathcal S_\vart\) is an approximate MCMC or sequential sampler, then the corresponding \(T_\vart\) represents the algorithmic approximation rather than the exact Potts law.
Let
\(
  \xi_0=\Phi_{\mathrm N}^{-1}(u),
\)
where \(\Phi_{\mathrm N}\) is the standard normal distribution function, applied componentwise.
The Gaussian innovation is then given by
\[
  \xi_0\in\mathbb R^q,
  \qquad
  \xi_0\sim\N(0,I_q).
\]
The corresponding generator is given by
\[
  T_\vart(\xi_0)
  =
  \mathcal S_\vart(\Phi_{\mathrm N}(\xi_0)).
\]
Exact Potts fields are discrete; gradient-based guidance therefore requires a relaxation, surrogate, or derivative-free guidance estimator.
\\

\midrule

Max-stable process
&
A finite spectral truncation of a max-stable process can be written as
\[
  f(s_i)
  =
  \max_{1\le k\le K}
  P_k W_{\vart,k}(s_i),
  \qquad
  P_k=\Gamma_k^{-1},
  \qquad
  \Gamma_k=\sum_{\ell=1}^k E_\ell,
\]
where \(E_k\stackrel{\mathrm{i.i.d.}}{\sim}\operatorname{Exp}(1)\), the spectral functions \(W_{\vart,k}\) are independent and identically distributed, and the sequence \((W_{\vart,k})_{k\ge1}\) is independent of \((E_k)_{k\ge1}\).
The spectral functions satisfy the usual normalisation \(\E[W_{\vart,k}(s)]=1\) for each \(s\).
For finite \(K\), this is a spectral truncation of the corresponding infinite max-stable representation.
The natural innovation is given by
\[
  \zeta_0=(E_{1:K},\upsilon_{1:K}),
\]
where \(\upsilon_k\) generates the spectral function \(W_{\vart,k}\).
Let
\[
  w_k^E=\Phi_{\mathrm N}^{-1}(\exp(-E_k)),
\]
so that, equivalently,
\[
  E_k=-\log\Phi_{\mathrm N}(w_k^E),
  \qquad
  w_k^E\sim\N(0,1).
\]
If the spectral innovation can be Gaussianised as
\[
  \upsilon_k=H_{\vart}^{W}(w_k^W),
  \qquad
  w_k^W\sim\N(0,I),
\]
then the Gaussian innovation is given by
\[
  \xi_0=(w_1^E,\ldots,w_K^E,w_1^W,\ldots,w_K^W).
\]
All variables \(w_k^E\) and \(w_k^W\) are taken to be mutually independent across \(k\) and across the two innovation blocks.
The corresponding generator is given by
\[
  (T_\vart(\xi_0))_i
  =
  \max_{1\le k\le K}
  \Gamma_k^{-1}
  W_\vart(s_i;H_{\vart}^{W}(w_k^W)),
  \qquad i=1,\ldots,m,
\]
where
\(
  \Gamma_k
  =
  \sum_{\ell=1}^k
  -\log\Phi_{\mathrm N}(w_\ell^E).
\)
The maximum is non-smooth, so smooth relaxations may be required for gradient-based guidance.
\\

\newpage
\midrule

\rowcolor{lfGreen}
\multicolumn{2}{l}{\textbf{\emph{Dynamical and spatiotemporal priors}}}
\\
\midrule

State-space model or stochastic recurrence
&
A state-space model or stochastic recurrence can be written as
\[
\begin{aligned}
  x_0
  &=
  \Psi_{\vart,0}(\eta_0),
  \\
  x_n
  &=
  \Psi_{\vart,n}(x_{n-1},\eta_n),
  \qquad n=1,\ldots,N,
\end{aligned}
\]
where \(\eta_n\stackrel{\mathrm{ind}}{\sim}r_n\) are model innovations and \(\Psi_{\vart,n}\) is the one-step simulator.
The natural innovation is given by
\[
  \zeta_0=(\eta_0,\ldots,\eta_N).
\]
Let \(H_{\vart,n}\) be a map such that
\[
  \eta_n=H_{\vart,n}(w_n),
  \qquad
  w_n\sim\N(0,I).
\]
The Gaussian innovation is then given by
\[
  \xi_0=(w_0,\ldots,w_N).
\]
The corresponding generator is given by
\[
\begin{aligned}
  x_0
  &=
  \Psi_{\vart,0}(H_{\vart,0}(w_0)),
  \\
  x_n
  &=
  \Psi_{\vart,n}(x_{n-1},H_{\vart,n}(w_n)),
  \qquad n=1,\ldots,N,
\end{aligned}
\]
with output
\[
  T_\vart(\xi_0)=(x_0,\ldots,x_N).
\]
When the original innovations are already Gaussian, \(H_{\vart,n}\) is the identity map.
For discrete or mixed innovations, the maps \(H_{\vart,n}\) may be nonsmooth, in which case gradient-based guidance again requires a relaxation, surrogate, or derivative-free estimator.
Linear Gaussian and nonlinear Gaussian state-space models correspond to particular choices of \(\Psi_{\vart,n}\).
\\

\midrule

Diffusion process
&
A diffusion process satisfying
\[
  \dd x_\tau=b_\vart(\tau,x_\tau)\dd \tau+\Sigma_\vart(\tau,x_\tau)\dd w_\tau
\]
can be discretised by Euler--Maruyama as
\[
  x_{n+1}
  =
  x_n+b_\vart(\tau_n,x_n)\Delta_n
  +
  \Sigma_\vart(\tau_n,x_n)\sqrt{\Delta_n}\eta_n,
\]
where \(\Delta_n=\tau_{n+1}-\tau_n\), \(\eta_n\stackrel{\mathrm{ind}}{\sim}\N(0,I)\), and \(\Sigma_\vart(\tau_n,x_n)\) maps Gaussian innovations into diffusion increments.
The innovation is given by
\[
  \xi_0=(\eta_{-1},\eta_0,\ldots,\eta_{N-1}),
\]
where \(\eta_{-1}\sim\N(0,I)\) generates the initial state if \(x_0\) is random, through an initial-state map \(x_0=\chi_{\vart,0}(\eta_{-1})\), and is omitted otherwise.
The generator is given by applying the Euler--Maruyama recursion above and returning
\[
  T_\vart(\xi_0)=(x_0,\ldots,x_N).
\]
More generally, higher-order numerical solvers correspond to different choices of \(T_\vart\).
\\

\midrule

Jump diffusion, Lévy-driven SDE, or Lévy process
&
A jump diffusion or Lévy-driven SDE satisfying
\[
  \dd x_\tau
  =
  b_\vart(\tau,x_\tau)\dd \tau
  +
  \Sigma_\vart(\tau,x_\tau)\dd W_\tau
  +
  \dd j_\tau
\]
can be discretised as
\[
  x_{n+1}
  =
  x_n
  +
  b_\vart(\tau_n,x_n)\Delta_n
  +
  \Sigma_\vart(\tau_n,x_n)\sqrt{\Delta_n}\eta_n
  +
  \ell_n,
\]
where \(\Delta_n=\tau_{n+1}-\tau_n\), \(\eta_n\stackrel{\mathrm{ind}}{\sim}\N(0,I)\), and \(\ell_n\sim\nu_{\vart,\Delta_n}\) is the jump or Lévy increment over \([\tau_n,\tau_{n+1}]\).
The natural innovation is given by
\[
  \zeta_0
  =
  (\eta_{-1},\eta_0,\ldots,\eta_{N-1},\ell_0,\ldots,\ell_{N-1}),
\]
where \(\eta_{-1}\sim\N(0,I)\) generates the initial state if \(x_0\) is random, through an initial-state map \(x_0=\chi_{\vart,0}(\eta_{-1})\), and is omitted otherwise.
Let \(H_{\vart,\Delta_n}\) be a map such that
\[
  \ell_n=H_{\vart,\Delta_n}(w_n),
  \qquad
  w_n\sim\N(0,I).
\]
The Gaussian innovation is then given by
\[
  \xi_0
  =
  (\eta_{-1},\eta_0,\ldots,\eta_{N-1},w_0,\ldots,w_{N-1}).
\]
The corresponding generator is given by applying the jump--Euler recursion with
\[
  \ell_n=H_{\vart,\Delta_n}(w_n)
\]
and returning
\[
  T_\vart(\xi_0)=(x_0,\ldots,x_N).
\]
A Lévy process or subordinator is obtained as the special case with no state-dependent drift or diffusion term and additive Lévy increments.
\\

\midrule

Time-discretised SPDE
&
A spatially discretised SPDE of the form
\[
  \dd u_\tau=\mathcal A_\vart(u_\tau)\dd \tau+\mathcal B_\vart(u_\tau)\dd W_\tau
\]
can be time-discretised as
\[
  u_{n+1}=\Psi_{\vart,n}(u_n,\eta_n),
  \qquad n=0,\ldots,N-1,
\]
where \(\eta_n\stackrel{\mathrm{ind}}{\sim}\N(0,I)\) are finite-dimensional Gaussian noise innovations and \(\Psi_{\vart,n}\) is the numerical SPDE time-stepper.
The innovation is given by
\[
  \xi_0=(\eta_{-1},\eta_0,\ldots,\eta_{N-1}),
\]
where \(\eta_{-1}\sim\N(0,I)\) generates the initial state if \(u_0\) is random, through an initial-state map \(u_0=\chi_{\vart,0}(\eta_{-1})\), and is omitted otherwise.
The generator is given by applying the numerical SPDE solver above and returning the discretised space-time field.
\\

\midrule

\newpage
\rowcolor{lfAmber}
\multicolumn{2}{l}{\textbf{\emph{Implicit and simulator-defined priors}}}
\\
\midrule

Neural or deep implicit process
&
A neural or deep implicit process can be written as
\[
  f(s_i)
  =
  G_\vart(s_i,h,\epsilon_i),
  \qquad i=1,\ldots,m,
\]
where \(G_\vart\) is a neural decoder, \(h\) is a global latent variable, and \(\epsilon_i\) are optional local noise variables.
The innovation is given by
\[
  \xi_0=(h,\epsilon_1,\ldots,\epsilon_m),
\]
with a simple reference law, typically standard Gaussian.
The generator is given by
\[
  (T_\vart(\xi_0))_i
  =
  G_\vart(s_i,h,\epsilon_i),
  \qquad i=1,\ldots,m.
\]
If the decoder is differentiable, \(D_\xi T_\vart\) is available by automatic differentiation.
\\

\midrule

Black-box simulator
&
A black-box simulator can be written as
\[
  f_0=\mathcal S_\vart(s_{1:m};u),
\]
where \(\mathcal S_\vart\) is the simulator and \(u\sim\operatorname{Unif}((0,1)^q)\) is the simulator random seed.
The natural innovation is given by
\[
  \zeta_0=u\in(0,1)^q.
\]
Let
\(
  \xi_0=\Phi_{\mathrm N}^{-1}(u),
\)
where \(\Phi_{\mathrm N}\) is the standard normal distribution function, applied componentwise.
Then
\[
  \xi_0\sim\N(0,I_q).
\]
The corresponding generator is given by
\[
  T_\vart(\xi_0)=\mathcal S_\vart(s_{1:m};\Phi_{\mathrm N}(\xi_0)).
\]
If \(\mathcal S_\vart\) is not differentiable, guidance requires a smooth surrogate, relaxation, or derivative-free estimator.
\\

\end{longtable}
\endgroup

\section{Proofs for the latent conditioning principle}
\label[appendix]{app:basic-proofs}

\begin{proof}[Proof of \Cref{prop:latent-process-agree}]
Let \(\xi_0\sim\rho_{\vart}(\cdot\mid\calC)\).  Then, for any measurable \(A\subseteq E_m\), we have that
\begin{align*}
  \mathbb P\bigl(T_{\vart}(\xi_0)\in A\bigr)
  &=
  \frac{1}{Z_{\vart}}
  \int_{\mathbb R^{m_\xi}}
  \1_{\{T_{\vart}(\xi_0)\in A\}}
  L_{\calC}(T_{\vart}(\xi_0))\,r(\xi_0)\dd \xi_0
  \\
  &=
  \frac{1}{Z_{\vart}}
  \int_A L_{\calC}(f_0)\,p_{\vart}(f_0)\dd f_0 
  =
  \int_A \pi_{\vart}(f_0\mid\calC)\dd f_0.
\end{align*}
The third equality follows from the pushforward relation
\(T_{\vart\#}(r(\xi_0)\dd\xi_0)=p_{\vart}(f_0)\dd f_0\), applied to
the nonnegative measurable function
\(g(f_0)=\mathbf 1_A(f_0)L_{\calC}(f_0)\).
Hence \(T_{\vart}(\xi_0)\) has density \(\pi_{\vart}(\cdot\mid\calC)\), as required.
\end{proof}

\begin{remark}
\label{app:measure-theoretic-view}
We use density notation in the main text to simplify the exposition.
Our approach, however, admits a much more general measure-theoretic formulation.
Let \(P_{\vart}\) denote the prior law of \(f_0\). 
Assume that \(L_{\calC}:E_m\to[0,\infty)\) is measurable.
The conditional law on \(E_m\) is then
\begin{equation}
  \Pi_\vart(df_0\mid\calC)
  =
  \frac{L_{\calC}(f_0)P_\vart(df_0)}{Z_\vart}, \qquad Z_\vart
  \coloneqq
  \int L_{\calC}(f_0)P_\vart(df_0). \label{eq:process-target-measure}
\end{equation}
Our assumption now reads as follows. 
There exist a latent dimension \(m_\xi\), a reference probability measure \(R\) on \(\R^{m_\xi}\), and a measurable map
\(
  T_{\vart}:\R^{m_\xi}\to E_m
\)
such that
\(
  P_{\vart}
  =
  (T_{\vart})_{\#}R.
\)
Equivalently, if \(\xi_0\sim R\), then \(f_0=T_{\vart}(\xi_0)\sim P_{\vart}\).
The corresponding latent target is given by
\begin{equation}
  \rho_\vart(d\xi\mid\calC)
  =
  \frac{L_{\calC}(T_\vart(\xi))R(d\xi)}{Z_\vart}, \qquad 
  Z_\vart
  =
  \int L_{\calC}(T_\vart(\xi))R(d\xi). \label{eq:latent-target-measure}
\end{equation}
We emphasise that the values of \(Z_{\vart}\) as defined in \eqref{eq:process-target-measure} and \eqref{eq:latent-target-measure} coincide, since \(P_\vart=T_{\vart\#}R\) implies
\(
  \int L_{\calC}(f_0)P_\vart(df_0)
  =
  \int L_{\calC}(T_\vart(\xi))R(d\xi)
\)
by definition of the pushforward.
Now, for any measurable set \(A\subseteq E_m\), we have that
\[
\begin{aligned}
  T_{\vart\#}\rho_\vart(\cdot\mid\calC)(A)
  &=
  \int \1_{\{T_\vart(\xi)\in A\}}
  \frac{L_{\calC}(T_\vart(\xi))R(d\xi)}{Z_\vart}
  \\
  &=
  \frac{1}{Z_{\vart}}
  \int_A L_{\calC}(f_0)P_\vart(df_0)=
  \Pi_\vart(A\mid\calC).
\end{aligned}
\]
It follows that \(T_{\vart\#}\rho_\vart(\cdot\mid\calC)=\Pi_\vart(\cdot\mid\calC)\).
This is the measure-theoretic form of \Cref{prop:latent-process-agree}.
The density-based statement in the main text is recovered when \(P_\vart\) admits a density \(p_\vart\).
This more general argument also covers singular and discrete laws.
\end{remark}

Our next result establishes that the latent representation also preserves the relative-entropy cost of conditioning.

\begin{proposition}
\label{prop:kl-identity}
Suppose that Assumption~\ref{ass:latent-rep} holds and that \(0<Z_{\vart}<\infty\).
Then, in the extended sense,
\begin{equation}
  \KL\!\left(
    \rho_{\vart}(\cdot\mid\calC)\,\big\|\,r
  \right)
  =
  \KL\!\left(
    \pi_{\vart}(\cdot\mid\calC)\,\big\|\,p_{\vart}
  \right).
  \label{eq:kl-identity}
\end{equation}
\end{proposition}

\begin{proof}
By construction, we have
\begin{equation}
  \frac{\rho_{\vart}(\xi_0\mid\calC)}{r(\xi_0)}
  =
  \frac{\Lambda_{\vart,\calC}(\xi_0)}{Z_{\vart}}
  =
  \frac{L_{\calC}(T_{\vart}(\xi_0))}{Z_{\vart}}
  \qquad r\text{-a.e.}
  \label{eq:latent-density-ratio}
\end{equation}
Similarly, due to \eqref{eq:process-target}, it holds that
\begin{equation}
  \frac{\pi_{\vart}(f_0\mid\calC)}{p_{\vart}(f_0)}
  =
  \frac{L_{\calC}(f_0)}{Z_{\vart}}
  \qquad p_{\vart}\text{-a.e.}
  \label{eq:process-density-ratio}
\end{equation}
Let
\(
  \ell(f_0)
  =
  \log\left(\frac{L_{\calC}(f_0)}{Z_{\vart}}\right),
\)
with the usual extended-value convention.
Using \eqref{eq:latent-density-ratio}, we obtain
\begin{equation}
  \KL\!\left(
    \rho_{\vart}(\cdot\mid\calC)\,\big\|\,r
  \right)
  =
  \int_{\mathbb R^{m_\xi}}
  \ell(T_{\vart}(\xi_0))
  \rho_{\vart}(\xi_0\mid\calC)\dd\xi_0.
  \label{eq:latent-kl-as-expectation}
\end{equation}
By \Cref{prop:latent-process-agree}, pushing \(\rho_{\vart}(\cdot\mid\calC)\) forward through \(T_{\vart}\) gives \(\pi_{\vart}(\cdot\mid\calC)\).
We therefore have,
\begin{equation}
  \int_{\mathbb R^{m_\xi}}
  \ell(T_{\vart}(\xi_0))
  \rho_{\vart}(\xi_0\mid\calC)\dd\xi_0
  =
  \int_{E_m}
  \ell(f_0)
  \pi_{\vart}(f_0\mid\calC)\dd f_0.
  \label{eq:pushforward-kl-expectation}
\end{equation}
Finally, using \eqref{eq:process-density-ratio}, the right-hand side of \eqref{eq:pushforward-kl-expectation} is exactly
\(
  \KL\!\left(
    \pi_{\vart}(\cdot\mid\calC)\,\big\|\,p_{\vart}
  \right).
\)
This proves the identity.
\end{proof}

This result suggests that the latent construction is not just a convenient sampling device.
Since the additional likelihood depends on the latent variable only through the generated process value \(T_{\vart}(\xi_0)\), the relative-entropy deformation from the prior \(p_{\vart}\) to the conditional \(\pi_{\vart}(\cdot\mid\calC)\) is exactly matched by the deformation from the reference \(r\) to the latent conditional \(\rho_{\vart}(\cdot\mid\calC)\).
In contrast, for a generic latent density \(\eta\) on \(\mathbb R^{m_\xi}\), if \(p_{\eta}\) denotes the density of \(T_{\vart}(\xi_0)\) when \(\xi_0\sim\eta\), then the data-processing inequality only gives
\(
  \KL(p_{\eta}\,\|\,p_{\vart})
  \le
  \KL(\eta\,\|\,r)
\).
Thus the latent target introduces no additional information beyond the output-level conditioning likelihood.
\section{Exactness of the latent dynamics}
\label[appendix]{app:exactness-proofs}

We now fix \(\vart\) and \(\calC\), and write \(T=T_{\vart}\), \(\Lambda=\Lambda_{\vart,\calC}\), \(Z=Z_{\vart}\), and \(\rho_0(\xi_0)=\rho_{\vart}(\xi_0\mid\calC)\).
We also write \(r(\xi)=\N(\xi;0,I_{m_\xi})\).

Recall that the VP noising schedule is specified by a locally integrable function \(\beta:[0,1)\to(0,\infty)\).
We write
\begin{equation}
  B(t)
  =
  \int_0^t\beta(s)\dd s,
  \qquad
  \alpha(t)
  =
  \exp\big\{-\frac12 B(t)\big\},
  \qquad
  \sigma(t)^2
  =
  1-\alpha(t)^2.
  \label{eq:appendix-vp-schedule}
\end{equation}
We now allow either a finite-noising horizon, where
\(
  B(1):=\lim_{t\uparrow1}B(t)<\infty,
\)
or an ideal infinite-noising horizon, where \(B(1)=\infty\) and hence \(\alpha(t)\downarrow0\) as \(t\uparrow1\).
For \(t\in[0,1)\), the VP bridge is
\begin{equation}
  \xi_t
  =
  \alpha(t)\xi_0+\sigma(t)\varepsilon,
  \qquad
  \varepsilon\sim \N(0,I_{m_{\xi}}),
  \label{eq:appendix-vp-bridge}
\end{equation}
where \(\varepsilon\) is independent of \(\xi_0\).
We write \(\rho_t\) for the density of the noised latent variable obtained by drawing \(\xi_0\sim\rho_0\) and applying \eqref{eq:appendix-vp-bridge}.
Equivalently,
\begin{equation}
  \rho_t(u)
  =
  \frac{\psi(t,u)r(u)}{Z},
  \qquad
  \psi(t,u)
  =
  \mathbb E_{\xi_0\sim r}\left[
    \Lambda(\xi_0)\mid \xi_t=u
  \right].
  \label{eq:appendix-guided-marginal-density}
\end{equation}
where the conditional expectation is taken under the reference bridge with \(\xi_0\sim r\).
Finally, assuming $\psi(t,\cdot)$ is differentiable, we write \(g(t,u)=\nabla_u\log\psi(t,u)\) for the latent guidance field.

In \Cref{sec:exact-latent}, we formally identified the SDE associated with the guided latent marginals \((\rho_t)_{0\le t<1}\).
We now recall these dynamics using notation that is convenient for the exactness statements below.
The forward noising SDE is given by
\begin{equation}
  \dd \xi_t
  =
  -\frac12\beta(t)\xi_t\,\dd t
  +
  \sqrt{\beta(t)}\,\dd w_t,
  \qquad
  \xi_0\sim\rho_0,
  \qquad
  0\leq t<1,
  \label{eq:forward-vp-sde}
\end{equation}
where \(w=(w_t)_{0\le t<1}\) is a standard \(\mathbb R^{m_\xi}\)-valued Brownian motion.
Fix \(0\le t_0<t_1<1\), and set
\(
  \tau_s=t_1-s,
\) 
for
\(
  0\le s\le t_1-t_0.
\)
The reverse-time SDE is then given by
\begin{equation}
  \dd \xi_s^{\leftarrow}
  =
  \left[
    \beta(\tau_s)g(\tau_s,\xi_s^{\leftarrow})
    -
    \frac12\beta(\tau_s)\xi_s^{\leftarrow}
  \right]\dd s
  +
  \sqrt{\beta(\tau_s)}\,\dd \overline w_s,
  \qquad
  \xi_0^{\leftarrow}\sim\rho_{t_1},
  \label{eq:latent-reverse-sde}
\end{equation}
where \(\overline w = (\overline{w}_s)_{0\leq s\leq t_1-t_0}\) is a standard Brownian motion in reverse time.
The same family of marginals can also be generated by a deterministic process known as the probability-flow ODE \citep[e.g.,][]{song2021score}.
In forward noising time, this ODE is given by
\begin{equation}
  \frac{\dd u_t}{\dd t}
  =
  -\frac12\beta(t)g(t,u_t),
  \qquad
  u_0\sim\rho_0,
  \qquad
  0\le t<1.
  \label{eq:latent-pf-ode-forward}
\end{equation}
For sampling, we again use the reverse-time parametrisation \(u_s^{\leftarrow}=u_{\tau_s}\), with $\tau_s$ as defined above.
In particular, the reverse-time probability-flow ODE is
\begin{equation}
  \frac{\dd u^{\leftarrow}_s}{\dd s}
  =
  \frac12\beta(\tau_s)g(\tau_s,u^{\leftarrow}_s), \qquad u_0^{\leftarrow} \sim \rho_{t_1}.
  \label{eq:latent-pf-ode}
\end{equation}

\begin{assumption}
\label{ass:pf-regularity}
For every compact interval \([t_0,t_1]\subset(0,1)\), \(\rho_t\) is a strictly positive Lebesgue density for all \(t\in[t_0,t_1]\).
The map \((t,u)\mapsto \nabla_u\log\rho_t(u)\) is sufficiently regular that the reverse-time SDE \eqref{eq:latent-reverse-sde} and the probability-flow ODE \eqref{eq:latent-pf-ode} are well posed on the corresponding time intervals, and their associated Fokker--Planck and continuity equations have unique density-valued solutions.
\end{assumption}

\begin{theorem}
\label{thm:latent-reverse-sde-exactness}
Suppose that Assumption~\ref{ass:latent-rep} holds, that \(r=\N(0,I_{m_\xi})\), that \(0<Z<\infty\), and that Assumption~\ref{ass:pf-regularity} holds.
Fix \(0<t_0<t_1<1\), and define \(\tau_s=t_1-s\) for \(0\le s\le t_1-t_0\).
Let \(\xi^{\leftarrow}_0\sim\rho_{t_1}\), and let \((\xi^{\leftarrow}_s)_{0\le s\le t_1-t_0}\) solve the reverse-time SDE \eqref{eq:latent-reverse-sde}.
Then \(\xi^{\leftarrow}_s\sim\rho_{\tau_s}\) for all \(0\le s\le t_1-t_0\).
In particular, \(\xi^{\leftarrow}_{t_1-t_0}\sim\rho_{t_0}\).
If the reverse SDE is well defined down to \(t_0=0\), then
\begin{equation}
  \xi^{\leftarrow}_{t_1}\sim\rho_0
  =
  \rho_{\vart}(\cdot\mid\calC),
  \qquad
  T_{\vart}(\xi^{\leftarrow}_{t_1})\sim\pi_{\vart}(\cdot\mid\calC).
\end{equation}
\end{theorem}

\begin{proof}
The proof is standard. 
Let
\(
  b_t(u)
  =
  -\frac12\beta(t)u
\)
denote the drift of the forward VP SDE.
By the standard time-reversal formula for diffusions \citep[e.g.,][]{anderson1982reverse}, the reverse-time diffusion with clock \(\tau_s=t_1-s\) has drift given by
\begin{equation}
  \widetilde b_s(u)
  =
  -b_{\tau_s}(u)
  +
  \beta(\tau_s)\nabla_u\log\rho_{\tau_s}(u).
  \label{eq:reverse-drift-from-score}
\end{equation}
Moreover, since \(\rho_t(u)=Z^{-1}\psi(t,u)r(u)\) and \(r(u)=\N(u;0,I_{m_\xi})\), it follows straightforwardly that
\begin{equation}
  \nabla_u\log\rho_t(u)
  =
  \nabla_u\log\psi(t,u)-u
  =
  g(t,u)-u.
  \label{eq:rho-score-psi-score}
\end{equation}
We thus have, substituting \eqref{eq:rho-score-psi-score} into \eqref{eq:reverse-drift-from-score}, that
\begin{align}
  \widetilde b_s(u)
  =
  \frac12\beta(\tau_s)u
  +
  \beta(\tau_s)\left(g(\tau_s,u)-u\right)
  =
  \beta(\tau_s)g(\tau_s,u)
  -
  \frac12\beta(\tau_s)u.
\end{align}
This is exactly the drift in \eqref{eq:latent-reverse-sde}.
Equivalently, if \(q_s(u)=\rho_{\tau_s}(u)\), then \(q_s\) solves the Fokker--Planck equation associated with \eqref{eq:latent-reverse-sde}.
Indeed,
\begin{align}
  \partial_s q_s
  =
  -\partial_t\rho_t\big|_{t=\tau_s}
  =
  \nabla\cdot(b_{\tau_s}\rho_{\tau_s})
  -
  \frac12\beta(\tau_s)\Delta \rho_{\tau_s}
  =
  -\nabla\cdot(\widetilde b_s q_s)
  +
  \frac12\beta(\tau_s)\Delta q_s.
\end{align}
By uniqueness in Assumption~\ref{ass:pf-regularity}, the law of \(\xi^{\leftarrow}_s\) has density \(q_s=\rho_{\tau_s}\).
The claims at \(t_0=0\) follow because \(\rho_0=\rho_{\vart}(\cdot\mid\calC)\). 
Finally, the process-space statement follows from \Cref{prop:latent-process-agree}.
\end{proof}

\begin{theorem}
\label{thm:latent-pf-exactness}
Suppose that Assumption~\ref{ass:latent-rep} holds, that \(r=\N(0,I_{m_\xi})\), that \(0<Z<\infty\), and that Assumption~\ref{ass:pf-regularity} holds.
Fix \(0<t_0<t_1<1\), and define \(\tau_s=t_1-s\) for \(0\le s\le t_1-t_0\).
If \(u^{\leftarrow}_0\sim\rho_{t_1}\) and \((u^{\leftarrow}_s)_{0\le s\le t_1-t_0}\) solves the reverse-time probability-flow ODE \eqref{eq:latent-pf-ode}, then
\(
  u^{\leftarrow}_s\sim\rho_{\tau_s}
\)
for all $0\le s\le t_1-t_0$.
In particular,
\(
  u^{\leftarrow}_{t_1-t_0}\sim\rho_{t_0}.
\)
If the flow is well defined down to \(t_0=0\), then
\begin{equation}
  u^{\leftarrow}_{t_1}\sim\rho_0
  =
  \rho_{\vart}(\cdot\mid\calC),
  \qquad
  T_{\vart}(u^{\leftarrow}_{t_1})\sim\pi_{\vart}(\cdot\mid\calC).
\end{equation}
\end{theorem}

\begin{proof}
The marginals \((\rho_t)_{0\le t<1}\) are generated by the forward VP-SDE when the initial law is \(\rho_0=\rho_{\vart}(\cdot\mid\calC)\).
Indeed, for any bounded measurable test function \(\varphi\),
\begin{align}
  \mathbb E_{\xi_0\sim\rho_0}\left[\varphi(\xi_t)\right]
  &=
  \frac1Z
  \mathbb E_{\xi_0\sim r,\varepsilon\sim r}
  \left[
    \Lambda(\xi_0)
    \varphi\left(\alpha(t)\xi_0+\sigma(t)\varepsilon\right)
  \right]
  \nonumber
  \\
  &=
  \frac1Z
  \mathbb E_{u\sim r}
  \left[
    \varphi(u)
    \mathbb E\left[\Lambda(\xi_0)\mid \xi_t=u\right]
  \right]
  =
  \int_{\mathbb R^{m_\xi}}
  \varphi(u)\rho_t(u)\dd u.
\end{align}
Thus \(\rho_t\) solves the Fokker--Planck equation
\begin{equation}
  \partial_t\rho_t
  =
  -\nabla\cdot(b_t\rho_t)
  +
  \frac12\beta(t)\Delta\rho_t,
  \qquad
  b_t(u)=-\frac12\beta(t)u.
  \label{eq:vp-fokker-planck-rho}
\end{equation}
Writing \(s_t(u)=\nabla_u\log\rho_t(u)\), we can rewrite \eqref{eq:vp-fokker-planck-rho} as
\begin{align}
  \partial_t\rho_t
  &=
  -\nabla\cdot(b_t\rho_t)
  +
  \frac12\beta(t)\nabla\cdot(\rho_t s_t)
  =
  -\nabla\cdot
  \big(
    \big[
      b_t
      -
      \frac12\beta(t)s_t
    \big]
    \rho_t
  \big).
\end{align}
It follows that \(\rho_t\) also solves the continuity equation
\begin{equation}
  \partial_t\rho_t+\nabla\cdot(v_t\rho_t)=0,
  \qquad
  v_t(u)=b_t(u)-\frac12\beta(t)s_t(u).
  \label{eq:pf-continuity-equation}
\end{equation}
Meanwhile, substituting \eqref{eq:rho-score-psi-score}, we have that 
\begin{align}
  v_t(u)
  &=
  -\frac12\beta(t)u
  -
  \frac12\beta(t)\left(g(t,u)-u\right)
  =
  -\frac12\beta(t)g(t,u).
\end{align}
This is exactly the vector field in the forward probability-flow ODE \eqref{eq:latent-pf-ode-forward}.
By Assumption~\ref{ass:pf-regularity}, the ODE flow is well posed and gives the unique solution of the continuity equation.
Therefore the backward flow \eqref{eq:latent-pf-ode} transports \(\rho_{t_1}\) to \(\rho_{t_0}\).
Taking \(t_0=0\) gives \(u^{\leftarrow}_{t_1}\sim\rho_0=\rho_{\vart}(\cdot\mid\calC)\), and the process-space statement follows from \Cref{prop:latent-process-agree}.
\end{proof}

\begin{remark}
If \(B(1)=\infty\), then \(\alpha(t)\downarrow0\) and the noised target marginals satisfy
\(
  \rho_t\to r
\)
in total variation as \(t\uparrow1\), provided \(\Lambda\in L^1(r)\).
In this sense, starting from \(u^{\leftarrow}_0\sim r\) can be viewed as the limiting infinite-noising version of \Cref{thm:latent-pf-exactness}.
For a finite noising horizon, exact sampling from a terminal time \(t_1\) requires initialising from the terminal marginal \(\rho_{t_1}\), not from
\(r\).
If one instead starts from \(r\), the initialisation discrepancy is \(\|\rho_{t_1}-r\|_{\mathrm{TV}}\), and the discrepancy after exact reverse probability flow is at most this quantity, with equality when the deterministic flow is invertible.
\end{remark}
\section{Approximation and stability guarantees}
\label[appendix]{sec:approx-guidance-stability}

The practical sampler replaces the exact guidance field \(g(t,u)\) by an approximate field \(\widehat g_S(t,u)\), and then solves the resulting ODE or SDE numerically.
This introduces two sources of error: the guidance approximation error and the numerical discretisation error.
Both errors enter the probability-flow ODE and the reverse-time SDE through the same perturbation mechanism.
We thus state the main stability and numerical-error bounds in a unified form.

\subsection{Common perturbation framework}
\label{subsec:latent-common-perturbation-framework}

We work on a truncated forward-time interval \([\tau,t_1]\), where \(0<\tau<t_1<1\), and use reverse-time notation throughout this subsection.
Set
\begin{equation}
  R_{\tau,t_1}=t_1-\tau,
  \qquad
  \kappa_s=t_1-s,
  \qquad
  0\le s\le R_{\tau,t_1}.
  \label{eq:latent-reverse-time-notation}
\end{equation}
Thus \(s\) denotes reverse time and \(\kappa_s\) denotes the corresponding forward noising time.
We index the two reverse dynamics by \(\chi\in\{\mathrm{ode},\mathrm{sde}\}\), and define
\begin{equation}
\begin{array}{c|ccc}
\chi
& \lambda_\chi
& \rho_\chi
& \Sigma_\chi(s) \\
\hline
\mathrm{ode}
& \frac12
& 0
& 0 \\
\mathrm{sde}
& 1
& \frac12
& \sqrt{\beta(\kappa_s)}I_{m_\xi}.
\end{array}
\label{eq:latent-dynamics-constants}
\end{equation}
For an exact guidance field \(g\) and an approximate guidance field \(\widehat g\), define the exact and approximate reverse-time drifts
\begin{equation}
  F_\chi(s,u)
  =
  \lambda_\chi\beta(\kappa_s)g(\kappa_s,u)
  -
  \rho_\chi\beta(\kappa_s)u,
  \qquad
  \widehat F_\chi(s,u)
  =
  \lambda_\chi\beta(\kappa_s)\widehat g(\kappa_s,u)
  -
  \rho_\chi\beta(\kappa_s)u.
  \label{eq:latent-unified-reversed-drifts}
\end{equation}
In particular, for the probability-flow ODE we have
\begin{align}
  a(s,u)=F_{\mathrm{ode}}(s,u)
  &=
  \frac12\beta(\kappa_s)g(\kappa_s,u),
  \\
  \widehat a(s,u)=\widehat F_{\mathrm{ode}}(s,u)
  &=
  \frac12\beta(\kappa_s)\widehat g(\kappa_s,u).
  \label{eq:latent-reversed-drifts}
\end{align}
Meanwhile, for the reverse-time SDE we have
\begin{align}
  b(s,u)=F_{\mathrm{sde}}(s,u)
  &=
  \beta(\kappa_s)g(\kappa_s,u)
  -
  \frac12\beta(\kappa_s)u,
  \\
  \widehat b(s,u)=\widehat F_{\mathrm{sde}}(s,u)
  &=
  \beta(\kappa_s)\widehat g(\kappa_s,u)
  -
  \frac12\beta(\kappa_s)u.
  \label{eq:latent-reverse-sde-drifts}
\end{align}
The exact and approximate dynamics may then be written as
\begin{equation}
  \dd X^\chi_s
  =
  F_\chi(s,X^\chi_s)\dd s
  +
  \Sigma_\chi(s)\dd W_s,
  \qquad
  \dd \widehat X^\chi_s
  =
  \widehat F_\chi(s,\widehat X^\chi_s)\dd s
  +
  \Sigma_\chi(s)\dd W_s.
  \label{eq:latent-unified-exact-approx-dynamics}
\end{equation}
When \(\chi=\mathrm{ode}\), we write
\(
  X^{\mathrm{ode}}=u^{\leftarrow},
  \widehat X^{\mathrm{ode}}=\widehat u^{\leftarrow}.
\)
When \(\chi=\mathrm{sde}\), we write
\(
  X^{\mathrm{sde}}=\xi^{\leftarrow},
  \widehat X^{\mathrm{sde}}=\widehat \xi^{\leftarrow}.
\)

Let \(\mathcal R_{\tau,t_1}\subset[\tau,t_1]\times\mathbb R^{m_\xi}\) be a region on which the exact and approximate guidance fields are compared, and which contains the exact and approximate trajectories under consideration.
Define the realised pointwise guidance error by
\begin{equation}
  \varepsilon(s)
  =
  \sup\Bigl\{
    \|g(\kappa_s,u)-\widehat g(\kappa_s,u)\|:
    (\kappa_s,u)\in\mathcal R_{\tau,t_1}
  \Bigr\}, \qquad \delta_\chi(s)
  =
  \lambda_\chi\beta(\kappa_s)\varepsilon(s).
  \label{eq:latent-pointwise-guidance-error}
\end{equation}
For the ODE and the SDE, we thus have that
\begin{equation}
  \delta_{\mathrm{ode}}(s)
  =
  \frac12\beta(\kappa_s)\varepsilon(s),
  \qquad
  \delta_{\mathrm{sde}}(s)
  =
  \beta(\kappa_s)\varepsilon(s).
  \label{eq:latent-ode-error-functions}
\end{equation}
In the common case where
\(
  \varepsilon_{\tau,t_1}
  :=
  \sup_{(t,u)\in\mathcal R_{\tau,t_1}}
  \|g(t,u)-\widehat g(t,u)\|<\infty,
\)
it follows from \eqref{eq:latent-pointwise-guidance-error} that
\(
  \delta_\chi(s)
  \le
  \lambda_\chi\beta(\kappa_s)\varepsilon_{\tau,t_1}.
\)

Given a time-dependent one-sided Lipschitz rate \(\eta\in L^1([0,R_{\tau,t_1}])\), we define the integral
\begin{equation}
  \Theta_\eta(s,r)
  =
  \exp\left\{\int_r^s \eta(q)\dd q\right\},
  \qquad
  0\le r\le s\le R_{\tau,t_1}.
  \label{eq:latent-theta-eta}
\end{equation}
Finally, throughout the statements below, we treat \(\widehat g\) as a realised approximate guidance field.
If \(\widehat g\) is random, our results are to be read conditionally on this realisation.
Probabilistic rates follow by combining these conditional stability bounds with separate probabilistic control of the realised guidance error.

\subsection{Unified stability of approximate guidance}
\label{subsec:latent-unified-stability}

The following result is a variable-rate version of the deterministic perturbation bound obtained in \citep[Proposition~E.5]{moss2026conditioning}, restated in our present notation and extended to cover the reverse-time SDE under a synchronous coupling.

\begin{proposition}
\label{prop:reverse-sde-synchronous-stability}
Fix \(\chi\in\{\mathrm{ode},\mathrm{sde}\}\), and let \(X^\chi\) and
\(\widehat X^\chi\) solve
\eqref{eq:latent-unified-exact-approx-dynamics} on
\([0,R_{\tau,t_1}]\).
For the SDE, assume that the two equations are driven by the same
Brownian motion.
Assume that
\(
  (\kappa_s,X^\chi_s),(\kappa_s,\widehat X^\chi_s)
  \in
  \mathcal R_{\tau,t_1}
\)
for all
\(
  s\in[0,R_{\tau,t_1}].
\)
Suppose that the exact reverse-time drift satisfies the one-sided
Lipschitz condition
\begin{equation}
  \langle x-y,F_\chi(s,x)-F_\chi(s,y)\rangle
  \le
  \eta_\chi(s)\|x-y\|^2
  \label{eq:latent-osl}
\end{equation}
whenever \((\kappa_s,x)\) and \((\kappa_s,y)\) belong to
\(\mathcal R_{\tau,t_1}\).
Then, for every \(0\le s\le R_{\tau,t_1}\), the following bound holds:
\begin{equation}
  \|X^\chi_s-\widehat X^\chi_s\|
  \le
  \Theta_{\eta_\chi}(s,0)\|X^\chi_0-\widehat X^\chi_0\|
  +
  \int_0^s \Theta_{\eta_\chi}(s,r)\delta_\chi(r)\dd r.
  \label{eq:latent-guidance-error-bound-reverse-time}
\end{equation}
If the exact and approximate solves are coupled by the same reverse-time
initial point, equivalently by the same terminal point at forward time
\(t_1\), the first term is absent.
Moreover, for  \(p\ge1\), provided the initial laws have finite \(p\)-th moments, 
\begin{equation}
  W_p\!\left(\Law(X^\chi_s),\Law(\widehat X^\chi_s)\right)
  \le
  \Theta_{\eta_\chi}(s,0)
  W_p\!\left(\Law(X^\chi_0),\Law(\widehat X^\chi_0)\right)
  +
  \int_0^s \Theta_{\eta_\chi}(s,r)\delta_\chi(r)\dd r.
  \label{eq:latent-unified-wasserstein-stability}
\end{equation}
\end{proposition}

\begin{proof}
Let \(D_s=X^\chi_s-\widehat X^\chi_s\).
When \(\chi=\mathrm{ode}\), this is an ordinary deterministic difference equation.
When \(\chi=\mathrm{sde}\), the Brownian terms cancel under the synchronous coupling.
In either case, we have
\begin{equation}
  D_s
  =
  D_0
  +
  \int_0^s
  \{F_\chi(r,X^\chi_r)-\widehat F_\chi(r,\widehat X^\chi_r)\}\dd r.
\end{equation}
We can write the integrand as
\(
  F_\chi(r,X^\chi_r)-\widehat F_\chi(r,\widehat X^\chi_r)
  =
  F_\chi(r,X^\chi_r)-F_\chi(r,\widehat X^\chi_r)
  +
  F_\chi(r,\widehat X^\chi_r)-\widehat F_\chi(r,\widehat X^\chi_r).
\)
The first term is controlled by \eqref{eq:latent-osl}.
The second term satisfies
\begin{equation}
  \|F_\chi(r,u)-\widehat F_\chi(r,u)\|
  =
  \lambda_\chi\beta(\kappa_r)
  \|g(\kappa_r,u)-\widehat g(\kappa_r,u)\|
  \le
  \delta_\chi(r)
\end{equation}
on the region under consideration.
Hence the upper right-hand derivative satisfies
\begin{equation}
  \frac{\dd^+}{\dd s}\|D_s\|
  \le
  \eta_\chi(s)\|D_s\|
  +
  \delta_\chi(s).
\end{equation}
The variable-coefficient Gronwall inequality gives \eqref{eq:latent-guidance-error-bound-reverse-time}.
The Wasserstein estimate follows by coupling the initial variables optimally and, in the SDE case, using the same Brownian motion for the two SDEs.
\end{proof}

\begin{corollary}
\label{cor:latent-constant-rate-simplification}
Suppose the conditions of \Cref{prop:reverse-sde-synchronous-stability} hold. Suppose in addition that \(\eta_\chi(s)\le\bar\eta_{\chi,\tau,t_1}\) on
\([0,R_{\tau,t_1}]\) and
\(
  \bar\beta_{\tau,t_1}
  =
  \operatorname*{ess\,sup}_{t\in[\tau,t_1]}\beta(t)
  <\infty.
\)
Then, under common terminal coupling, it holds for all $s\in[0,R_{\tau,t_1}]$ that
\begin{equation}
  \|X^\chi_s-\widehat X^\chi_s\|
  \le
  \lambda_\chi\bar\beta_{\tau,t_1}\varepsilon_{\tau,t_1}
  \Psi_{\bar\eta_{\chi,\tau,t_1}}(s), \qquad \Psi_\eta(r)
  =
  \begin{cases}
    \frac{e^{\eta r}-1}{\eta}, & \eta\ne0,\\
    r, & \eta=0.
  \end{cases}
  \label{eq:latent-unified-constant-rate-bound}
\end{equation}
\end{corollary}

\begin{proof}
Under common terminal coupling, the first term in
\eqref{eq:latent-guidance-error-bound-reverse-time} vanishes.
Moreover, since \(\eta_\chi(s)\le\bar\eta_{\chi,\tau,t_1}\) and
\(\delta_\chi(s)\le\lambda_\chi\bar\beta_{\tau,t_1}\varepsilon_{\tau,t_1}\), we have
\(
  \Theta_{\eta_\chi}(s,r)
  \le
  \exp\{\bar\eta_{\chi,\tau,t_1}(s-r)\}.
\)
Therefore
\[
  \|X^\chi_s-\widehat X^\chi_s\|
  \le
  \lambda_\chi\bar\beta_{\tau,t_1}\varepsilon_{\tau,t_1}
  \int_0^s
  \exp\{\bar\eta_{\chi,\tau,t_1}(s-r)\}
  \dd r=\lambda_\chi\bar\beta_{\tau,t_1}\varepsilon_{\tau,t_1}
  \Psi_{\bar\eta_{\chi,\tau,t_1}}(s).
\]
\end{proof}

\subsection{Curvature regimes}
\label{subsec:latent-curvature-regimes}

The sign of the one-sided Lipschitz rate $\eta_{\chi}$ determines the nature of the bound in \Cref{prop:reverse-sde-synchronous-stability} and \Cref{cor:latent-constant-rate-simplification}.
This, in turn, is determined by the curvature of the smoothed latent log-likelihood. 
This is the subject of the following proposition.

\begin{proposition}
\label{prop:latent-curvature-regimes}
Suppose that, uniformly for \(t\in[\tau,t_1]\), $g$ is one-sided Lipschitz with constant $\ell_{\tau,t_1}$ on the truncated region:
\begin{equation}
  \langle x-y,g(t,x)-g(t,y)\rangle
  \le
  \ell_{\tau,t_1}\|x-y\|^2
  \label{eq:latent-general-monotonicity}
\end{equation}
Then \Cref{prop:reverse-sde-synchronous-stability} applies with
\begin{equation}
  \eta_\chi(s)
  =
  \beta(\kappa_s)
  \left(\lambda_\chi\ell_{\tau,t_1}-\rho_\chi\right).
  \label{eq:latent-curvature-rate}
\end{equation}
Suppose, more strongly, that uniformly for $t\in[\tau,t_1]$, $g$ is strongly dissipative with constant \(\mu_{\tau,t_1}\ge0\), on the truncated region:
\begin{equation}
  \langle x-y,g(t,x)-g(t,y)\rangle
  \le
  -\mu_{\tau,t_1}\|x-y\|^2.
  \label{eq:latent-strong-monotone}
\end{equation}
Then \Cref{prop:reverse-sde-synchronous-stability} applies with
\begin{equation}
  \eta_\chi(s)
  =
  -\left(\rho_\chi+\lambda_\chi\mu_{\tau,t_1}\right)\beta(\kappa_s).
  \label{eq:latent-contractive-rate}
\end{equation}
\end{proposition}

\begin{proof}
Recalling the definition of the reverse-time drift $F_{\chi}$ from \eqref{eq:latent-unified-reversed-drifts}, it is straightforward to compute
\begin{align}
  \langle x-y,F_\chi(s,x)-F_\chi(s,y)\rangle
  &=
  \lambda_\chi\beta(\kappa_s)
  \langle x-y,g(\kappa_s,x)-g(\kappa_s,y)\rangle
  -
  \rho_\chi\beta(\kappa_s)\|x-y\|^2.
\end{align}
The rates given in \eqref{eq:latent-curvature-rate} and \eqref{eq:latent-contractive-rate} now follow immediately after substituting \eqref{eq:latent-general-monotonicity} and \eqref{eq:latent-strong-monotone}, respectively.
\end{proof}

\subsection{Guidance and discretisation error}
\label{subsec:latent-guidance-discretisation-error}

The previous stability result controls the error between the exact guided dynamics and the exact approximately guided dynamics.
It is also necessary to control the discretisation error introduced by a numerical solver.

\begin{corollary}
\label{cor:latent-guidance-discretisation-error}
Assume that the hypotheses of \Cref{prop:reverse-sde-synchronous-stability} hold for the exact and approximately guided ODE and SDE dynamics.
Let \(0=s_0<\cdots<s_N=R_{\tau,t_1}\) be a reverse-time grid with maximal step size \(h\).
Define
\begin{equation}
  A_{\eta_\chi}
  =
  \sup_{0\le s\le R_{\tau,t_1}}\Theta_{\eta_\chi}(s,0),
  \qquad
  H_{\eta_\chi,\delta_\chi}
  =
  \sup_{0\le s\le R_{\tau,t_1}}
  \int_0^s\Theta_{\eta_\chi}(s,r)\delta_\chi(r)\dd r.
  \label{eq:latent-A-H-stability-constants}
\end{equation}
Suppose that the numerical approximations \(\widehat u^{\leftarrow,h}\) and \(\widehat\xi^{\leftarrow,h}\) to the approximately guided ODE and SDE satisfy
\begin{equation}
  \max_{0\le n\le N}
  \|\widehat u^{\leftarrow}_{s_n}-\widehat u^{\leftarrow,h}_{s_n}\|
  \le
  C^{\mathrm{ode}}_{\mathrm{num},\tau,t_1}h^q,
  \qquad
  \big\|
    \max_{0\le n\le N}
    \|\widehat\xi^{\leftarrow}_{s_n}-\widehat\xi^{\leftarrow,h}_n\|
  \big\|_{L^p}
  \le
  C^{\mathrm{sde}}_{\mathrm{num},p}h^\varrho.
  \label{eq:latent-numerical-error-assumption}
\end{equation}
Then it holds that
\begin{align}
  \max_{0\le n\le N}
  \|u^{\leftarrow}_{s_n}-\widehat u^{\leftarrow,h}_{s_n}\|
  &\le
  A_{\eta_{\mathrm{ode}}}\|u^{\leftarrow}_0-\widehat u^{\leftarrow}_0\|
  +
  H_{\eta_{\mathrm{ode}},\delta_{\mathrm{ode}}}
  +
  C^{\mathrm{ode}}_{\mathrm{num},\tau,t_1}h^q,
  \label{eq:latent-ode-total-error-bound-reverse-time}
  \\
  \big\|
    \max_{0\le n\le N}
    \|\xi^{\leftarrow}_{s_n}-\widehat\xi^{\leftarrow,h}_n\|
  \big\|_{L^p}
  &\le
  A_{\eta_{\mathrm{sde}}}
  \|\xi^{\leftarrow}_0-\widehat\xi^{\leftarrow}_0\|_{L^p}
  +
  H_{\eta_{\mathrm{sde}},\delta_{\mathrm{sde}}}
  +
  C^{\mathrm{sde}}_{\mathrm{num},p}h^\varrho.
  \label{eq:latent-sde-total-error-bound}
\end{align}
In both displays, the initial-discrepancy term vanishes under common terminal coupling.
\end{corollary}

\begin{proof}
Apply \Cref{prop:reverse-sde-synchronous-stability} at the grid points and take the maximum over \(n\).
For the ODE, add the deterministic numerical error by the triangle inequality.
For the SDE, first use the pathwise stability bound, then take the \(L^p\)-norm and apply Minkowski's inequality together with \eqref{eq:latent-numerical-error-assumption}.
\end{proof}

\begin{corollary}
\label{cor:latent-rate-consequence}
Assume that the hypotheses of \Cref{cor:latent-guidance-discretisation-error} hold under a common terminal coupling, and consider a joint limit in which \(S\to\infty\), \(h\to0\), and \(a_S\to0\).
Suppose the realised guidance error satisfies
\(
  \varepsilon_{\tau,t_1}=O_{\mathbb P}(a_S)
\)
on \(\mathcal R_{\tau,t_1}\).
For \(\chi\in\{\mathrm{ode},\mathrm{sde}\}\), define
\[
  B_{\chi,\tau,t_1}
  =
  \sup_{0\le s\le R_{\tau,t_1}}
  \int_0^s
  \Theta_{\eta_\chi}(s,r)
  \lambda_\chi\beta(\kappa_r)
  \dd r.
\]
Assume that the stability and numerical-error constants are uniformly
controlled along this limit, in the sense that
\(
  B_{\chi,\tau,t_1}=O_{\mathbb P}(1),
\)
\(
  C^{\mathrm{ode}}_{\mathrm{num},\tau,t_1}=O_{\mathbb P}(1),
\)
\(
  C^{\mathrm{sde}}_{\mathrm{num},p}=O_{\mathbb P}(1).
\)
This is automatic when these quantities are deterministic
and bounded independently of \(S\) and \(h\).
Then
\begin{equation}
  \max_{0\le n\le N}
  \|u^{\leftarrow}_{s_n}-\widehat u^{\leftarrow,h}_{s_n}\|
  =
  O_{\mathbb P}(a_S+h^q),
  \qquad
  \max_{0\le n\le N}
  \|\xi^{\leftarrow}_{s_n}-\widehat\xi^{\leftarrow,h}_n\|
  =
  O_{\mathbb P}(a_S+h^\varrho).
  \label{eq:latent-joint-guidance-numerical-rates}
\end{equation}
In particular, if a uniform Monte Carlo guidance bound gives
\(a_S=S^{-1/2}\), then the corresponding rates are
\(O_{\mathbb P}(S^{-1/2}+h^q)\) for the ODE and
\(O_{\mathbb P}(S^{-1/2}+h^\varrho)\) for the SDE.
\end{corollary}

\begin{proof}
Under a common terminal coupling, the initial-discrepancy terms in
\Cref{cor:latent-guidance-discretisation-error} vanish.
Moreover,
\[
  H_{\eta_\chi,\delta_\chi}
  \le
  B_{\chi,\tau,t_1}\varepsilon_{\tau,t_1}
  =
  O_{\mathbb P}(a_S).
\]
For the ODE, the numerical term satisfies
\(
  C^{\mathrm{ode}}_{\mathrm{num},\tau,t_1}h^q
  =O_{\mathbb P}(h^q)
\), so the first conclusion follows from
\eqref{eq:latent-ode-total-error-bound-reverse-time}.
For the SDE, the conditional \(L^p\) estimate in
\eqref{eq:latent-numerical-error-assumption}, together with
\(C^{\mathrm{sde}}_{\mathrm{num},p}=O_{\mathbb P}(1)\), implies that
\[
  \max_{0\le n\le N}
  \|\widehat\xi^{\leftarrow}_{s_n}
  -\widehat\xi^{\leftarrow,h}_n\|
  =
  O_{\mathbb P}(h^\varrho).
\]
due to Markov's inequality.
The second conclusion follows from the triangle inequality.
\end{proof}

\subsection{Process-space propagation}
\label{subsec:latent-process-propagation}

The latent bounds pass directly to process space whenever the generator is Lipschitz on the region visited by the coupled latent trajectories.

\begin{corollary}
\label{cor:latent-process-error}
Assume the hypotheses of \Cref{cor:latent-guidance-discretisation-error}.
Suppose that \(E_m\) is equipped with a norm and that \(T:\mathbb R^{m_\xi}\to E_m\) is \(L_T\)-Lipschitz on the latent region visited by the exact, approximately guided, and numerical trajectories.
Then the ODE and SDE bounds in \Cref{cor:latent-guidance-discretisation-error} remain valid after applying \(T\), with the right-hand sides multiplied by \(L_T\).
In particular,
\begin{align}
  \max_{0\le n\le N}
  \|T(u^{\leftarrow}_{s_n})-T(\widehat u^{\leftarrow,h}_{s_n})\|
  &\le
  L_T
  \left(
    A_{\eta_{\mathrm{ode}}}\|u^{\leftarrow}_0-\widehat u^{\leftarrow}_0\|
    +
    H_{\eta_{\mathrm{ode}},\delta_{\mathrm{ode}}}
    +
    C^{\mathrm{ode}}_{\mathrm{num},\tau,t_1}h^q
  \right),
  \label{eq:latent-process-ode-error}
  \\
  \left\|
    \max_{0\le n\le N}
    \|T(\xi^{\leftarrow}_{s_n})-T(\widehat\xi^{\leftarrow,h}_n)\|
  \right\|_{L^p}
  &\le
  L_T
  \left(
    A_{\eta_{\mathrm{sde}}}
    \|\xi^{\leftarrow}_0-\widehat\xi^{\leftarrow}_0\|_{L^p}
    +
    H_{\eta_{\mathrm{sde}},\delta_{\mathrm{sde}}}
    +
    C^{\mathrm{sde}}_{\mathrm{num},p}h^\varrho
  \right).
  \label{eq:latent-process-sde-error}
\end{align}
Moreover, any Wasserstein bound above is inherited under pushforward, since
\(
  W_p(T_\#\mu,T_\#\nu)
  \le
  L_T W_p(\mu,\nu)
\)
whenever the relevant latent coupling is supported on the Lipschitz region.
\end{corollary}

\begin{proof}
The pathwise and strong \(L^p\) statements follow from
\(\|T(x)-T(y)\|\le L_T\|x-y\|\).
The Wasserstein statement follows by pushing forward the same latent coupling through \(T\).
\end{proof}

\subsection{Additional path-law control for the reverse SDE}
\label{subsec:reverse-sde-stability}

The preceding bounds control coupled paths.
For the reverse-time SDE, one also obtains a path-law bound by Girsanov's theorem.

\begin{proposition}
\label{prop:reverse-sde-girsanov-stability}
Assume that the usual Girsanov conditions hold, and let
\(\mathsf P\) and \(\widehat{\mathsf P}\) denote the path laws of the exact and approximately guided reverse SDEs on
\(C([0,R_{\tau,t_1}];\mathbb R^{m_\xi})\).
Assume also that
\(
  \Law(\xi^{\leftarrow}_0)
  \ll
  \Law(\widehat\xi^{\leftarrow}_0).
\)
Then
\begin{align}
  \KL(\mathsf P\|\widehat{\mathsf P})
  &=
  \KL\!\left(
    \Law(\xi^{\leftarrow}_0)
    \,\middle\|\,
    \Law(\widehat\xi^{\leftarrow}_0)
  \right)
  +
  \frac12
  \E_{\mathsf P}
  \int_0^{R_{\tau,t_1}}
  \beta(\kappa_s)
  \|g(\kappa_s,\xi^{\leftarrow}_s)-\widehat g(\kappa_s,\xi^{\leftarrow}_s)\|^2
  \dd s.
  \label{eq:latent-sde-path-kl}
\end{align}
In particular, if the two SDEs start from the same initial law, then the first term vanishes.
If the realised guidance error is uniformly bounded by \(\varepsilon_{\tau,t_1}\) along the exact paths, then, for \(0\le s\le R_{\tau,t_1}\),
\begin{align}
  \KL(\mathsf P_{[0,s]}\|\widehat{\mathsf P}_{[0,s]})
  \le
  \KL\!\left(
    \Law(\xi^{\leftarrow}_0)
    \,\middle\|\,
    \Law(\widehat\xi^{\leftarrow}_0)
  \right)
  +
  \frac12\varepsilon_{\tau,t_1}^2
  \int_0^s\beta(\kappa_r)\dd r.
  \label{eq:latent-sde-path-kl-bound}
\end{align}
Consequently,
\begin{equation}
  \left\|
    \Law(\xi^{\leftarrow}_s)-\Law(\widehat\xi^{\leftarrow}_s)
  \right\|_{\mathrm{TV}}^2
  \le
    \frac12
    \KL\!\left(
      \Law(\xi^{\leftarrow}_0)
      \,\middle\|\,
      \Law(\widehat\xi^{\leftarrow}_0)
    \right)
    +
    \frac14\varepsilon_{\tau,t_1}^2
    \int_0^s\beta(\kappa_r)\dd r.
  \label{eq:latent-sde-tv-bound}
\end{equation}
\end{proposition}

\begin{proof}
The two SDEs have a common diffusion coefficient. Moreover, we have from the definitions that
\[
  \|\Sigma_{\mathrm{sde}}(s)^{-1}(b(s,u)-\widehat b(s,u))\|^2
  =
  \beta(\kappa_s)
  \|g(\kappa_s,u)-\widehat g(\kappa_s,u)\|^2.
\]
The identity in \eqref{eq:latent-sde-path-kl} follows immediately from Girsanov's theorem.
The KL bound follows from the uniform guidance-error assumption.
The total-variation bound follows by data processing and Pinsker's inequality.
\end{proof}

\subsection{Derivative-free guidance}
\label{sec:derivative-free-guidance}

When \(\nabla_{\xi}\log\Lambda(\xi)\) is unavailable, one can estimate the guidance field using a form of Fisher's identity.
Under the tilted bridge conditional density, we have
\[
  p(\xi_0\mid \xi_t=u,\calC)
  \propto
  \Lambda(\xi_0)
  \N\!\left(
    \xi_0;\alpha(t)u,\sigma(t)^2I_{m_\xi}
  \right).
\]
Using Fisher's identity, it follows that
\begin{equation}
  \nabla_u\log\rho_t(u)
  =
  -\frac{1}{\sigma(t)^2}
  \left(
    u-\alpha(t)\E[\xi_0\mid \xi_t=u,\calC]
  \right).
  \label{eq:fisher-guided-score}
\end{equation}
Meanwhile, since \(\rho_t(u)\propto \psi(t,u)r(u)\), we have $ \nabla_u\log\rho_t(u) = g(t,u) - u$. 
Combining this with the previous expression, we arrive at
\begin{equation}
  g(t,u)
  =
  \frac{\alpha(t)}{\sigma(t)^2}
  \left(
    \E[\xi_0\mid \xi_t=u,\calC]
    -
    \alpha(t)u
  \right).
  \label{eq:mean-difference-guidance}
\end{equation}
This immediately suggests a gradient-free estimator. For a given pair $(t,u)$, draw bridge samples
\[
  \xi^{(s)}=\alpha(t)u+\sigma(t)\epsilon^{(s)},
  \qquad
  \epsilon^{(s)}\sim\N(0,I_{m_\xi}), \qquad s=1,\dots,S.
\]
We can then define the self-normalised approximation
\begin{equation}
  \widehat g_S^{\mathrm{df}}(t,u)
  =
  \frac{\alpha(t)}{\sigma(t)^2}
  \left(
    \sum_{s=1}^S \bar w_s\xi^{(s)}
    -
    \alpha(t)u
  \right), \qquad \bar w_s
  =
  \frac{\Lambda(\xi^{(s)})}
  {\sum_{\ell=1}^S\Lambda(\xi^{(\ell)})}.
  \label{eq:derivative-free-guidance}
\end{equation}
This estimator only requires likelihood evaluations, but can be less sample-efficient than the gradient-based estimator because it estimates a small conditional-mean difference. 
Moreover, the prefactor \(\alpha(t)/\sigma(t)^2\) can amplify Monte Carlo noise at small noising times.
\section{Additional details for guidance-aware parameter estimation}
\label[appendix]{sec:parameter-adapt-add-details}

In this appendix, we fix the discretisation level \(m\) and suppress this from the notation.
We thus write \(Z_{\vart}\) for \(Z_{\vart,m}\), and
\(
  \mathcal J(\vart)
  =
  \log Z_{\vart}
\).

\subsection{A static interpretation via the Gibbs--Donsker--Varadhan formula}
The evidence for the condition has a useful interpretation via the Gibbs--Donsker--Varadhan formula \citep{donsker1975asymptotic}.
\begin{proposition}
\label{prop:conditional-evidence-gibbs}
Assume that \(\Lambda_{\vart,\calC}\ge0\) and that
\(
  0<Z_{\vart}<\infty
\).
Then the log evidence admits the representation
\begin{equation}
  \log Z_{\vart}
  =
  \sup_{q\ll r}
  \left\{
    \E_{\xi_0\sim q}
    \left[
      \log\Lambda_{\vart,\calC}(\xi_0)
    \right]
    -
    \KL(q\|r)
  \right\}.
  \label{eq:gibbs-variational}
\end{equation}
The supremum is attained at
\(
  \rho_{\vart}(\xi_0\mid\calC)
\),
whenever the displayed quantities are finite.
Moreover, whenever \(q\ll\rho_{\vart}(\cdot\mid\calC)\),
\begin{equation}
  \E_{\xi_0\sim q}
  \left[
    \log\Lambda_{\vart,\calC}(\xi_0)
  \right]
  -
  \KL(q\|r)
  =
  \log Z_{\vart}
  -
  \KL\bigl(q\|\rho_{\vart}(\cdot\mid\calC)\bigr).
  \label{eq:gibbs-gap-identity}
\end{equation}
\end{proposition}

\begin{proof}
The proof is standard.
By definition,
\(
  \frac{\rho_{\vart}(\xi_0\mid\calC)}{r(\xi_0)}
  =
  \frac{\Lambda_{\vart,\calC}(\xi_0)}{Z_{\vart}}.
\)
It follows that, whenever \(q\ll\rho_{\vart}(\cdot\mid\calC)\),
\begin{align}
  \KL\bigl(q\|\rho_{\vart}(\cdot\mid\calC)\bigr)
  &=
  \int
  \log
  \left(
    \frac{q(\xi_0)}
    {\rho_{\vart}(\xi_0\mid\calC)}
  \right)
  q(\xi_0)\dd\xi_0
  \nonumber\\
  &=
  \KL(q\|r)
  -
  \E_{\xi_0\sim q}
  \left[
    \log\Lambda_{\vart,\calC}(\xi_0)
  \right]
  +
  \log Z_{\vart}.
\end{align}
Rearranging this identity gives \eqref{eq:gibbs-gap-identity}.
The variational formula follows from nonnegativity of KL divergence, with equality at \(q=\rho_{\vart}(\cdot\mid\calC)\).
\end{proof}

By evaluating the identity in \Cref{prop:conditional-evidence-gibbs} at \(q=\rho_{\vart}(\cdot\mid\calC)\), we obtain
\begin{align}
  \log Z_{\vart}
  =
  \E_{\xi_0\sim\rho_{\vart}(\cdot\mid\calC)}
  \left[
    \log\Lambda_{\vart,\calC}(\xi_0)
  \right]
  -
  \KL\bigl(
    \rho_{\vart}(\cdot\mid\calC)
    \|r
  \bigr)
  \label{eq:static-free-energy-latent}
\end{align}
Combining this with the identity in \Cref{prop:kl-identity}, we also have that
\begin{align}
  \log Z_{\vart}
  =
  \E_{f_0\sim\pi_{\vart}(\cdot\mid\calC)}
  \left[
    \log L_{\calC}(f_0)
  \right]
  -
  \KL\bigl(
    \pi_{\vart}(\cdot\mid\calC)
    \|p_{\vart}
  \bigr).
  \label{eq:static-free-energy-process}
\end{align}
This shows that the evidence represents a balance between satisfaction of \(\calC\) and the relative-entropy cost of deforming \(p_{\vart}\) into \(\pi_{\vart}(\cdot\mid\calC)\).

\begin{corollary}
\label{cor:hard-event-evidence}
Suppose that the additional condition \(\calC\) is the event that the discretised process lies in a feasible set \(A\subseteq E_m\), so that
\(
  L_{\calC}(f_0)
  =
  \mathbf 1_A(f_0)
\).
If \(p_{\vart}(A)>0\), then
\(
  Z_{\vart}
  =
  p_{\vart}(A)
\)
and
\(
  \pi_{\vart}(\cdot\mid\calC)
  =
  p_{\vart}(\cdot\mid A).
\)
Thus, in particular, 
\begin{equation}
  -\log Z_{\vart}
  =
  \KL\bigl(p_{\vart}(\cdot\mid A)\|p_{\vart}\bigr).
  \label{eq:hard-event-evidence-kl}
\end{equation}
Thus, maximising the evidence is equivalent to minimising the KL deformation from the prior process law to the constrained process law.
\end{corollary}

\begin{proof}
Since \(L_{\calC}=\mathbf 1_A\),
\(
  Z_{\vart}
  =
  \int_A p_{\vart}(f_0)\dd f_0
  =
  p_{\vart}(A)
\),
and the conditioned law is \(p_{\vart}(\cdot\mid A)\).
Moreover,
\[
  \frac{\dd p_{\vart}(\cdot\mid A)}{\dd p_{\vart}}(f_0)
  =
  \frac{\mathbf 1_A(f_0)}{p_{\vart}(A)}.
\]
It follows that
\[
  \KL\bigl(p_{\vart}(\cdot\mid A)\|p_{\vart}\bigr)
  =
  \E_{p_{\vart}(\cdot\mid A)}
  \left[
    \log\frac{1}{p_{\vart}(A)}
  \right]
  =
  -\log p_{\vart}(A)
  =
  -\log Z_{\vart}.
\]
\end{proof}

\begin{corollary}
\label{cor:soft-residual-evidence}
Suppose that the additional condition is encoded by a likelihood of the form
\(
  L_{\calC}(f_0)
  =
  a\exp\{-\Phi_{\calC}(f_0)\},
\)
where \(\Phi_{\calC}(f_0)\ge0\) and \(a>0\) is independent of \(\vart\).
Then, up to additive constants independent of \(\vart\),
\begin{equation}
  -\log Z_{\vart}
  \equiv
  \KL\bigl(\pi_{\vart}(\cdot\mid\calC)\|p_{\vart}\bigr)
  +
  \E_{f_0\sim\pi_{\vart}(\cdot\mid\calC)}
  \left[
    \Phi_{\calC}(f_0)
  \right].
  \label{eq:soft-residual-evidence-kl}
\end{equation}
Thus evidence maximisation balances a small deformation from the ordinary process law against good satisfaction of the soft condition.
\end{corollary}

\begin{proof}
By \eqref{eq:static-free-energy-process}, we have that
\[
  \log Z_{\vart}
  =
  \E_{f_0\sim\pi_{\vart}(\cdot\mid\calC)}
  \left[
    \log L_{\calC}(f_0)
  \right]
  -
  \KL\bigl(
    \pi_{\vart}(\cdot\mid\calC)
    \|p_{\vart}
  \bigr).
\]
Since
\(
  \log L_{\calC}(f_0)
  =
  \log a-\Phi_{\calC}(f_0),
\)
we obtain
\[
  -\log Z_{\vart}
  =
  -\log a
  +
  \E_{f_0\sim\pi_{\vart}(\cdot\mid\calC)}
  \left[
    \Phi_{\calC}(f_0)
  \right]
  +
  \KL\bigl(
    \pi_{\vart}(\cdot\mid\calC)
    \|p_{\vart}
  \bigr).
\]
The term \(-\log a\) is independent of \(\vart\), giving \eqref{eq:soft-residual-evidence-kl}.
\end{proof}

\subsection{A dynamic interpretation via Doob transform and Girsanov's theorem}
\label{subsec:dynamic-interpretation}

The static free-energy identity shows that maximising \(Z_{\vart}\) balances condition satisfaction against the KL cost of deforming the ordinary process law.
We can also obtain a dynamic interpretation of this deformation cost, using Doob transforms and Girsanov's theorem \citep{doob1957conditional,girsanov1960transforming,oksendal2003stochastic}.
The result below shows that, under a suitable finite-horizon reference diffusion, the same deformation cost is represented by the energy of the guidance drift.

For each \(\vart\), recall that the latent likelihood is \(\Lambda_{\vart,\calC}\), the normalising constant is \(Z_{\vart}\), the reference density is \(r(\xi)=\N(\xi;0,I_{m_\xi})\), and the latent target is
\begin{equation}
  \rho_{\vart}(\xi\mid\calC)
  =
  \frac{\Lambda_{\vart,\calC}(\xi)r(\xi)}{Z_{\vart}}.
  \label{eq:dynamic-latent-target-recall}
\end{equation}
We now give a path-space version of the same construction.
Instead of tilting the static latent law \(r\) by \(\Lambda_{\vart,\calC}(\xi)\), we consider a stationary reference diffusion in latent space and tilt its path law by the likelihood of its terminal endpoint.
This lets us interpret the KL deformation in terms of the energy of the Doob-transform guidance drift.

We work on a finite generation horizon, rescaled to \([0,1]\).
Let \(\bar\beta\in L^1([0,1])\) be nonnegative, and let \(\mathsf P_{\vart}\) be the path law of the stationary reference latent Ornstein--Uhlenbeck process
\begin{equation}
  \dd U_s
  =
  -\frac12\bar\beta(s)U_s\dd s
  +
  \sqrt{\bar\beta(s)}\dd B_s,
  \qquad
  U_0\sim r,
  \qquad
  0\le s\le1.
  \label{eq:reference-generation-sde}
\end{equation}
Under \(\mathsf P_{\vart}\), the one-time marginal of \(U_s\) is \(r\) for every \(s\in[0,1]\).
Thus \(\mathsf P_{\vart}\) is an unconditioned stationary reference dynamics in latent space.
Define the tilted path law \(\mathsf Q_{\vart}\) by
\begin{equation}
  \frac{\dd \mathsf Q_{\vart}}{\dd \mathsf P_{\vart}}
  =
  \frac{\Lambda_{\vart,\calC}(U_1)}{Z_{\vart}}.
  \label{eq:path-tilt}
\end{equation}
Under \(\mathsf Q_{\vart}\), paths are reweighted according to how well their terminal state satisfies the condition.
The terminal marginal under \(\mathsf P_{\vart}\) is \(r\), while the terminal marginal under \(\mathsf Q_{\vart}\) is precisely \(\rho_{\vart}(\cdot\mid\calC)\).
The corresponding Doob potential is
\begin{equation}
  h_{\vart}(s,u)
  =
  \E_{\mathsf P_{\vart}}
  \left[
    \Lambda_{\vart,\calC}(U_1)
    \mid
    U_s=u
  \right].
  \label{eq:h-generation}
\end{equation}
This is the expected future condition likelihood from the current generation-time state \(u\).
Equivalently, \(h_{\vart}\) is the same smoothed likelihood object as the noising-time potential \(\psi_{\vart}\), written in generation time rather than noising time.

\begin{assumption}
\label{ass:girsanov}
For the fixed parameter value \(\vart\), assume that
\(0<Z_{\vart}<\infty\), that
\(\Lambda_{\vart,\calC}>0\) \(r\)-almost everywhere, and that
\(
  \KL\bigl(\rho_{\vart}(\cdot\mid\calC)\|r\bigr)<\infty.
\)
Assume that \(h_{\vart}\) in \eqref{eq:h-generation} has a positive,
sufficiently regular version for which \(\nabla_u\log h_{\vart}\) is
well defined, It\^o's formula applies to \(h_{\vart}(s,U_s)\), and the
backward Kolmogorov drift term vanishes.
Assume that the density process
\[
  M_{\vart,s}
  \coloneqq
  \frac{h_{\vart}(s,U_s)}{Z_{\vart}},
  \qquad 0\le s\le1,
\]
is a true \(\mathsf P_{\vart}\)-martingale.
Further, suppose that this process admits the stochastic exponential representation as $M_{\vart,s}=M_{\vart,0}\mathcal E_{\vart,s}$, where \(M_{\vart,0}=\frac{h_{\vart}(0,U_0)}{Z_{\vart}}\) and
\[
  \mathcal E_{\vart,s}
  =
  \exp\left\{
    \int_0^s
    \sqrt{\bar\beta(r)}
    \nabla^\top_u\log h_{\vart}(r,U_r)
    \dd B_r
    -
    \frac12
    \int_0^s
    \bar\beta(r)
    \|\nabla_u\log h_{\vart}(r,U_r)\|^2
    \dd r
  \right\}.
\]
\end{assumption}

\begin{theorem}
\label{thm:girsanov}
Under Assumption~\ref{ass:girsanov}, the tilted law \(\mathsf Q_{\vart}\) is the Doob transform of \(\mathsf P_{\vart}\) associated with \(h_{\vart}\).
In particular, under \(\mathsf Q_{\vart}\), the coordinate process \(U_s\) solves
\begin{equation}
  \dd U_s
  =
  \left(
    -\frac12\bar\beta(s)U_s
    +
    \bar\beta(s)\nabla_u\log h_{\vart}(s,U_s)
  \right)\dd s
  +
  \sqrt{\bar\beta(s)}\dd B_s^{\mathsf Q_{\vart}},
  \label{eq:guided-sde-under-Q}
\end{equation}
where \(B^{\mathsf Q_{\vart}}\) is Brownian motion under \(\mathsf Q_{\vart}\), with initial law \(U_0\sim(\mathsf Q_{\vart})_0\).
Moreover,
\begin{equation}
  \frac12
  \E_{\mathsf Q_{\vart}}
  \left[
    \int_0^1
    \bar\beta(s)
    \|\nabla_u\log h_{\vart}(s,U_s)\|^2
    \dd s
  \right]
  =
  \KL\bigl(\rho_{\vart}(\cdot\mid\calC)\|r\bigr)
  -
  \KL((\mathsf Q_{\vart})_0\|(\mathsf P_{\vart})_0).
  \label{eq:energy-with-boundary}
\end{equation}
\end{theorem}

\begin{proof}
By the Markov property and Assumption~\ref{ass:girsanov}, \(h_{\vart}\) solves the backward Kolmogorov equation associated with \eqref{eq:reference-generation-sde}, with terminal condition \(h_{\vart}(1,u)=\Lambda_{\vart,\calC}(u)\).
Using also It\^o's formula, we have that
\begin{equation}
  \dd h_{\vart}(s,U_s)
  =
  \sqrt{\bar\beta(s)}
  \nabla h_{\vart}(s,U_s)^\top
  \dd B_s.
\end{equation}
The density process of \(\mathsf Q_{\vart}\) with respect to \(\mathsf P_{\vart}\) is
\begin{equation}
  M_{\vart,s}
  =
  \E_{\mathsf P_{\vart}}
  \left[
    \frac{\Lambda_{\vart,\calC}(U_1)}{Z_{\vart}}
    \mid
    \mathcal F_s
  \right]
  =
  \frac{h_{\vart}(s,U_s)}{Z_{\vart}}.
\end{equation}
It follows that
\begin{equation}
  \frac{\dd M_{\vart,s}}{M_{\vart,s}}
  =
  \sqrt{\bar\beta(s)}
  \nabla_u\log h_{\vart}(s,U_s)^\top
  \dd B_s.
\end{equation}
By Girsanov's theorem,
\[
  B_s^{\mathsf Q_{\vart}}
  =
  B_s
  -
  \int_0^s
  \sqrt{\bar\beta(r)}
  \nabla_u\log h_{\vart}(r,U_r)
  \dd r
\]
is Brownian motion under \(\mathsf Q_{\vart}\).
Substituting this into \eqref{eq:reference-generation-sde} gives \eqref{eq:guided-sde-under-Q}.
For the energy identity, write
\begin{align}
  \log M_{\vart,1}
  =
  \log M_{\vart,0}
  &+
  \int_0^1
  \sqrt{\bar\beta(s)}
  \nabla_u\log h_{\vart}(s,U_s)^\top
  \dd B_s^{\mathsf Q_{\vart}}
  \\
  &+
  \frac12
  \int_0^1
  \bar\beta(s)
  \|\nabla_u\log h_{\vart}(s,U_s)\|^2
  \dd s.
\end{align}
Taking expectation under \(\mathsf Q_{\vart}\), the stochastic integral has mean zero, and hence
\begin{align}
  \KL(\mathsf Q_{\vart}\|\mathsf P_{\vart})
  &=
  \E_{\mathsf Q_{\vart}}[\log M_{\vart,1}]
  \nonumber\\
  &=
  \KL((\mathsf Q_{\vart})_0\|(\mathsf P_{\vart})_0)
  +
  \frac12
  \E_{\mathsf Q_{\vart}}
  \int_0^1
  \bar\beta(s)
  \|\nabla_u\log h_{\vart}(s,U_s)\|^2
  \dd s.
  \label{eq:path-kl-energy}
\end{align}
On the other hand, the likelihood ratio \(\dd\mathsf Q_{\vart}/\dd\mathsf P_{\vart}\) depends only on \(U_1\).
The left-hand side therefore simplifies to
\[
  \KL(\mathsf Q_{\vart}\|\mathsf P_{\vart})
  =
  \KL((\mathsf Q_{\vart})_1\|(\mathsf P_{\vart})_1).
\]
Under \(\mathsf P_{\vart}\), \(U_1\sim r\), and under \(\mathsf Q_{\vart}\), \(U_1\sim\rho_{\vart}(\cdot\mid\calC)\).
Substituting this into the previous display gives
\begin{equation}
  \KL(\mathsf Q_{\vart}\|\mathsf P_{\vart})
  =
  \KL\bigl(\rho_{\vart}(\cdot\mid\calC)\|r\bigr).
\end{equation}
Combining this with \eqref{eq:path-kl-energy} proves \eqref{eq:energy-with-boundary}.
\end{proof}

\begin{remark}
For a finite stationary OU horizon with finite integrated noise, \(U_0\) and \(U_1\) are generally correlated.
Terminal reweighting can thus change the law of \(U_0\), resulting in the boundary term in \eqref{eq:energy-with-boundary}.
If $\smash{\int_0^1\bar\beta(s)\dd s\to\infty}$, then the correlation between \(U_0\) and \(U_1\) tends to zero.
Equivalently, the initial generation-time state becomes independent of the endpoint.
In that limit, terminal reweighting does not alter the initial law, and \((\mathsf Q_{\vart})_0=(\mathsf P_{\vart})_0\).
\end{remark}

We can now provide an interpretation of the evidence in terms of the guidance.
Define the guidance energy by
\begin{equation}
  \mathcal E_{\mathrm{guide}}(\vart)
  =
  \frac12
  \E_{\mathsf Q_{\vart}}
  \left[
    \int_0^1
    \bar\beta(s)
    \|\nabla_u\log h_{\vart}(s,U_s)\|^2
    \dd s
  \right].
  \label{eq:guide-energy-def}
\end{equation}
Suppose that we are in the limiting regime where \((\mathsf Q_{\vart})_0=(\mathsf P_{\vart})_0\). Then, by \Cref{thm:girsanov}, the guidance energy coincides with the relative-entropy cost
\begin{equation}
  \mathcal E_{\mathrm{guide}}(\vart)
  =
  \KL\bigl(\rho_{\vart}(\cdot\mid\calC)\|r\bigr).
  \label{eq:guide-energy-boundary}
\end{equation}
It follows, combining this with the result in \Cref{prop:conditional-evidence-gibbs}, that the log evidence can be written as
\begin{equation}
  \log Z_{\vart}
  =
  \E_{\xi_0\sim\rho_{\vart}(\cdot\mid\calC)}
  \left[
    \log\Lambda_{\vart,\calC}(\xi_0)
  \right]
  -
  \mathcal E_{\mathrm{guide}}(\vart).
  \label{eq:evidence-guidance-energy-boundary-free}
\end{equation}
Thus evidence maximisation can be interpreted dynamically as balancing condition satisfaction against the guidance energy required to deform the reference latent process.
For hard constraints, interpreted as limits of strictly positive soft constraints, the log-likelihood contribution vanishes on the constrained support, so evidence maximisation reduces, in the boundary-free regime, to minimising guidance energy.
For soft constraints, the objective also includes the expected residual loss in \Cref{cor:soft-residual-evidence}.

\subsection{Basic theoretical guarantees for parameter estimation}
\label{subsec:parameter-adaptation-theory}

We now record some basic theoretical results relating to our parameter estimation scheme.
We begin by establishing that the exact fixed-parameter time slices give unbiased evidence gradients.

\begin{proposition} 
\label{prop:intermediate-gradient}
Assume that \(0<Z_{\vart}<\infty\), that the reference density \(r\) and the noising schedule do not depend on \(\vart\), and that differentiation may be interchanged with integration.
Then, for every \(t\in[0,1)\) such that \(\psi_{\vart}(t,\cdot)\) is positive and differentiable in \(\vart\), 
\begin{equation}
  \nabla_{\vart}\mathcal J(\vart)
  =
  \nabla_{\vart}\log Z_{\vart}
  =
  \E_{u\sim\rho_{\vart,t}(\cdot\mid\calC)}
  \left[
    \nabla_{\vart}
    \log\psi_{\vart}(t,u)
  \right].
  \label{eq:intermediate-gradient}
\end{equation}
Thus, if a fixed-parameter \textnormal{\textsc{LatentFlow}} solve at \(\vart\) produces particles with selected time-slice marginals
\(
  u_j^{(i)}\sim\rho_{\vart,t_j}(\cdot\mid\calC),
\)
with
\(
  0\le t_j<1,
\)
then an unbiased estimator of \(\nabla_{\vart}\mathcal J(\vart)\) is 
\begin{equation}
  \widehat G(\vart)
  =
  \frac{1}{MJ}
  \sum_{i=1}^M\sum_{j=1}^J
  \nabla_{\vart}\log\psi_{\vart}(t_j,u_j^{(i)}).
  \label{eq:offline-gradient-estimator-appendix}
\end{equation}
\end{proposition}

\begin{proof}
Recall that
\(
  Z_{\vart}
  =
  \int_{\mathbb R^{m_\xi}}
  \psi_{\vart}(t,u)r(u)\dd u.
\)
Taking derivatives, and simplifying, we have
\begin{align*}
  \nabla_{\vart}\log Z_{\vart}
  &=
  \frac{1}{Z_{\vart}}
  \int_{\mathbb R^{m_\xi}}
  \nabla_{\vart}\psi_{\vart}(t,u)r(u)\dd u
  \\
  &=
  \int_{\mathbb R^{m_\xi}}
  \nabla_{\vart}\log\psi_{\vart}(t,u)
  \frac{\psi_{\vart}(t,u)r(u)}{Z_{\vart}}
  \dd u =
  \E_{u\sim\rho_{\vart,t}(\cdot\mid\calC)}
  \left[
    \nabla_{\vart}\log\psi_{\vart}(t,u)
  \right].
\end{align*}
\end{proof}

We now record an elementary convergence result for our update scheme.
In practice, we cannot compute 
\(
  h_{\vart}(t,u)
  :=
  \nabla_{\vart}\log\psi_{\vart}(t,u)
\)
directly, and it must be replaced by a self-normalised ratio estimator
\(\smash{
  \widehat h_{\vart,S}(t,u)
  :=
  \nabla_{\vart}\log\widehat\psi_{\vart,S}(t,u)
  }
\); 
see \Cref{sec:mc-guidance} for a related discussion.
This estimator is generally biased for finite $S$, since it is a ratio of Monte Carlo averages.
It is, however, consistent as \(S\to\infty\), under the usual integrability assumptions and provided that the limiting denominator is positive.
In any case, the following result only requires a conditional bias bound.

\begin{theorem}
\label{thm:offline-eb-ascent-biased}
Let \(\Theta\subseteq\mathbb R^q\) be convex, and suppose that \(\mathcal J:\Theta\to\mathbb R\) is continuously differentiable, bounded above by \(\mathcal J^\star\), and \(L_{\mathcal J}\)-smooth along the iterates.
Consider the updates 
\[
  \vart_{k+1}
  =
  \vart_k+\eta_k\widehat G_k,
  \qquad
  0<\eta_k<\frac{1}{2L_{\mathcal J}},
\]
and assume that the iterates remain in \(\Theta\).
Let \(\mathcal F_k\) be the filtration generated by the algorithm before drawing \(\widehat G_k\).
Suppose, in addition, that \(\widehat G_k\) admits the decomposition
\[
  \widehat G_k
  =
  \nabla_{\vart}\mathcal J(\vart_k)+b_k+\xi_k,
\]
where \(b_k\) is \(\mathcal F_k\)-measurable, and satisfies \(\|b_k\|\le\beta_k\) for some $\mathcal{F}_k$-measurable \(\beta_k\ge0\), and where $\xi_k$ satisfies
\[
  \E[\xi_k\mid\mathcal F_k]=0,
  \qquad
  \E[\|\xi_k\|^2\mid\mathcal F_k]\le\sigma^2.
\]
Then, for every \(K\ge1\),
\begin{equation}
  \min_{0\le k<K}
  \E\bigl[\|\nabla_{\vart}\mathcal J(\vart_k)\|^2\bigr]
  \le
  \frac{
    \mathcal J^\star-\mathcal J(\vart_0)
    +
    \sum_{k=0}^{K-1}\eta_k\,\E[\beta_k^2]
    +
    \frac{L_{\mathcal J}\sigma^2}{2}\sum_{k=0}^{K-1}\eta_k^2
  }{
    \sum_{k=0}^{K-1}\frac{\eta_k}{2}(1-2L_{\mathcal J}\eta_k)
  }.
  \label{eq:offline-min-gradient-bound-biased}
\end{equation}
In particular, if \(\eta_k\equiv\eta<1/(2L_{\mathcal J})\) and \(\beta_k\le\beta\) for all \(k\), the limiting stationarity floor is of order
\[
  O\bigl(\beta^2+L_{\mathcal J}\sigma^2\eta\bigr).
\]
If, in addition, \(\E[\beta_k^2]\to0\), then the bias contribution to this floor vanishes; the remaining stochastic contribution can be made small by decreasing \(\eta\), and disappears when \(\sigma=0\).
\end{theorem}

\begin{proof}
This follows from the standard smooth nonconvex analysis of stochastic gradient methods with biased gradients; see, for instance, \citet{ajalloeian2020convergence}.
Applying the same descent-lemma argument to \(f=-\mathcal J\), and using Young's inequality to absorb the bias term \(b_k\), gives \eqref{eq:offline-min-gradient-bound-biased}.
\end{proof}

\section{Additional Related Work}
\label[appendix]{sec:add-related-work}

\textbf{Conditioning stochastic-process priors.}
Perhaps the most well-studied case of stochastic-process conditioning is inference for Gaussian processes (GPs).
Under finitely many linear-Gaussian observations, the posterior over any finite collection of function values remains Gaussian and can be computed analytically \citep{rasmussen2006gaussian}.
This conjugacy is generally lost under non-Gaussian likelihoods, motivating a large literature on approximate inference  \citep{kuss2005assessing,nickisch2008approximations,hensman2015scalable,nickisch2018state}.
Related work encodes known linear operator or differential-equation constraints directly in the GP prior or regression construction \citep{jidling2017linearly,langehegermann2018algorithmic,gulian2020gaussian}.
By contrast, shape information such as monotonicity, boundedness, positivity, and convexity is commonly incorporated through virtual inequality observations, finite-dimensional coefficient inequalities, truncation, or posterior projection, generally resulting in non-Gaussian or projected posterior laws \citep{riihimaki2010gaussian,daveiga2012gaussian,lin2014bayesian,maatouk2017gaussian,lopezlopera2018finite,agrell2019gaussian,swiler2020survey,maatouk2025bayesian,astfalck2026posterior}.
An alternative generalised Bayes linear approach enforces domain restrictions through the choice of inferential solution space \citep{astfalck2024generalised}.
The closest GP-specific antecedent to the current work is \textsc{FlowGP} \citep{moss2026conditioning}, which conditions a finite-dimensional whitened GP representation on nonlinear or non-Gaussian information by solving a guided probability-flow ODE.
Our framework contains \textsc{FlowGP} as a special case but applies to a substantially broader class of stochastic-process priors; see~\Cref{app:innovation-representations}.

For SDE priors, \textsc{LatentFlow} is closely connected to diffusion-bridge simulation and conditioned diffusion processes.
Classical bridge methods condition a diffusion to hit specified endpoints or endpoint observations, often by adding an approximate guiding term to its drift.
\citet{delyon2006simulation} construct readily simulated endpoint-guided diffusions and derive an explicit likelihood-ratio correction.
\citet{schauer2017guided} derive guiding terms from auxiliary tractable diffusions with known transition densities, while \citet{bierkens2020simulation} extend this guided-proposal framework to partial endpoint observations and elliptic or hypoelliptic settings.
Related bridge samplers use coupling or other proposal mechanisms  \citep{bladt2014simple,bladt2016simulation}.
More recent methods learn time reversals using score matching \citep{heng2021simulating}, estimate bridge scores without first learning the time reversal \citep{baker2025score}, develop infinite-dimensional conditioning constructions \citep{baker2024conditioning}, introduce neural guided bridges \citep{yang2025neural}, or exploit Malliavin-calculus identities \citep{pidstrigach2025conditioning}.
Rare-event simulation and transition-path sampling address closely related path ensembles.
Transition-path sampling uses path-space Monte Carlo to study rare reactive trajectories \citep{dellago1998transition,dellago2002transition}, with recent diffusion-based methods also targeting transition-path ensembles \citep{seong2025transition}, whereas genealogical particle methods and adaptive multilevel splitting estimate rare-event probabilities by repeatedly selecting and resampling trajectories according to their progress toward a rare set \citep{delmoral2005genealogical,cerou2007adaptive,cerou2012sequential}.
\textsc{LatentFlow} targets likelihood-tilted path laws by tilting the
law of the underlying innovations and transporting it toward regions of
high condition likelihood.
It therefore provides a likelihood-guided complement to output-space bridge proposals and splitting methods when suitable pulled-back gradients or guidance estimators are available.

\textbf{Latent and reference-space formulations.}
Many stochastic-process laws admit representations as pushforwards of a tractable reference distribution rather than through an explicit density on the resulting process space.
This perspective is implicit in simulator and filtering representations of state-space models \citep{doucet2001sequential,sarkka2013bayesian}; in GMRF and SPDE constructions of spatial fields \citep{rue2005gaussian,lindgren2011explicit}; and in process priors and finite-dimensional representations built by composing or transforming latent random variables or functions, including random-feature kernel representations, latent-force models, GP state-space models, and deep GPs \citep{rahimi2007random,alvarez2009latent,alvarez2013linear,frigola2013bayesian,damianou2013deep,frigola2014variational}.
Pathwise GP constructions based on Matheron's rule similarly represent posterior sample paths as transformations of prior samples and provide scalable samples for downstream Monte Carlo computation \citep{wilson2020efficient,wilson2021pathwise}.
Innovation and non-centred parameterisations make this explicit for diffusion processes by parameterising latent paths through transformed path variables or driving-noise variables together with model parameters \citep{roberts2001inference,golightly2008bayesian,papaspiliopoulos2013data,whitaker2017bayesian,graham2022manifold}.
Neural-process and implicit-process families likewise define learned distributions over functions through neural maps \citep{garnelo2018neural,kim2019attentive,ma2019variational,gordon2020convolutional}.
Similarly, normalising flows push forward simple latent laws, although they usually learn invertible transformations so that output densities remain available by change of variables \citep{rezende2015variational,papamakarios2021normalizing}.

The same reference-law viewpoint is central to Bayesian inverse problems.
A posterior on a possibly infinite-dimensional parameter or function space is commonly defined by tilting a tractable reference measure by a likelihood \citep{stuart2010inverse,dashti2017bayesian}.
Function-space MCMC methods, including pCN and Hilbert-space HMC, are designed directly for such targets and can remain well behaved under mesh refinement \citep{beskos2011hybrid,cotter2013mcmc,hairer2014spectral}.
SMC, annealed importance sampling, and flow-augmented transport samplers construct sequences of intermediate distributions between reference and posterior laws \citep{neal2001annealed,delmoral2006sequential,arbel2021annealed}, while transport-map methods seek explicit reference-to-posterior maps, including in function-space inverse problems \citep{elmoselhy2012bayesian,hosseini2025conditional}.
Recent score-based methods have likewise emphasised defining and approximating posterior sampling dynamics directly on function spaces \citep{baldassari2023conditional,baldassari2024taming,baldassari2025preconditioned}.

\textbf{Bayesian inference in generator noise space.}
A related literature performs posterior inference directly in the input-noise coordinates of a fixed generative model.
For example, \citet{holden2022bayesian} use parallel-tempered pCN in the latent space of a VAE, while \citet{patel2022solution} use HMC together with a GAN.
Another early contribution is \citet{bohm2019uncertainty}.
Normalising-flow priors have also been combined with variationally trained conditional flows for amortised inverse-problem inference \citep{whang2021composing}.
Related early work performs constrained-HMC inference in the input-noise space of a differentiable generator \citep{graham2017asymptotically}, while distributional-control methods learn a transformation of the generator's reference law toward specified output-space constraints \citep{wu2022promptgen}.

There is also a substantial literature on combining transport maps with Monte Carlo methods.
Transport-map methods construct invertible transformations between tractable reference laws and target laws, either as approximate posterior representations or as preconditioners and proposals for Monte Carlo sampling \citep{elmoselhy2012bayesian,parno2018transport,hoffman2019neutra,cabezas2023transport,cabezas2024markovian}.
When an approximate map is embedded within a valid MCMC kernel with the appropriate correction, map approximation error affects sampling efficiency rather than the invariant target \citep{parno2018transport,hoffman2019neutra}.
NeuTra HMC uses a variationally trained normalising flow followed by HMC in the transported coordinates \citep{hoffman2019neutra}, whereas transport elliptical slice sampling uses a related nonlinear preconditioning strategy for elliptical slice sampling \citep{cabezas2023transport}.
A closely related latent construction is the neural-transport MCMC method of \citet{nijkamp2022mcmc}, in which an energy correction exponentially tilts the output law of a learned flow and sampling is performed in its latent coordinates.
That work is directed toward learning and sampling energy-based generative models, rather than conditioning a stochastic-process simulator.

More recent methods apply the same principle to diffusion and flow-based generators.
Outsourced diffusion sampling considers a deterministic representation \(x=f_{\vart}(z)\), with \(z\) drawn from a Gaussian reference, and trains an amortised diffusion sampler for the latent posterior induced by an output-space likelihood or reward \citep{venkatraman2025outsourced}.
Noise-space HMC instead applies HMC directly to the initial noise of a
deterministic reverse-diffusion map \citep{xia2026noise}.
Similar constructions have appeared for flow-matching models: Source-Guided Flow Matching modifies the source distribution while retaining the learned vector field, whereas D-Flow SGLD samples the source posterior induced by a new measurement operator \citep{wang2025sourceguided,parikh2026dflow}.
Related inverse-problem methods perform SMC in the latent coordinates of diffusions, guide sampling trajectories under pre-trained latent flow priors, or combine latent diffusion posterior sampling with surrogate-likelihood guidance \citep{achituve2025inverse,askari2025latent,wang2026latent}.

These approaches share the central principle of conditioning a generator by changing the distribution of its input noise.
\textsc{LatentFlow} differs primarily in the scope of the generator and in how the latent posterior is sampled.
In our case, \(T_{\vart}\) may describe an explicit stochastic-process prior, a numerical simulator, or a pre-trained neural generator, and need not be invertible or density-evaluable.
Moreover, \textsc{LatentFlow} does not train an auxiliary conditional generator or amortised latent sampler.
Instead, its time-dependent guidance is determined directly by the pulled-back likelihood and a known Gaussian noising bridge.

\textbf{Diffusion guidance and finite-time controlled transport.}
The sampling construction underlying \textsc{LatentFlow} is closely related to diffusion and flow-based generative modelling.
Diffusion probabilistic models define a forward noising process from a data distribution to a simple reference law together with a learned reverse process \citep{sohl2015deep,ho2020denoising}; continuous-time score-based models generate samples by reversing a noising SDE or solving the associated probability-flow ODE \citep{song2021score}.
Flow matching, rectified flows, and stochastic interpolants provide related continuous-time transport formulations \citep{lipman2023flow,liu2023flow,albergo2023building,albergo2023stochastic}.
Recent guided flow-matching methods extend these constructions to conditional generation and energy- or reward-guided sampling \citep{feng2025guidance,mark2025feynman}.

A large literature conditions pre-trained diffusion models on labels, measurements, inverse problems, or reward functions by modifying their reverse dynamics through classifier, classifier-free, measurement, reward, or off-the-shelf guidance terms \citep{dhariwal2021diffusion,ho2022classifierfree,chung2023diffusion,bansal2023universal}.
To control bias introduced by approximate guidance, other conditional diffusion samplers use SMC, particle MCMC, Feynman--Kac representations, twisting, or related correction mechanisms \citep{wu2023practical,cardoso2024monte,corenflos2025conditioning,singhal2025general,skreta2025feynman,wu2025reverse}.
Under their stated assumptions, some of these methods are asymptotically exact as the particle number increases, while others preserve a model-defined conditional target through a valid Monte Carlo correction; see~\citet{zhao2025conditional} for a recent review.
Diffusion and flow models have also been developed directly for function-valued data and function-space posterior sampling \citep{dutordoir2023neural,franzese2023continuous,kerrigan2024functional,lim2025score,yao2025guided}.

These constructions admit a broader interpretation through finite-time stochastic control.
Doob \(h\)-transforms provide a classical mechanism for conditioning Markov processes; for diffusions, the transformed dynamics acquire a drift correction involving the log-gradient of an appropriate space--time \(h\)-function \citep{doob1957conditional,didi2024framework}.
Related entropy-minimisation and Schr\"odinger-bridge formulations characterise controlled path measures as minimum-relative-entropy perturbations of a reference process subject to marginal or endpoint constraints \citep{follmer1985entropy,leonard2014survey,chen2016relation}.
Modern Schr\"odinger-bridge, Schr\"odinger--F\"ollmer, and controlled-diffusion samplers exploit this finite-time transport perspective for sampling unnormalised targets \citep{bernton2019schrodinger,huang2021schrodinger,zhang2022path,vargas2022bayesian,vargas2023denoising,vargas2024controlled,tamogashev2026data}, while diffusion Schr\"odinger-bridge methods construct transports between prescribed endpoint laws \citep{debortoli2021diffusion,peluchetti2023diffusion,liu2023i2sb,liu2024generalized}.
Several recent conditional-diffusion methods make the connection to Doob transforms explicit \citep{didi2024framework,denker2024deft,guo2026conditional}.
In infinite dimensions, \citet{park2024stochastic} develop a function-space stochastic-control formulation of diffusion bridges; \citet{yang2025infinite} learn discretisation-equivariant bridge scores using operator learning; and \citet{pieper2025simulation} construct guided bridge measures for semilinear SPDEs conditioned on linear terminal observations. 
\citet{baker2026supervised} develop a complementary construction for conditioning trained function-space diffusion priors by supervised guidance.

\textsc{LatentFlow} instantiates the finite-time guidance principle in latent innovation space.
Rather than learning a noising or transport process from samples, it uses a known reference-space bridge and the stochastic-process law supplied by \(T_{\vart}\).
Its guiding field is the gradient of \(\log h_t\), where \(h_t\) is the conditional expectation of the pulled-back likelihood under the latent noising bridge, equivalently the log-gradient of the corresponding smoothed pulled-back likelihood.
Thus, \textsc{LatentFlow} combines the pullback formulation used in generator-noise inference with a diffusion-style finite-time transport, yielding a reference-space sampler for a broad class of stochastic-process priors.

\section{Additional Numerical Experiments}
\label[appendix]{app:numerics-additional}

\subsection{Sampling Benchmarks}
\label[appendix]{app:sampling-benchmarks}

We now evaluate \textsc{LatentFlow} on three canonical benchmark distributions that are challenging for standard inference methods.
Each target takes the form
\(
    p^*(x) \;\propto\; p_{\text{prior}}(x)\cdot p(y \mid x),
\)
with prior $p_{\text{prior}}(x) = \mathcal{N}(x;\,0,I)$ and likelihoods $p(y \mid x)$ described below.

For \textsc{LatentFlow}, we use an Euler--Maruyama discretisation with $100$ time steps, and $10$ Monte Carlo samples for the guidance approximation.
We benchmark against the No-U-Turn Sampler \citep[NUTS;][]{hoffman2014no}, mean-field VI \citep{blei2017variational}, and a Laplace approximation.
Sample quality is measured using the Maximum Mean Discrepancy (MMD$^2$) with a multi-scale RBF kernel, with uncertainty reported as standard errors over $5$ independent runs, alongside the effective sample size per second (ESS/s).
We generate $N=1{,}000$ samples using each method. 
NUTS thus requires $1,000$ steps to produce its $1,000$ samples, as well as an additional $500$ warmup steps to ensure the chains are well mixed. 

Across all benchmarks, \textsc{LatentFlow} achieves MMD$^2$ scores competitive with or better than NUTS, while producing \emph{independent} samples.
\textsc{LatentFlow} also succeeds on problems where standard MCMC collapses to a single mode.

\textbf{Banana Distribution.}
The banana (twisted Gaussian) distribution has a curved, parabolic valley that resists random-walk exploration:
\[
    x_1 \sim \mathcal{N}(0,\,\sigma_1^2), \qquad
    x_2 \mid x_1 \sim \mathcal{N}\!\bigl(b(x_1^2 - \sigma_1^2),\,\sigma_2^2\bigr),
\]
with $b = 0.7$, $\sigma_1 = 1.0$, and $\sigma_2 = 0.4$.
\Cref{fig:banana} shows samples from each method overlaid on the true density. 
\textsc{LatentFlow} closely follows the curved valley, while VI and Laplace fail to capture the non-linear geometry.
NUTS tracks the true distribution but produces far fewer independent samples, reflected in the low ESS of $67$; see \Cref{tab:banana-quant}.
\textsc{LatentFlow} achieves the best MMD$^2$, and is substantially faster than NUTS, achieving over $100$ times greater ESS/s.

\begin{figure}[ht!]
    \centering
    \includegraphics[width=.98\linewidth]{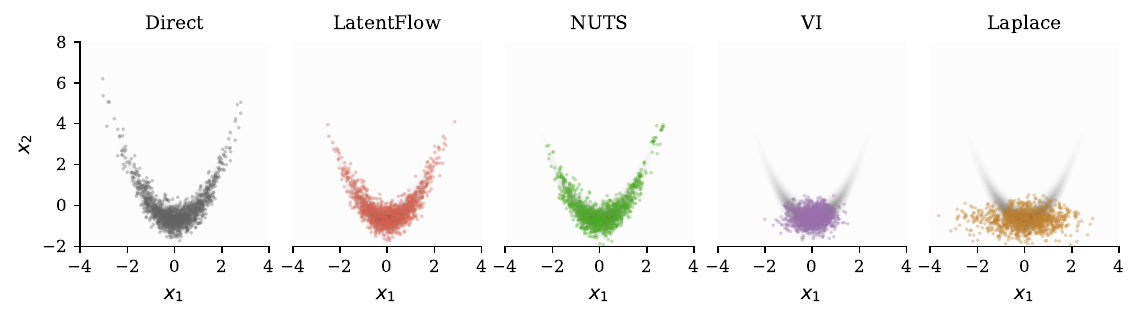}
    \caption{Banana distribution. 
    {Background:} true density. 
    {Scatter:} $1{,}000$ samples from each method. 
    Direct samples are exact draws.
    \textsc{LatentFlow} matches the true density well.
    VI and Laplace fail to capture the curved geometry. 
    }
    \label{fig:banana}
\end{figure}

\begin{table}[ht!]
\small
\centering
\caption{Banana distribution. MMD$^2$ $\pm$ stderr (lower is better), wall-clock time, and ESS/s, all over $5$ independent runs (3 d.p.).}
\label{tab:banana-quant}
\begin{tabular}{lllll}
\toprule
\textbf{Method} & \textbf{MMD}$^2$& \textbf{Time (s)} & \textbf{ESS} & \textbf{ESS/s} \\
\midrule
\textsc{LatentFlow}     & $-0.001 \pm 0.000$ & $0.060\pm0.000$ & $1000$ & $18010.0$ \\
NUTS    & $ 0.004 \pm 0.003$ & $0.470\pm0.160$ & $67$   & $141.7$   \\
VI      & $ 0.059$        & $0.310$         & $1000$ & $3270.7$  \\
Laplace & $ 0.044$             & ---             & ---    & ---       \\
\bottomrule
\end{tabular}
\end{table}

\textbf{Ring Distribution.}
The ring (shell) distribution is concentrated on a thin annulus of radius $R$, with density given by
\[
    p^*(x) \;\propto\; \exp\!\left(-\frac{(\|x\|^2 - R^2)^2}{2\sigma^2}\right),
\]
with $R = 2.0$ and $\sigma = 0.3$.
Since gradients point predominantly \emph{radially}, MCMC chains circulate slowly around the ring and mix poorly.
As shown in \Cref{fig:ring}, \textsc{LatentFlow} traces the full annulus accurately and achieves over $20\times$ the ESS/s of NUTS.

\pagebreak

\begin{figure}[ht!]
    \centering
    \includegraphics[width=\linewidth]{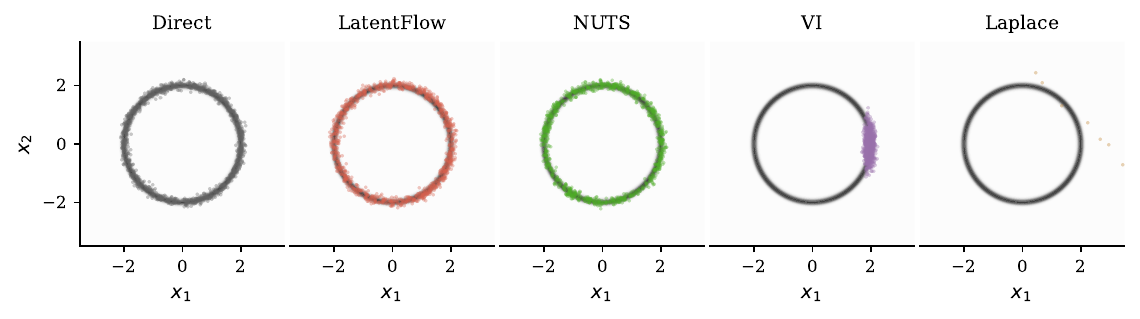}
    \caption{Ring distribution. %
    {Background:} true density. 
    {Scatter:} $1{,}000$ samples from each method.
    \textsc{LatentFlow} traces the annulus accurately.
    NUTS also covers the ring but requires more warm-up steps.
    VI collapses to a central blob.
    Laplace diverges catastrophically.}
    \label{fig:ring}
\end{figure}

\begin{table}[ht!]
\small
\centering
\caption{Ring distribution. MMD$^2$ $\pm$ stderr (lower is better), wall-clock time, ESS/s, and radial statistics, all over $5$ independent runs (3 d.p.). Ground truth: $\mathbb{E}[r]= 2$, $\mathrm{Std}[r]\approx 0.076$.}
\label{tab:ring-quant}
\begin{tabular}{lllllll}
\toprule
\textbf{Method} & \textbf{MMD}$^2$ & \textbf{Time (s)} & \textbf{ESS/s}
               & $\mathbb{E}[r]$ & $\mathrm{Std}[r]$ \\
\midrule
\textsc{LatentFlow}       & $ 0.000 \pm 0.000$ & $0.050\pm0.000$ & $21263.4$ & $2.004$ & $0.086$ \\
NUTS      & $ 0.001 \pm 0.001$ & $0.340\pm0.060$ & $897.1$   & $1.999$ & $0.075$ \\
VI        & $ 0.453$             & $0.290$         & $3412.0$  & $1.960$ & $0.087$ \\
Laplace   & $ 0.458$             & ---             & ---       & $274.105$ & $209.312$ \\
\bottomrule
\end{tabular}
\end{table}

\textbf{Mixture of Gaussians.}
We test multi-modal coverage using a mixture of $K=4$ equal-weight Gaussians centred at $(\pm 2, \pm 2)$ with $\sigma = 0.5$.
\textsc{LatentFlow} places samples in all four quadrants, while NUTS concentrates in quadrant~4; see \Cref{fig:mog}.
While NUTS achieves a raw ESS of $943$ and an ESS/s of $2{,}191.4$ (\Cref{tab:mog-quant}), its samples originate from a single mode.
Its apparent efficiency is thus misleading: it fails to cover multiple modes, reflected in its MMD$^2$ of $0.401 \pm 0.009$.
\textsc{LatentFlow} achieves the best sample quality with MMD$^2 = 0.00079 \pm 0.00027$, recovering the correct global mean of
approximately $(0, 0)$ and standard deviation near $2.04$ in both dimensions, compared to NUTS's collapsed mean of $(2.016,\,1.984)$ with standard deviation ${\approx}\,0.490$.

\begin{figure}[ht!]
    \centering
    \includegraphics[width=\linewidth]{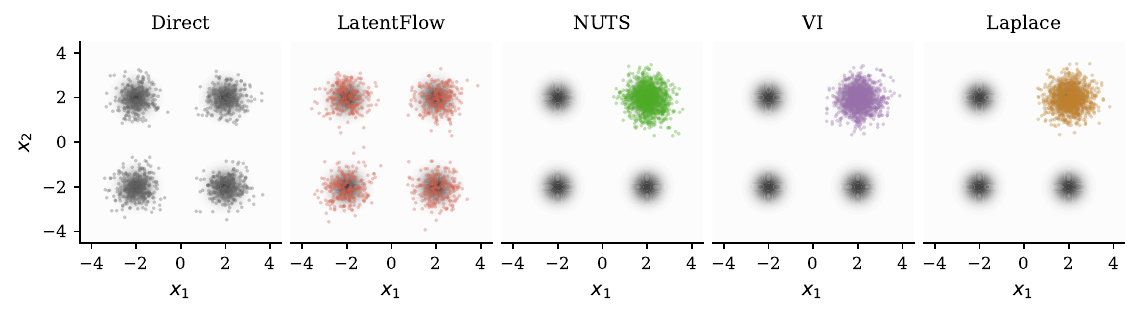}
    \caption{Mixture of Gaussians. 
    {Background:} true density (four modes). 
    {Scatter:} $1{,}000$ samples from each method.
    \textsc{LatentFlow} covers all four modes. 
    NUTS collapses entirely to quadrant~4 ($100\%$ of the chain). 
    VI collapses to a single Gaussian centred near $(2, 2)$.}
    \label{fig:mog}
\end{figure}

\begin{table}[ht!]
\small
\centering
\caption{Mixture of Gaussians.
MMD$^2$ $\pm$ stderr (lower is better), wall-clock time, and ESS/s, all over $5$ independent runs (3 d.p.).
Despite its high ESS/s, NUTS fails to achieve multi-modal coverage (all samples are from one mode).}
\label{tab:mog-quant}
\begin{tabular}{lllll}
\toprule
\textbf{Method} & \textbf{MMD}$^2$ & \textbf{Time (s)} & \textbf{ESS} & \textbf{ESS/s} \\
\midrule
\textsc{LatentFlow}  & $0.001 \pm 0.000$ & $0.050\pm0.000$ & $1000$ & $19865.9$ \\
NUTS & $0.401 \pm 0.009$     & $0.430\pm0.130$ & $943$  & $2191.4$  \\
VI   & $0.413$               & $0.140$         & $1000$ & $7119.4$  \\
\bottomrule
\end{tabular}
\end{table}

\pagebreak
\subsection{Spatial Processes}
\label[appendix]{app:spatial-processes}

~
\begin{figure}[ht!]
\vspace{-2mm}
  \centering
  \includegraphics[width=.95\linewidth]{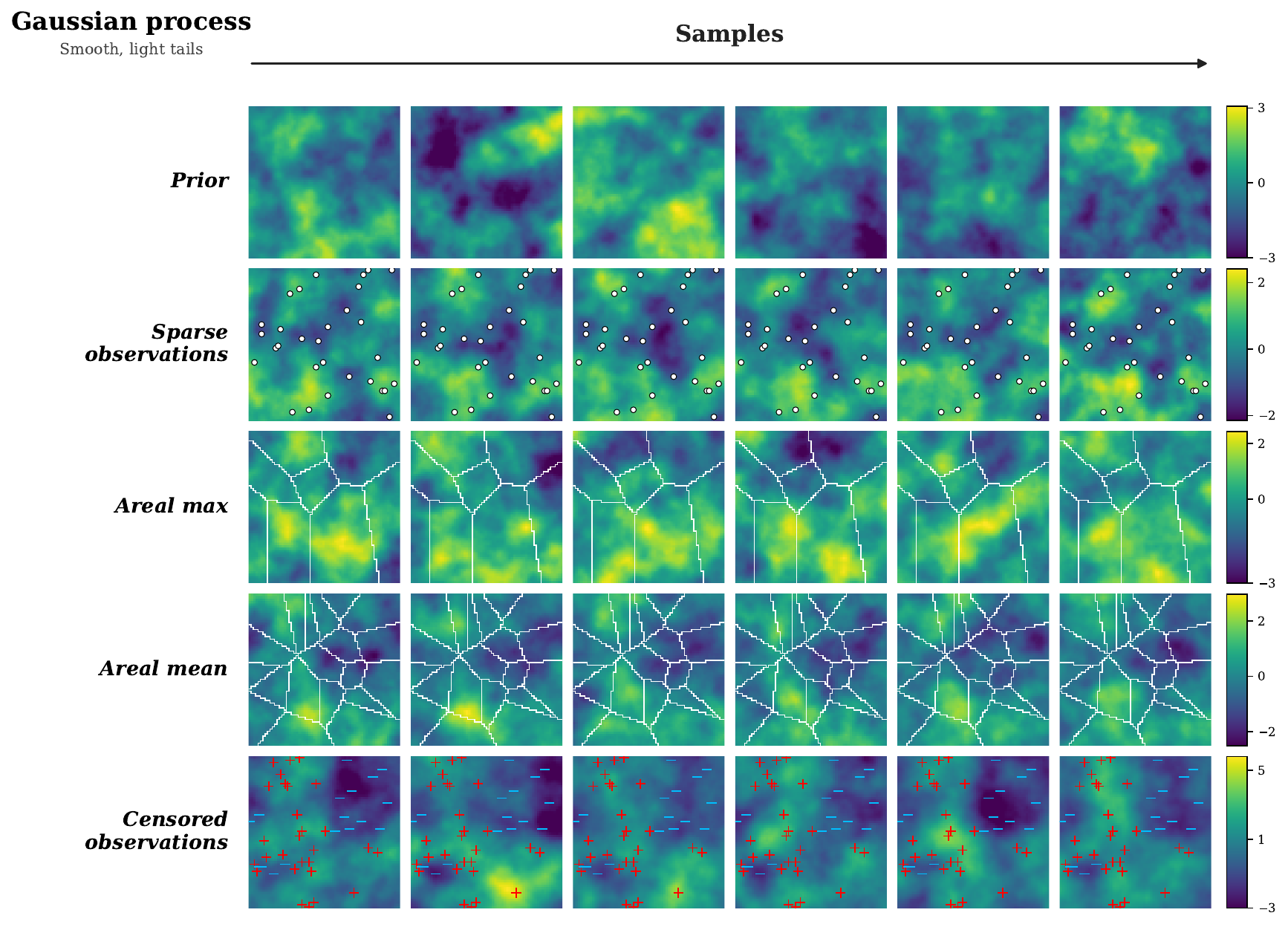}
  \caption{\textbf{Conditional sampling from a Gaussian process prior.} The base
    process is a smooth, light-tailed Gaussian field (Mat\'ern-$3/2$ kernel,
    lengthscale $0.15$). The posteriors interpolate smoothly
    between observations and honour each areal summary while preserving the
    short-lengthscale texture of the prior.}
  \label{fig:spatial-gp}
\end{figure}

\begin{figure}[ht!]
  \centering
  \includegraphics[width=.95\linewidth]{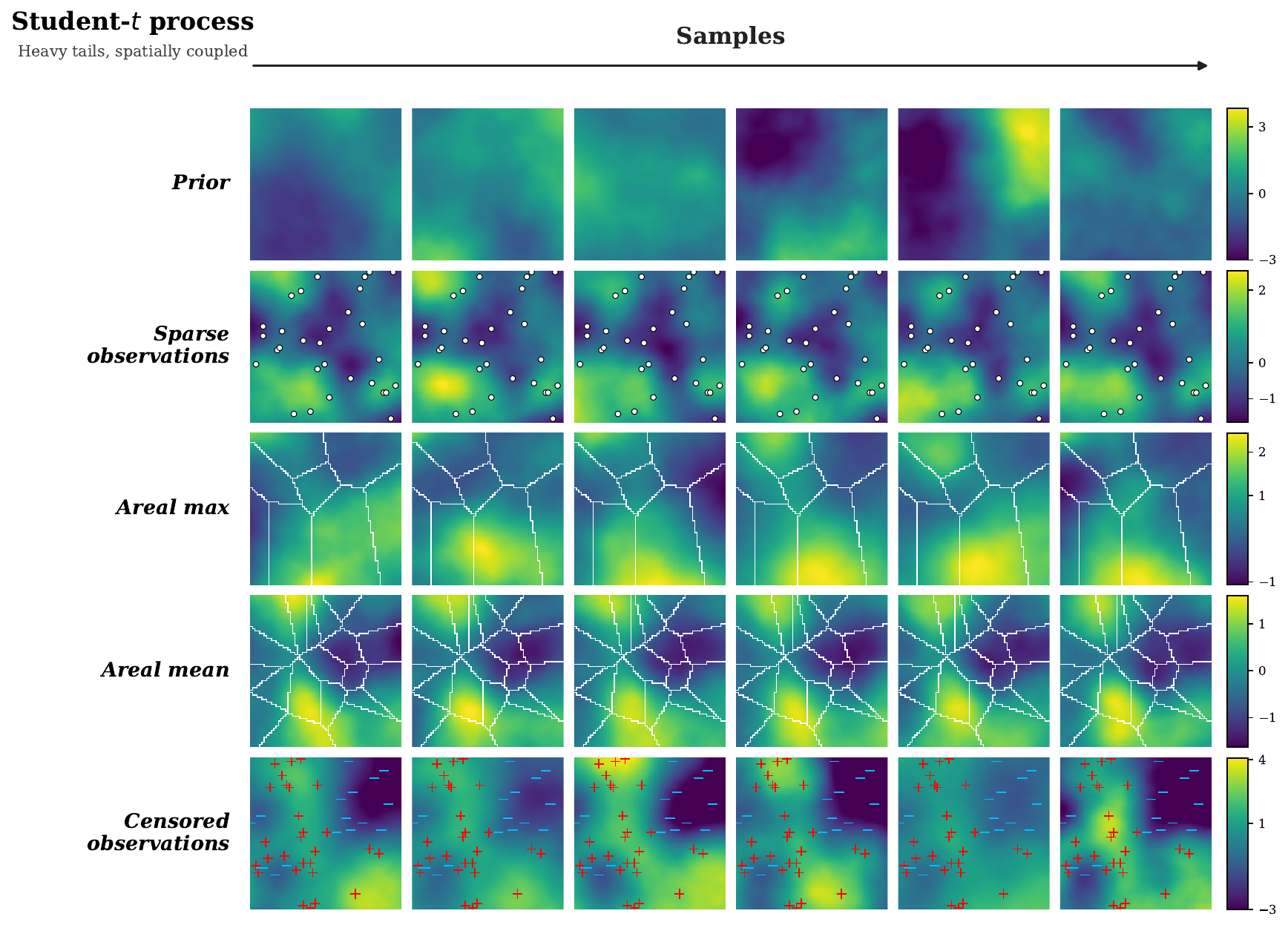}
  \caption{\textbf{Conditional sampling from a Student-$t$ process prior.} The
    base process is a heavy-tailed, spatially coupled Student-$t$ field
    ($\nu_t=4$, Mat\'ern-$3/2$ kernel, lengthscale $0.5$).}
  \label{fig:spatial-studentt}
  \vspace{-2mm}
\end{figure}

\begin{figure}[ht!]
  \centering
  \includegraphics[width=.95\linewidth]{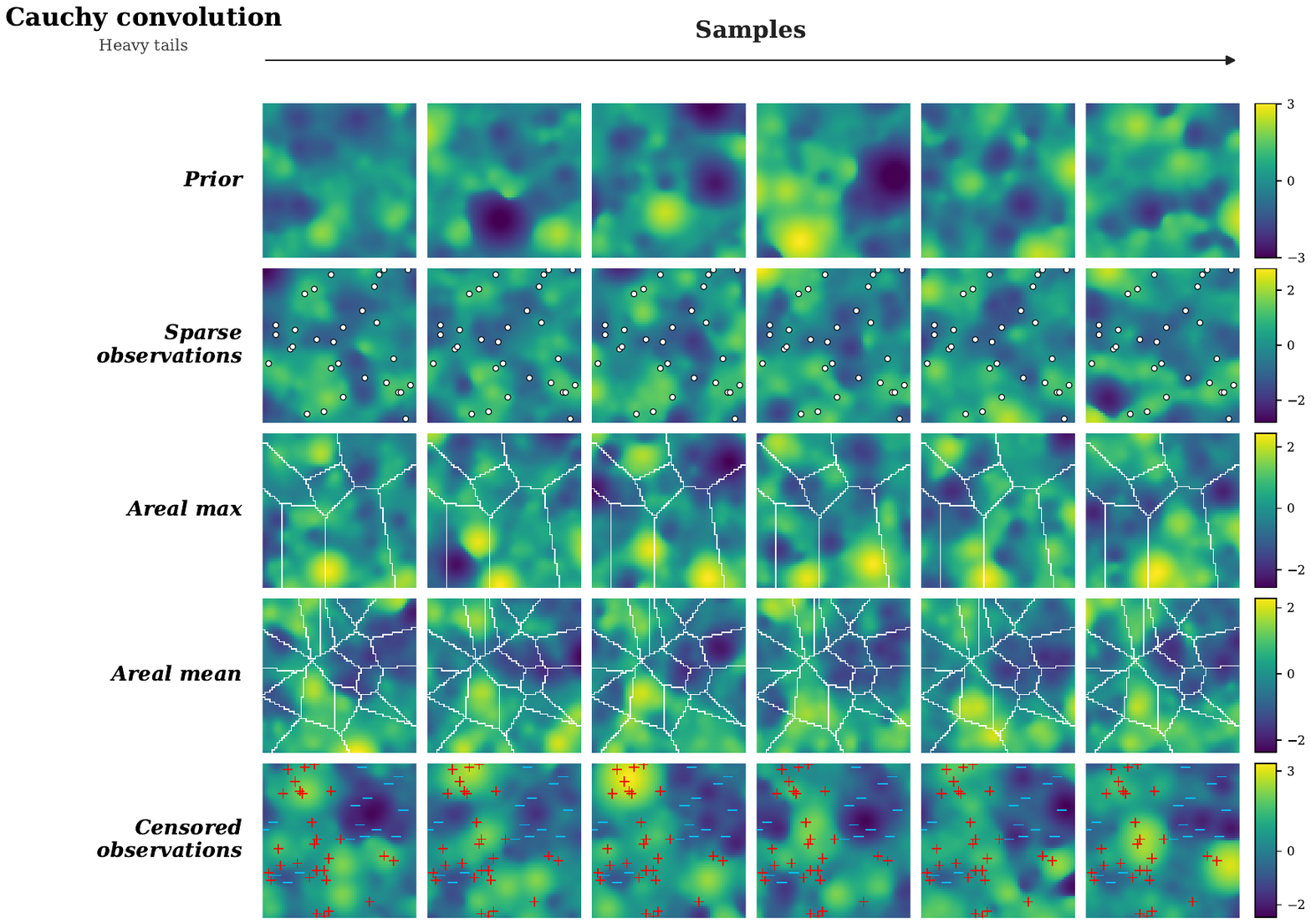}
  \caption{\textbf{Conditional sampling from a Cauchy convolution prior.} The
    base process is a very heavy-tailed field formed by convolving a Cauchy
    noise source with a short-lengthscale kernel (lengthscale $0.05$). The infinite-variance source produces
    sharp, highly localised features; the posteriors reconcile these with the
    imposed point, areal, and censoring constraints.}
  \label{fig:spatial-cauchy}
\end{figure}

\begin{figure}[ht!]
  \centering
  \includegraphics[width=.95\linewidth]{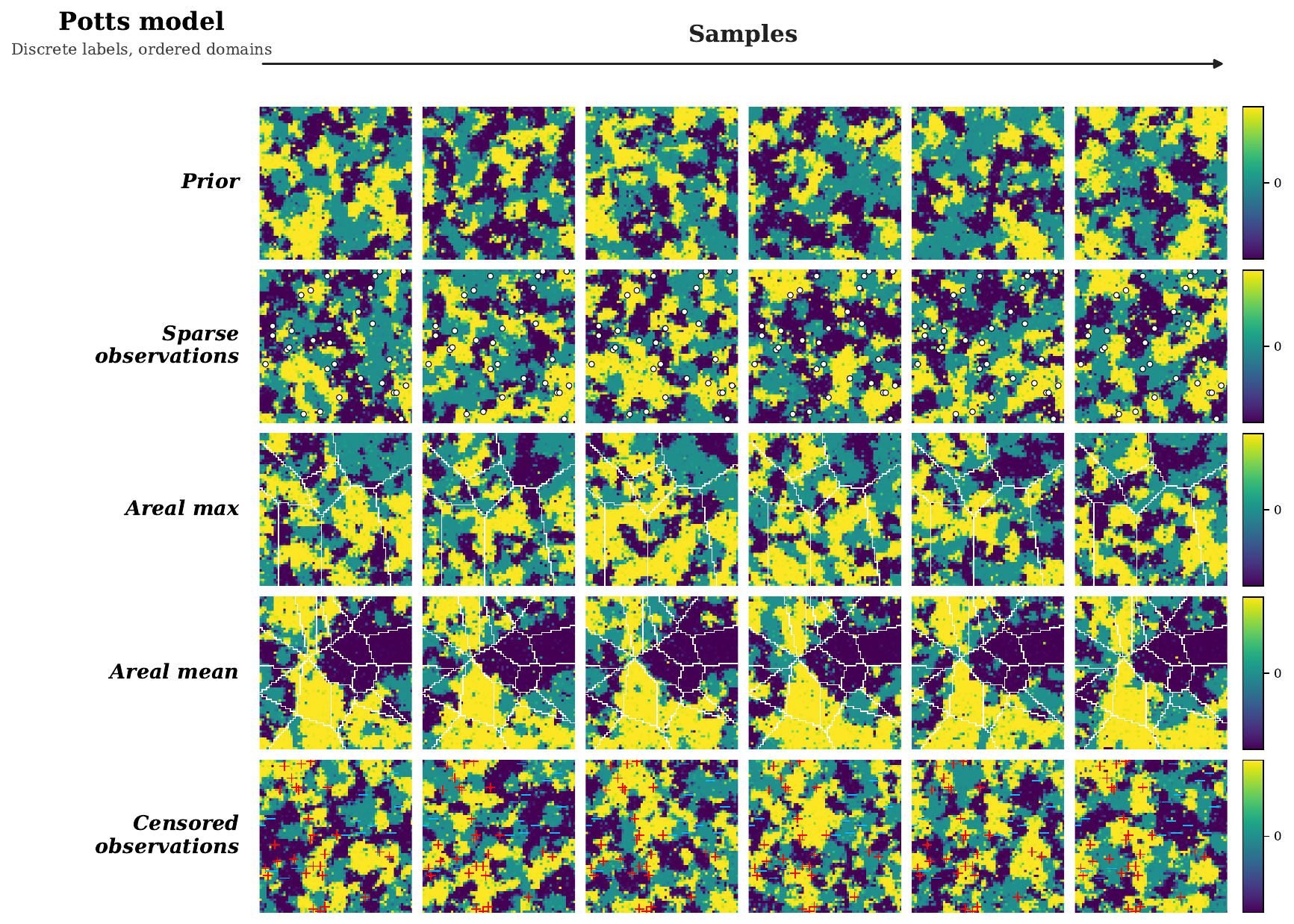}
  \caption{\textbf{Conditional sampling from a Potts model prior.} The base
    process is a discrete-label field with ordered domains (Potts model, $q=3$
    states mapped to $\{-1,0,1\}$, inverse temperature $\beta=1.2$). Conditioning reshapes the domain
    structure to match the observations while retaining the piecewise-constant,
    label-coherent character of the prior.}
  \label{fig:spatial-potts}
\end{figure}

\subsection{Temporal Processes}
\label{app:temporal-processes}
~

\begin{figure}[ht!]
\vspace{-1mm}
\centering
\includegraphics[width=.96\textwidth]{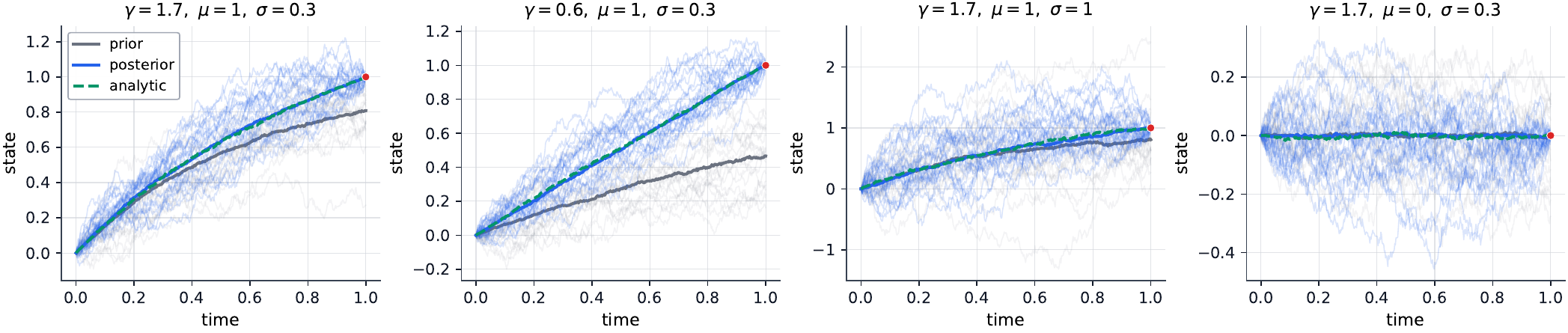}
\caption{
\textbf{Conditional sampling from an OU process prior}.
The prior is an OU diffusion process with mean-reversion rate \(\gamma\),  long-run mean \(\mu\), and diffusion coefficient \(\sigma\).
Each panel shows prior trajectories, posterior bridge samples, and the analytic bridge mean for a different choice of \((\gamma,\mu,\sigma)\).
The red marker denotes the terminal constraint, blue curves denote \textsc{LatentFlow} posterior samples, grey curves denote prior samples, and the green dashed curve denotes the analytic bridge mean.
}
\label{fig:ou-bridge-trajectories}
\vspace{1mm}
\end{figure}

\begin{figure}[ht!]
\centering
\includegraphics[width=.96\textwidth]{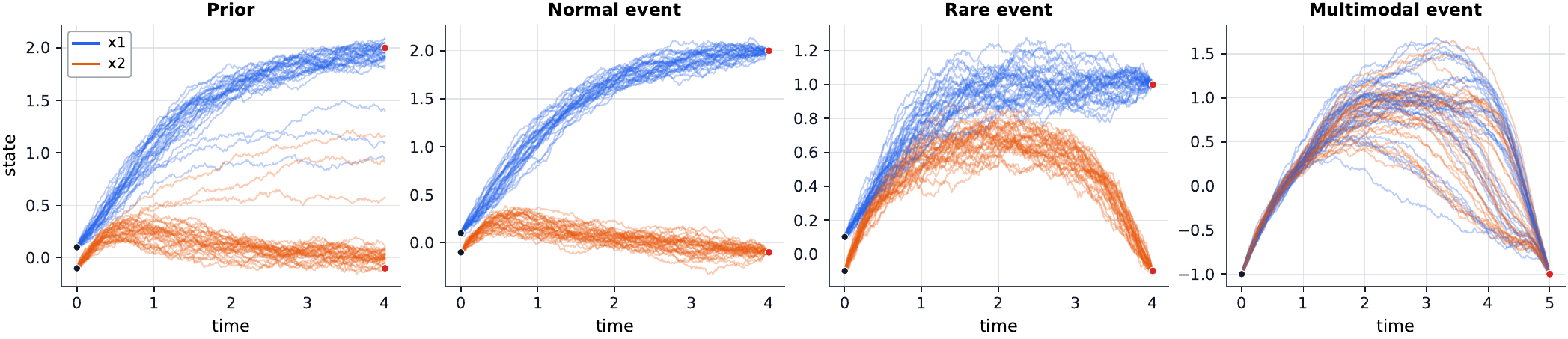}
\caption{
\textbf{Conditional sampling from a cell-differentiation diffusion process prior.}
The panels show samples from the prior and from three endpoint-conditioned laws: a typical terminal event, a rare terminal event, and a multimodal terminal event.
Blue and orange curves denote the two state variables \(x_1\) and \(x_2\), black markers denote the initial state, and red markers denote the terminal constraint.
The rare and multimodal examples illustrate that the same latent bridge construction can condition nonlinear diffusion priors on \emph{very} low-probability events. In particular, for the rare event and the multimodal event no valid samples were found after simulating the unconditioned process 100,000 times.
}
\label{fig:cell-differentiation-bridges}
\vspace{1mm}
\end{figure}

\begin{figure}[ht!]
\centering
\includegraphics[width=.96\textwidth]{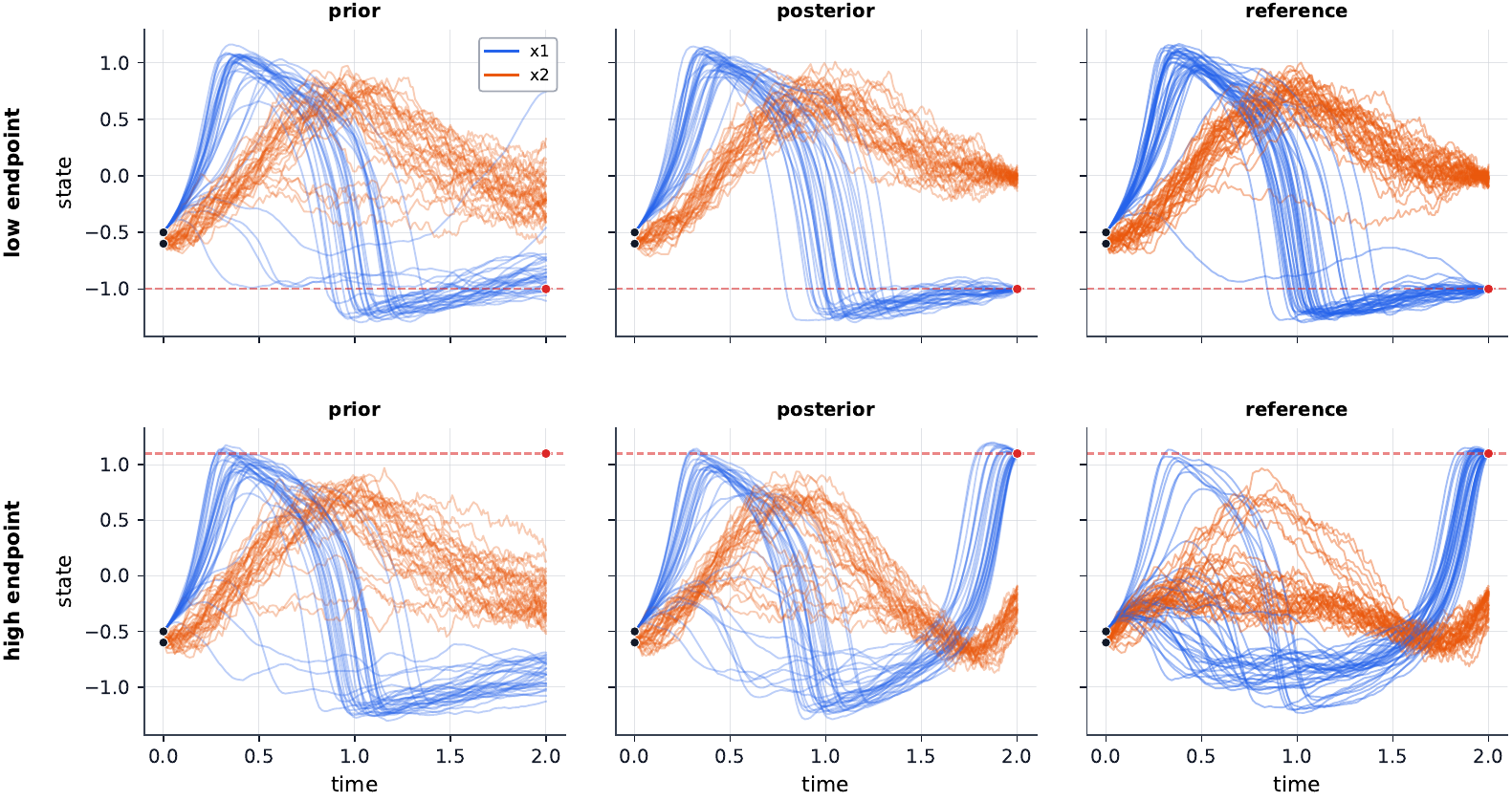}
\caption{
\textbf{Conditional sampling from a FitzHugh--Nagumo diffusion process prior}.
\emph{Rows}: two endpoint constraints on the first state variable: a low endpoint and a high endpoint.
\emph{Columns}: prior trajectories, \textsc{LatentFlow} posterior bridge samples, and reference bridge trajectories.
Blue and orange curves denote the two state variables \(x_1\) and \(x_2\), black markers denote the initial state, and red markers and dashed lines denote the terminal constraint.
The high-endpoint case illustrates conditioning on a rare excursion of the excitable system.
}
\label{fig:fhn-bridges}
\vspace{-1mm}
\end{figure}

\subsection{Spatio-temporal Processes}
\label[appendix]{app:spatio-temporal-processes}

\vspace*{\fill}

\begin{figure}[H]
    \centering
    \includegraphics[width=\textwidth]{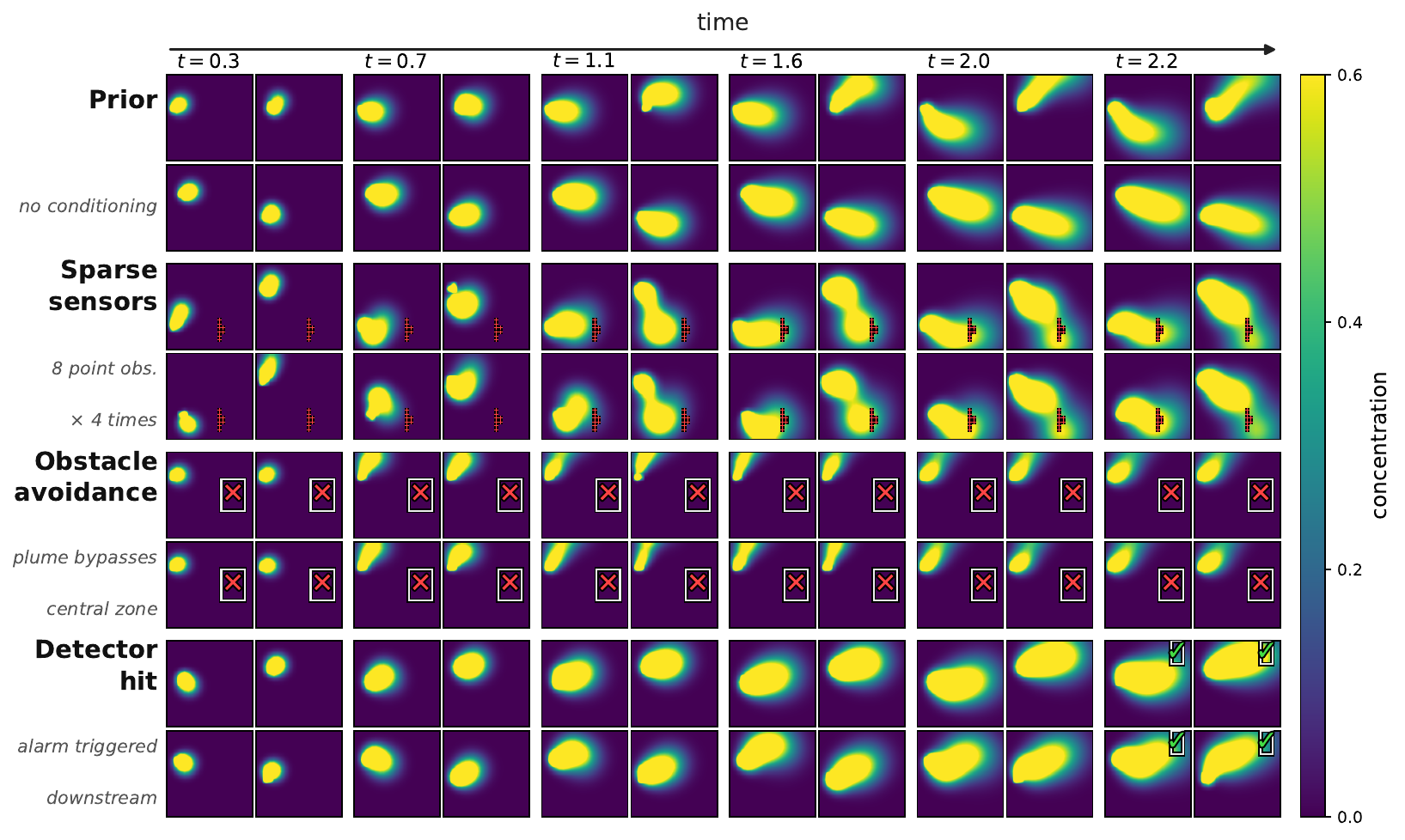}
    \caption{\textbf{Posterior samples from the advection--diffusion plume model under three conditioning regimes, each showing four independent draws at six time snapshots.}
    \textit{Prior:} unconditioned samples exhibiting diverse plume trajectories and intensities. 
    \textit{Sparse sensors:} conditioning on exact concentration values at 8 spatial locations observed at 4 times ($32$ observations total) recovers the ground-truth plume trajectory.
    \textit{Obstacle avoidance:} a binary condition requiring the plume to stay below concentration $0.06$ in the central rectangular region (dashed box, $\times$) at all times steers samples around the obstacle.
    \textit{Detector hit:} a binary condition requiring peak concentration to exceed $0.40$ in the downstream detector region (dashed box, $\checkmark$) concentrates mass on trajectories that reach the top-right corner.}
    \label{ffig:advection_spacetime}
\end{figure}

\vspace*{\fill}
\clearpage

\begin{figure}[p]
    \centering
    \includegraphics[width=\textwidth]{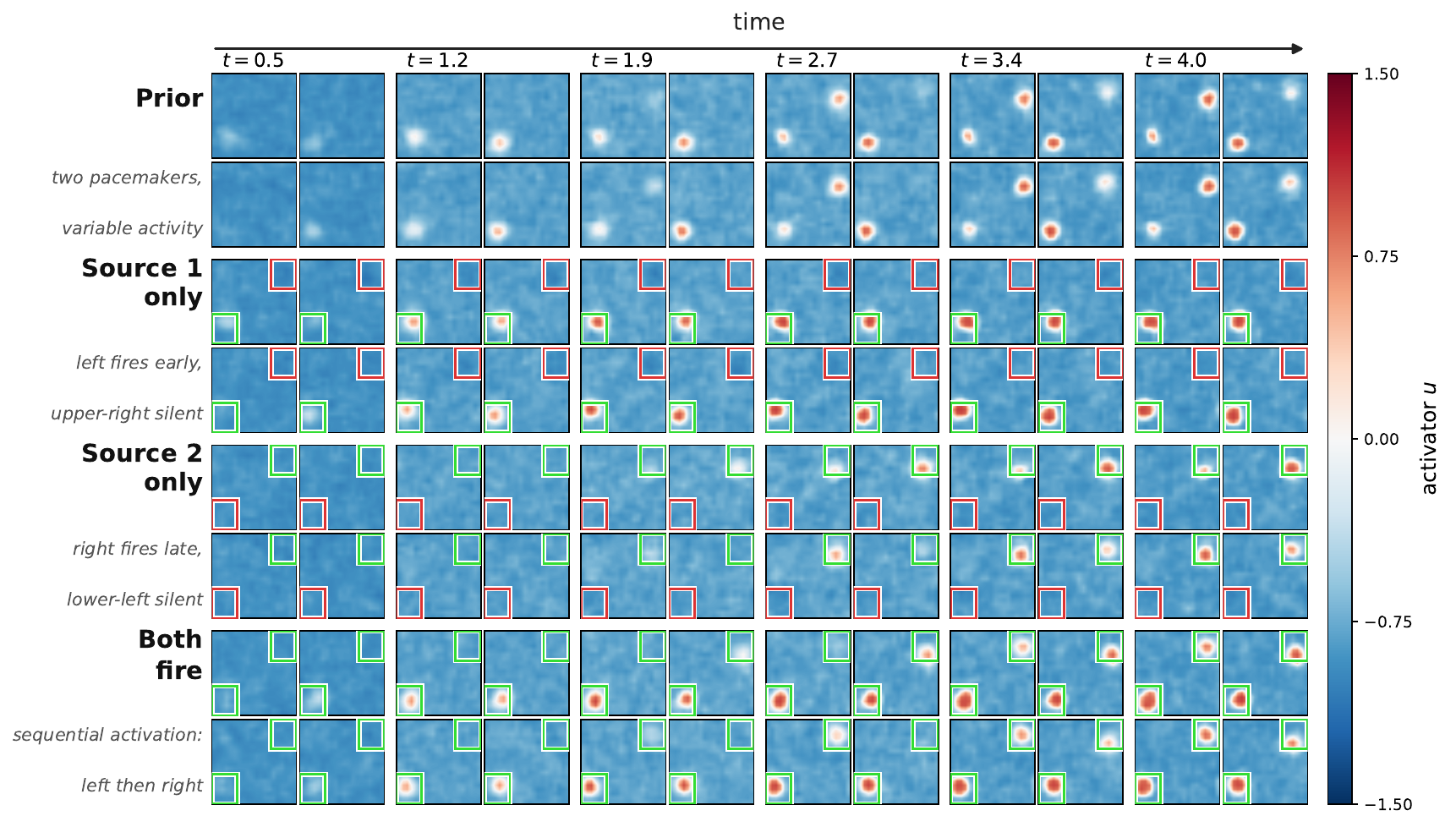}
    \caption{\textbf{Posterior samples from the FitzHugh--Nagumo excitable medium under three conditioning regimes, each showing four independent draws at six time snapshots}.
    \textit{Prior:} unconditioned samples show variable activation of two spatially separated pacemakers (green boxes: active; red boxes: silent). 
    \textit{Source~1 only:} conditioning on the lower-left pacemaker firing early ($t\lesssim 1.3$) while the upper-right remains silent produces a single travelling wavefront originating from the lower-left corner. 
    \textit{Source~2 only:} the complementary condition suppresses the lower-left source and requires the upper-right to fire late ($t\gtrsim 1.8$), yielding a wavefront from the upper-right.
    \textit{Both fire:} conditioning on sequential activation of both sources produces two distinct wavefronts propagating across the medium in succession.}
    \label{ffig:fitz_spacetime}
\end{figure}

\begin{figure}[p]
    \centering
    \includegraphics[width=\textwidth]{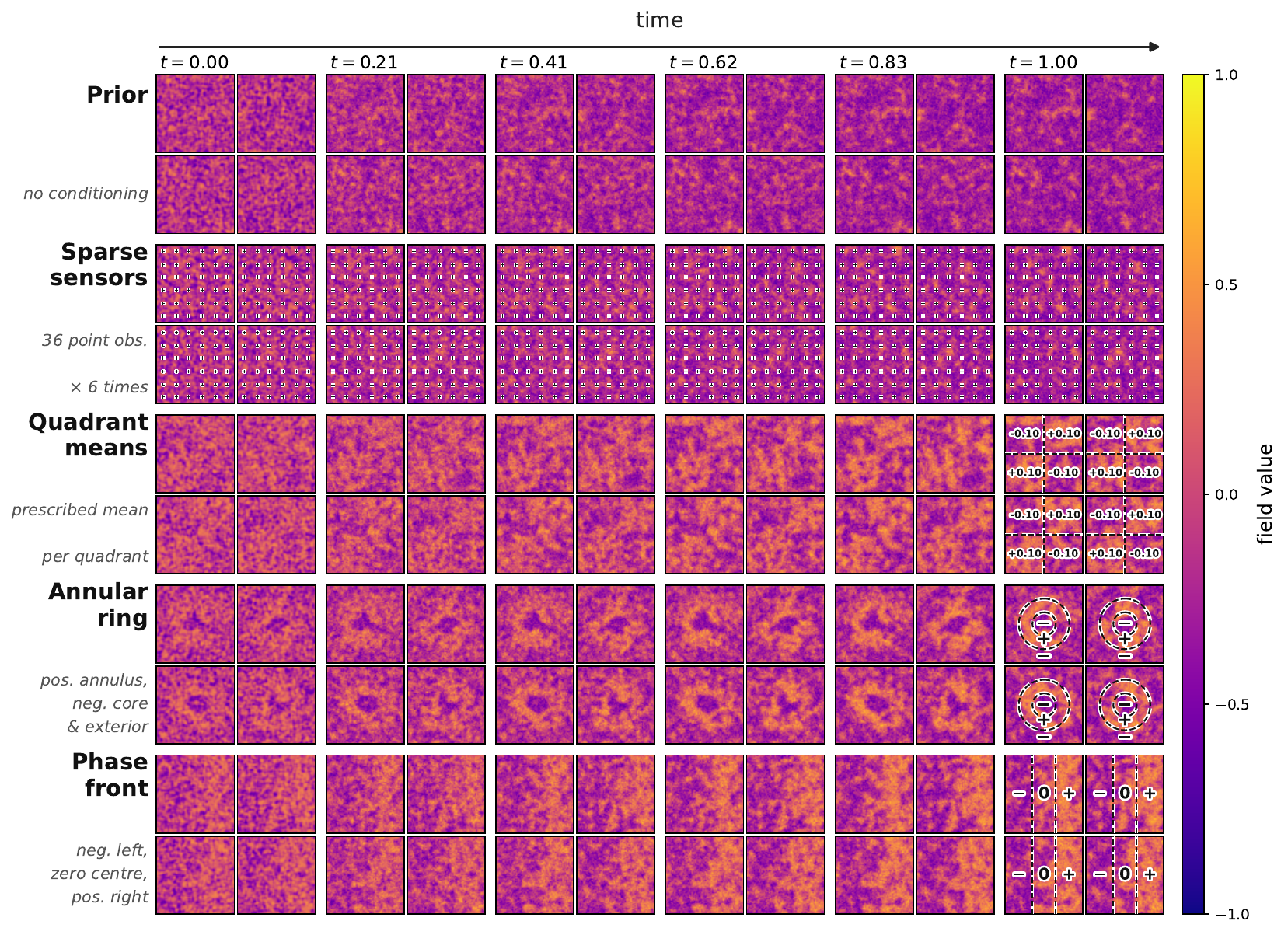}
    \caption{\textbf{Posterior samples from the Allen--Cahn phase-field model under four conditioning regimes, each showing four independent draws at six time snapshots.}  
    Except for the sparse-sensor experiment, all conditions constrain statistics of the final field $u(\cdot,t_{29})$.
    \textit{Prior:} unconditioned samples exhibit stochastic domain coarsening toward the $\pm 1$ phases.
    \textit{Sparse sensors:} conditioning on $36$ point observations at each
    of $6$ snapshot times ($216$ observations total) recovers the ground-truth field trajectory. 
    \textit{Quadrant means:} prescribing the spatial mean in each quadrant to a checkerboard pattern $(+0.10,\,-0.10,\,-0.10,\,+0.10)$ steers the final phase arrangement into the corresponding diagonal structure.  
    \textit{Annular ring:} prescribing a positive mean in an annular band ($0.30 \leq r \leq 0.65$) and negative means in the core and exterior produces a ring-shaped positive domain surrounded by negative phase.
    \textit{Phase front:} prescribing a left-to-right mean gradient $(-0.20,\,0.0,\,+0.20)$ across three vertical strips ($|x|\leq 0.30$ boundaries) induces a diffuse phase interface running vertically through the domain.}
    \label{ffig:allen_cahn_spacetime}
\end{figure}

\end{document}